\documentclass[lettersize,journal]{IEEEtran}  

\usepackage{amsthm}

\usepackage{times}
\usepackage{multicol}
\usepackage[bookmarks=true]{hyperref}
\usepackage{xcolor}
\usepackage{hyperref}
\usepackage{amsmath, amssymb}
\usepackage{amsfonts}
\usepackage{graphicx}
\usepackage{siunitx}
\usepackage{standalone}
\usepackage{booktabs}
\usepackage{mdframed} 
\usepackage{fancyvrb,multirow}
\usepackage{soul}
\usepackage{dsfont,mathabx}
\usepackage{array, booktabs}
\usepackage[para,online,flushleft]{threeparttable}
\usepackage{subfigure}
\usepackage{booktabs}
\usepackage{makecell}
\usepackage{threeparttable}
\usepackage{algorithmic}
\usepackage{multirow}
\usepackage{booktabs}
\usepackage{graphicx}
\usepackage{stfloats}
\usepackage{tabularx}
\usepackage{caption}
\usepackage{stfloats}
\usepackage{enumitem}
\usepackage[ruled,vlined,linesnumbered,noend]{algorithm2e}

\newtheorem{theorem}{Theorem}

\theoremstyle{definition}

\newtheorem{definition}{Definition}
\newtheorem{remark}{Remark}

\newtheorem{problem}{Problem}

\DeclareCaptionLabelFormat{upper}{\MakeUppercase{#1}~#2}
\newcolumntype{Y}{>{\centering\arraybackslash}X}
\newcolumntype{W}{>{\centering\arraybackslash}m{0.3cm}} 
\newcolumntype{Z}{>{\centering\arraybackslash}m{0.4cm}} 
\newcolumntype{M}[1]{>{\centering\arraybackslash}m{#1}}
\title{\LARGE \bf
  AMBUSH: Collaborative Capture in Complex Environments\\ with Neural Acceleration
}

\begin{document}

\author{Junfeng Chen$^{1,\star}$, YinHang Luo$^{1,\star}$,
Xinyi Wang$^2$, Junrui Li$^1$, and Meng Guo$^1$
\thanks{The authors are with $^1$the School of Advanced Manufacturing and Robotics,
  Peking University, Beijing 100871, China;
and $^2$the Distributed Autonomous Systems and Control Lab, University of 
Michigan, Ann Arbor 48109-1079, USA.
 {\tt\small meng.guo@pku.edu.cn}
}
}

\maketitle
\thispagestyle{empty}
\pagestyle{empty}

\begin{abstract}

 Collaborative capture of dynamic targets is common
 in nature as an essential strategy for weaker
 species against the strong.
 Similar concepts have shown to be useful for numerous
 robotic applications, such as security and surveillance,
 search and rescue.
 However, most existing works focus on analytical and geometric solutions or
 end-to-end reinforcement learning methods,
 which are largely constrained to obstacle-free environments or
 scenarios with sparse, regularly distributed obstacles.
 This work tackles the problem from a unique perspective:
 the renowned strategy of ``ambush'' alone would suffice
 for multiple slower pursuers to capture one
 faster evader with different levels of intelligence
 efficiently in complex environments.
 A parameterized strategy of ambush (including discrete
 and continuous parameters) is designed first,
 which takes into account the topological properties of the workspace,
 the truncated line-of-sight visibility,
 the relative speed ratio and the limited capture range.
 Then, a Hybrid Monte Carlo Tree Search (H-MCTS) algorithm is proposed
 to optimize the associated parameters through long-term planning,
 enabling the identification of highly promising parameters for future capture.
 Lastly, the neural acceleration is trained offline to learn
 the ranking of different choices of parameters across various environments,
 and to directly predict scores,
 replacing the rollout process in H-MCTS.
 The neural acceleration is adopted during online H-MCTS to accelerate
 the planning procedure while guaranteeing the planning quality.
 Its efficiency and effectiveness are validated in extensive simulations
 and hardware experiments,
 against evaders with different capabilities and intelligence levels,
 including  two-times higher velocity and human-controlled behavior.
\end{abstract}

\def\abstractname{Note to Practitioners}
\begin{abstract}
    This work is motivated by the practical challenges of enabling robotic teams to reliably capture an agile evader 
    in complex environments,
    which are critical for security patrols, intruder interception, and search-and-rescue operations. 
    Existing analytical methods often fail in obstacle-dense settings, 
    while learning-based approaches require extensive environment-specific retraining. 
    We demonstrate that the ambush strategy enables slower and fewer pursuers to capture an evader moving at two times higher speeds,
    including human-controlled adversaries exploiting environmental complexity. 
    Our solution combines a parameterized ambush framework adapting to the environmental topology and visibility constraints 
    with an H-MCTS planner enhanced by offline-learned heuristics. 
    This reduces the computation latency while maintaining high success rates.
    Through extensive simulations and hardware experiments, 
    we demonstrate the practical viability and robustness of our approach across diverse real-world scenarios.
    Practitioners can directly deploy this framework for the perimeter security in urban environments, 
    wildlife protection against poachers in dense terrain, 
    or unauthorized drone interception in cluttered airspace without environment-specific adaptation.
    The current framework uses centralized planning, 
    and extending it to fully decentralized execution is an important direction for distributed robotic teams.
    
\end{abstract}
\section{Introduction}\label{sec:intro}

Dynamic and collaborative capture refers to the process of coordinating
a group of pursuers to capture an evader, which might dynamically avoid
being captured.
It is ubiquitous in nature, e.g., animals often
hunt in groups~\cite{bailey2013group} especially when the targets
are superior in speed or agility.
Collaborative behaviors such as encirclement, ambush and allurement
have been observed in different species.
Due to its practical relevance to numerous applications such as
security, monitoring, surveillance, search and rescue,
these behaviors have attracted great attention from the robotics
and control community~\cite{
pierson2016intercepting,macharet2020adaptive,gan2023multi}.
Unmanned aerial vehicles (UAVs)
have been deployed to neutralize different threats for environmental
protection~\cite{kamminga2018poaching},
and ground vehicles (UGVs) for the intruder detection
~\cite{stolfi2021uav}.

\begin{figure}[!t]
    \centering
    \includegraphics[width=0.85\linewidth, height=0.88\linewidth]{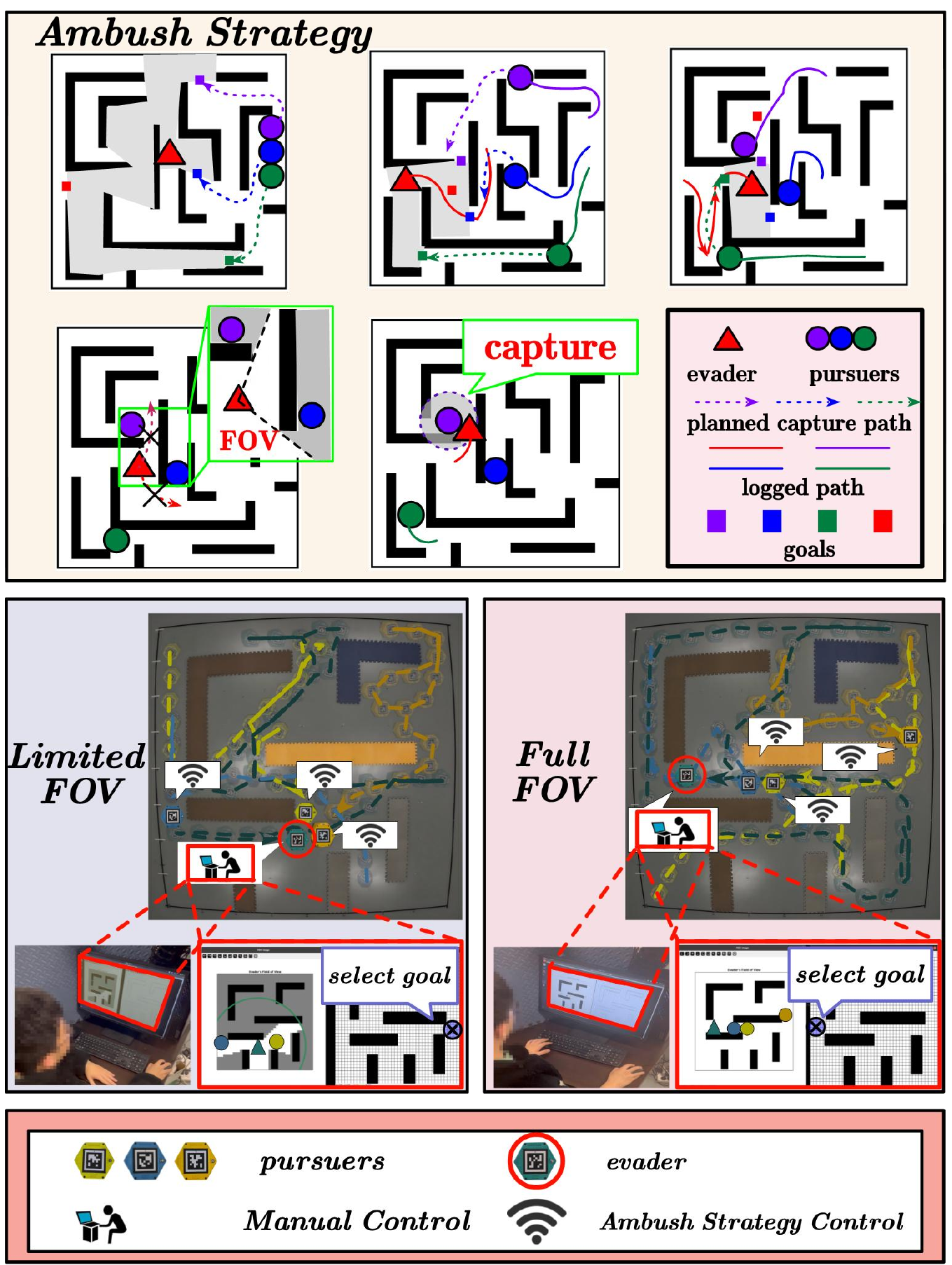}
    \caption{
      \textbf{Top}: Illustration of the proposed ambush strategy.
      The pursuer (in purple) hides in concealed positions
      and makes a surprise attack on the evader (in red)
      which is driven by other pursuers into the ambush area (gray region).
      \textbf{Bottom}: Hardware experiments
      against human-controlled evaders with limited or even full field-of-view.
    }
    \label{fig:ambush}
    \vspace{-5mm}
  \end{figure}

Starting from the classic formulation of ``cops and robbers''
in~\cite{fromme1984game},
a significant amount of work can be found to address the multi-agent
pursuit-evasion problem, e.g., solving a continuous differential
game~\cite{garcia2020multiple},
finding discrete movements over graphs~\cite{bonato2011game} or visibility roadmap~\cite{olsen2021visibility},
online minimization of the safe-reachable area of the evader~\cite{pierson2016intercepting},
adaptive voronoi partitioning~\cite{wang2023distributed,zhou2016cooperative},
behavior-based heuristic methods~\cite{janosov2017group},
and purely end-to-end methods based on reinforcement learning~\cite{lowe2017multi,vlahov2018developing}.
Despite the remarkable progress, most of the aforementioned works
cannot be generalized directly to arbitrary complex environments
without losing guarantees on the capture,
especially when facing faster evaders with different levels of intelligence,
even with human control.

\begin{figure*}[t!]
    \centering
    \includegraphics[width=0.95\linewidth,height=0.35\linewidth]{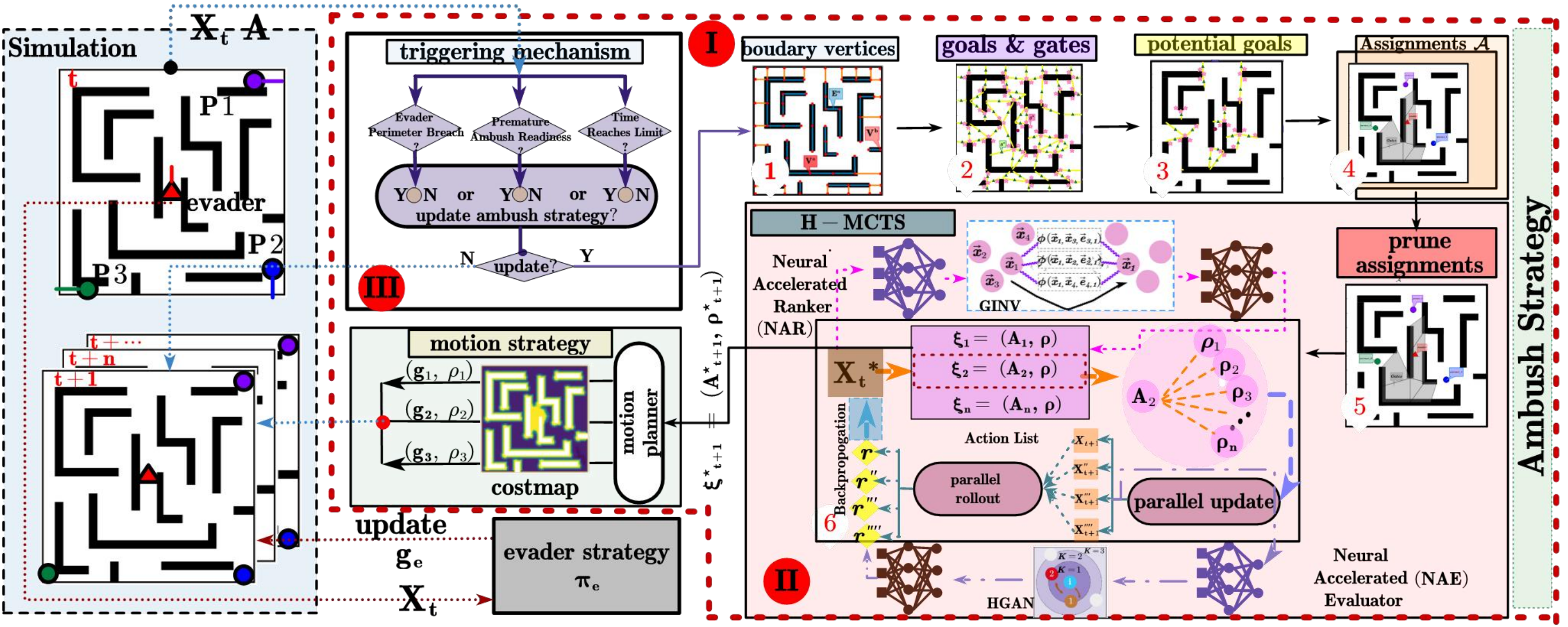}
    \caption{
      Overview of the proposed ambush framework,
      which (I) generates the role of pursuers as attacker or evader and their motion parameters (\textbf{right-top});
      (II) determines the motion policy under the parameters (\textbf{right-bottom});
      and (III) oversees the execution progress and reaction of the evader
      to trigger the adaptation scheme (\textbf{middle-top}).
    }
    \label{fig:ambush-framwork}
    \vspace{-5mm}
  \end{figure*}

As illustrated in Fig.~\ref{fig:ambush}, this work tackles the problem by introducing a coordinated ``ambush'' strategy, where concealed pursuers launch a surprise attack on a faster evader driven into the ambush zone. Unlike prior uses of the term~\cite{xu2025emergent, xu2025modeling}, our ambush focuses on hiding and sudden capture rather than interception along escape paths.
The strategy is realized through three integrated components (Fig.~\ref{fig:ambush-framwork}):  
First, topological analysis via projection and visibility graph minimization locates viable hiding and attack positions.  
Pursuers are assigned roles as hidden ``hiders'' or active ``attackers'' via a combinatorial assignment formulation.  
Then, a Hybrid Monte Carlo Tree Search (H-MCTS) dynamically optimizes role-goal combinations, coupled with motion strategies for precise execution.
To improve efficiency, 
a neural acceleration called Graph-\textbf{N}eur\textbf{a}l-Network-based \textbf{R}anker and \textbf{E}valuator (NARE) 
is trained offline to guide tree expansion and predict node rewards, replacing expensive rollouts.
Lastly, an online triggering mechanism monitors the evader's reactions, 
prompting re-planning when necessary to adapt to unexpected behaviors.
Theoretical analysis and large-scale simulations against strong baselines and human-controlled evaders validate the efficiency of the proposed framework in various complex scenarios.

Main contribution of this work are threefold:
(I) It provides a novel perspective on the problem
of dynamic capture in complex environments,
i.e., the single strategy of ambush is already highly effective;
(II) It designs a hybrid optimization framework that strategically
determines roles and target assignments for pursuers via an MCTS core,
while jointly optimizing their motion trajectories through the sampling of continuous parameters
to ensure successful ambush execution;
and (III) It proposes a general learning-based procedure
to accelerate the H-MCTS planning scheme,
which yields a $25\%$ increase in capture rate
and a $17.1\%$ decrease in capture time.
A similar procedure could be useful for other applications.

\section{Related Work} \label{sec:related}

\subsection{Pursuit in Free Space}
\label{subsec:related-free}

Early studies on pursuit--evasion games mainly focus on bounded,
obstacle-free environments~\cite{mu2023survey}.
Representative works include the geometric formulation in~\cite{garcia2020multiple},
which converts continuous pursuit into discrete combinatorial optimization,
the dimension-reduction method in~\cite{von2018pursuit} for saddle-point solutions,
and the pure pursuit policy in~\cite{li2019dimension} based on reducing the evader's generalized Voronoi partition.
Built on Isaacs' differential game theory~\cite{vajda1967differential},
these methods typically analyze pursuit strategies through the
Hamilton--Jacobi--Isaacs (HJI) equation~\cite{isaacs1999differential}.
However, they are largely restricted to simple obstacle-free domains
and often assume that pursuers are faster than the evader,
which is impractical in the real application.

To handle faster evaders, several extensions have been proposed.
The work in~\cite{li2022pursuit} extends~\cite{garcia2020multiple}
to arbitrary-dimensional spaces with slower pursuers.
Encirclement-based strategies are studied in~\cite{wang2021encirclement},
where slower pursuers are dynamically assigned around the evader,
and further combined with direct pursuit in~\cite{fang2020cooperative}
to reduce capture time and distance.
Other efforts include the game-theoretic utility tree search in~\cite{yang2022game},
the factor graph-based estimation and planning method in~\cite{esfahani2024fg},
and DRL-based collaborative capture approaches
~\cite{lowe2017multi,de2021decentralized, wang2020cooperative, gonultas2024learning, singh2020pursuit}.
Pursuit under limited sensing and environmental uncertainty has also been studied
in~\cite{bopardikar2008discrete}.
Nevertheless, most existing methods still rely on obstacle-free assumptions,
making non-convex or extended obstacles difficult to handle.

\subsection{Pursuit in Obstacle-Cluttered Space}
\label{subsec:related-obstacle}

Several studies have extended pursuit--evasion strategies to obstacle-cluttered environments
~\cite{wang2023distributed, tian2021distributed, rao2024decentralized, yan2024pursuit, ericsson2024pursuit}.
For example,~\cite{tian2021distributed} develops an obstacle-aware buffered Voronoi-cell policy
for capture and collision avoidance, but relies on smooth, radially symmetric obstacles.
The work in~\cite{rao2024decentralized} further handles non-smooth convex polygonal obstacles
in 2D and 3D scenarios.
However, such Voronoi-based policies may suffer from convergence issues
or suboptimal equilibria in environments with many nonconvex obstacles.
To address nonconvex polygonal obstacles,~\cite{wang2023distributed} combines
evader-centered Voronoi partitioning, encirclement, and navigation maps.
Nevertheless, this method is mainly suitable for small obstacles,
since large structures such as U-shaped corridors can prevent effective encirclement.
Probabilistic methods, e.g., entropy-based planning~\cite{hollinger2007probabilistic},
can reduce expected capture time in cluttered spaces,
but lack explicit mechanisms for handling complex nonconvex topology, corridors,
and faster evaders.

Different from geometric methods, recent RL-based approaches have also studied pursuit
in cluttered environments
~\cite{zhang2022game,chen2024multi, kouzeghar2023multi,zhang2022multi, qu2023pursuit}.
MADDPG-based methods are adopted in~\cite{zhang2022game,kouzeghar2023multi},
while~\cite{chen2024multi} proposes a MARL framework for unknown environments,
and~\cite{zhang2022multi} combines deep RL with artificial potential fields.
However, these methods still struggle with complex nonconvex obstacles.
Although~\cite{qu2023pursuit} considers USV pursuit in maritime environments with irregular obstacles,
it does not sufficiently evaluate varying speed ratios,
which limits its applicability when the evader is much faster.

Another related line considers pursuit--evasion under limited sensing or visibility constraints.
For instance,~\cite{wang2024viper} trains an end-to-end direct pursuit policy;
~\cite{sun2023toward} studies cooperative coevolution under limited visibility;
~\cite{sun2023matrixworld} provides the MatrixWorld platform;
and~\cite{durham2012distributed} enables distributed capture without global maps or mutual localization.
Although these works address realistic sensing limitations,
they often define capture as visual detection rather than physical interception with a capture radius.
Therefore, this paper does not focus on visibility-limited settings,
but on enabling a limited number of relatively slower pursuers to capture highly intelligent,
even human-controlled, evaders in complex scenarios.
\section{Problem Description}\label{sec:problem}
Consider a 2D and bounded polygon workspace~$\overline{\mathcal{\textbf{W}}} \subset \mathbb{R}^2$
as shown in Fig.~\ref{fig:ambush},
formed by~$\overline{K}\geq 3$ ordered vertices~$V_{\texttt{b}}\triangleq (\textbf{v}_{\texttt{b}}^1,
\textbf{v}_{\texttt{b}}^2, \cdots, \textbf{v}_{\texttt{b}}^{\overline{K}}) \subset \mathbb{R}^2$.
The workspace is cluttered with a set of polygon obstacles~$\mathcal{O}$,
each of which is defined as~$o \in   \mathcal{O}$ formed by~$K_{\texttt{o}}\geq 3$
ordered vertices~$V_{\texttt{o}} \triangleq (\textbf{v}^1_{\texttt{o}}, \textbf{v}^2_{\texttt{o}},
\cdots, \textbf{v}^{K_{\texttt{o}}}_{\texttt{o}}) \subset \overline{\mathcal{\textbf{W}}}$.
The free space is given by~$\mathcal{W}\triangleq \overline{\mathcal{\textbf{W}}} \backslash \mathcal{O}$.

Moreover, there are~$N_{\texttt{p}}>0$ pursuers and one evader moving
in the workspace,
each of which is modelled as a point-mass with a bounded maximum velocity.
Specifically, each pursuer \(i \in \mathcal{N}_{\texttt{p}} \triangleq \{1,\cdots,N_{\texttt{p}}\}\)
follows the first-order dynamics~\(\dot{\textbf{p}}_i(t) \triangleq \textbf{v}_i(t)\),
where \(\textbf{p}_i(t), \textbf{v}_i(t)\in \mathbb{R}^2\) are the 2D position and velocity,
with \(\|\textbf{v}_i(t)\| \leq \overline{v}_{{\texttt{p}}}\) being the maximum  velocity.
Similarly, the evader follows the same dynamics~\(\dot{\textbf{p}}_{\texttt{e}}(t) \triangleq \textbf{v}_{\texttt{e}}(t)\),
where \(\textbf{p}_{\texttt{e}}(t)\) and \(\textbf{v}_{\texttt{e}}(t)\) are the 2D
position and velocity, with \(\|\textbf{v}_{\texttt{e}}(t)\| \leq \overline{v}_{\texttt{e}}\)
being the maximum velocity.
In this work, the evader has a \emph{larger} velocity than the pursuers,
i.e., \(\overline{v}_{\texttt{p}} \leq  \overline{v}_{\texttt{e}}\).
In addition, all pursuers can communicate freely and are fully aware
of the complete workspace and the position of the evader.
Any pursuer can capture the evader if their line-of-sight (LOS)
distance is less than~$d_{\texttt{c}}>0$,
i.e., their LOS is clear and the relative distance is less than~$d_{\texttt{c}}$.
Once captured, the evader becomes immobile and the mission is accomplished.
On the other hand, the evader can only detect
a pursuer if their LOS distance is less than~$d_{\texttt{o}}>0$.

The overall objective is to design the control strategy for the pursuers,
such that the evader is captured as soon as possible.
It is worth noting that the behaviors of the evader are assumed to be unknown,
with representative ones described in the sequel.



\begin{remark}\label{rm:global-com}
  Global knowledge by the pursuers
  is a common assumption in the related work~\cite{wang2023distributed, fang2020cooperative, de2021decentralized}.
  This assumption can be relaxed by incorporating an exploration mechanism, 
  where pursuers search for the evader
  when it is not visible and switch to a capture mode upon detection,
  as discussed and validated in the subsequent sections.
\hfill $\blacksquare$
\end{remark}

\begin{figure}[t!]
  \centering
  \includegraphics[width=1.0\linewidth]{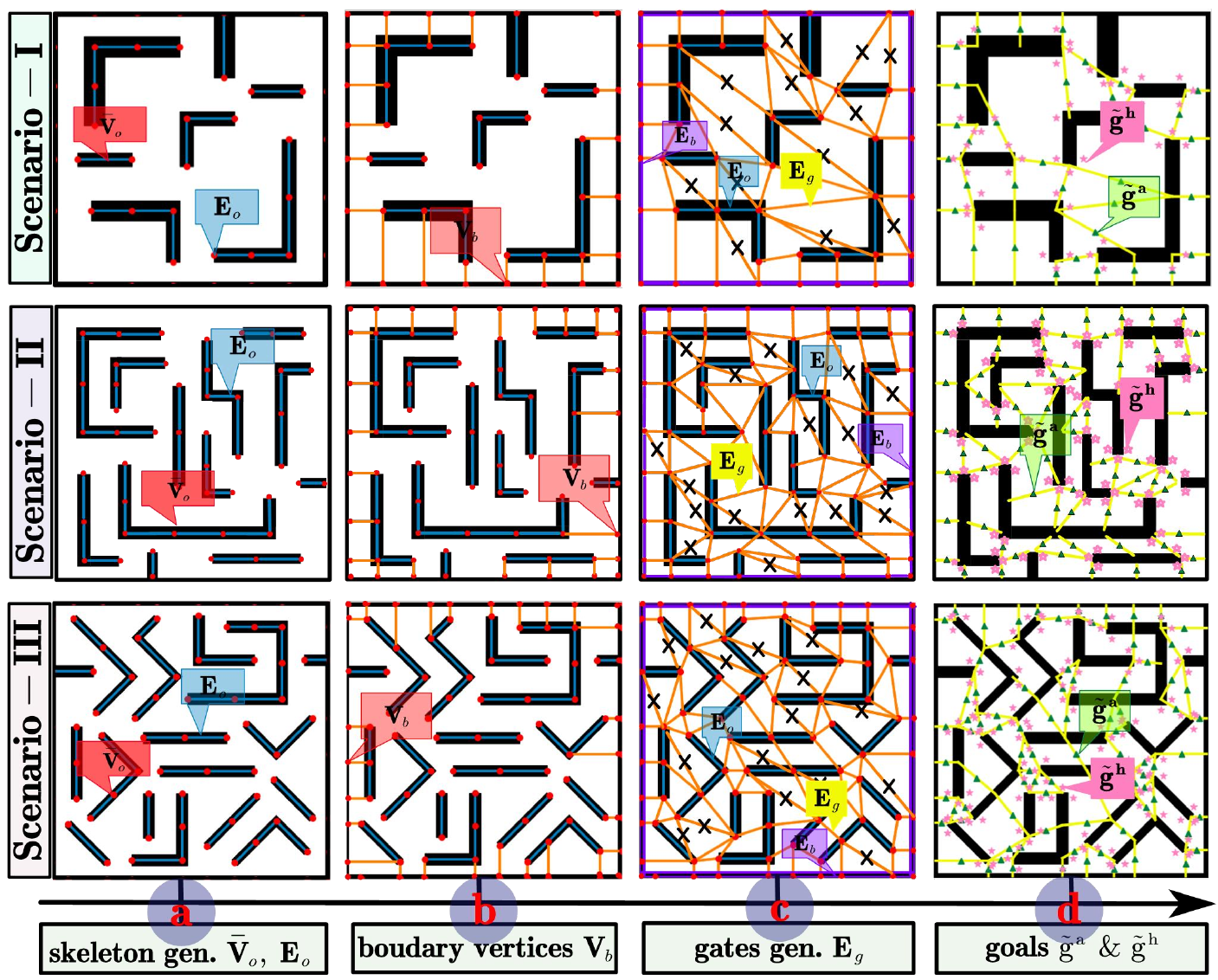}
  \vspace{-3mm}
  \caption{
Illustration of the goal generation and role assignment procedure in different workspaces. 
\textbf{(a)} Red points denote obstacle skeleton vertices $V_\texttt{o}$, and olive segments denote obstacle skeleton edges $E_\texttt{o}$. 
\textbf{(b)} Boundary vertices $V_\texttt{b}$ are obtained by projecting skeleton vertices onto the workspace boundary, where red boundary points denote the projected vertices and orange segments denote the projection lines. 
\textbf{(c)} Candidate gates are constructed between boundary vertices and skeleton vertices, where orange segments denote candidate gate edges $E_\texttt{g}$ and crossed segments indicate invalid gate connections. 
\textbf{(d)} Final pursuer goals are generated from valid gates and obstacle vertices: green triangles denote attacker goals $\tilde{\textbf{g}}^{\texttt{a}}$ at gate midpoints, pink stars denote hider goals $\tilde{\textbf{g}}^{\texttt{h}}$ 
near obstacle vertices, and yellow segments denote the selected valid gates.
}
  \label{fig:ambush-goal}
  \vspace{-4mm} 
\end{figure}

\section{Proposed Solution}\label{sec:solution}


\subsection{Parameterized Strategy of Ambush}\label{sec:ambush}


The ambush strategy partitions pursuers into hiders and attackers. Hiders conceal themselves near obstacles, out of the evader's sight, while attackers actively herd the evader toward these hidden positions for a surprise capture. First, a topological map is constructed by connecting the vertices of obstacles and boundaries. Key positions are then identified,
i.e., hiding spots near obstacles and strategic gates between them. 
Pursuers are assigned roles, i.e., 
hiders (\(\mathcal{N}_{\texttt{h}}\)) at corners while
others as attackers (\(\mathcal{N}_{\texttt{a}}\)) at gates. 
Finally, a cooperative movement strategy guides the attackers to encircle the evader using parameterized cost maps.

\subsubsection{Goal Generation and Role Assignment}\label{subsec:goal-assign}


As shown in Fig.~\ref{fig:ambush-goal}, 
we present a geometric approach to goal generation and role
assignment, derived from structured collision-free connections in the environment.
The terrain is first partitioned into discrete regions according to the obstacle configuration. By connecting vertices of different obstacles, a set of critical edges is obtained, which represent strategic locations for attackers to establish an enclosing formation. Meanwhile, positions near obstacle vertices are treated as candidate goals for hiders, since they naturally support concealment and ambush behaviors. Assigning pursuers to these candidate locations yields different assignments, and each assignment further defines an ambush region shown in Fig.~\ref{fig:capture-graph}, i.e., the enclosing capture graph considered in the subsequent analysis.
To this end, we introduce the following concepts
used in our method: \emph{gate},
\emph{pursuer goals}, \emph{assignment}, and \emph{capture graph}.

\begin{figure}[!t]
  \centering
  \includegraphics[width=0.99\linewidth]{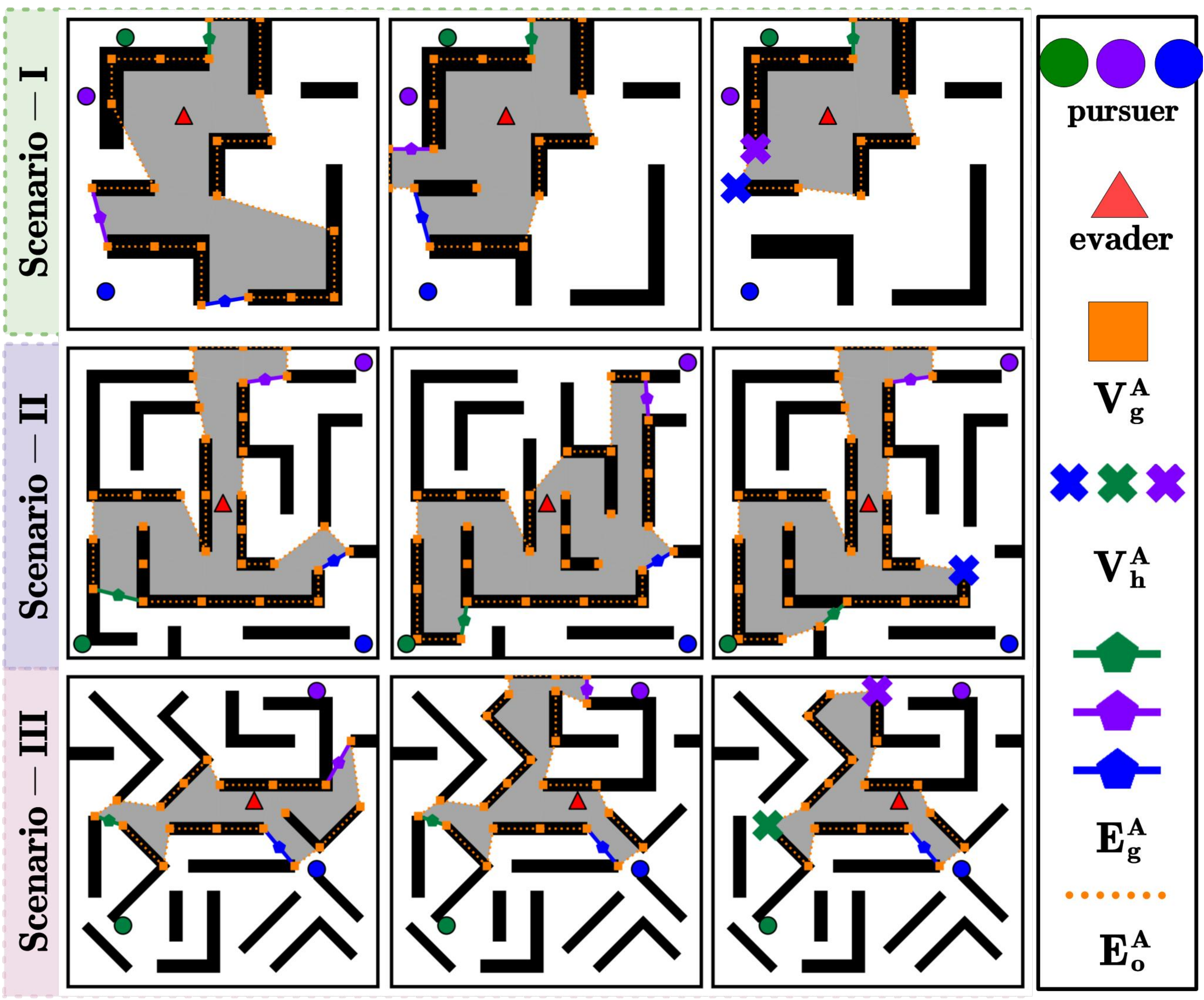}
  \caption{
  Examples of the capture graphs $G_A$ in Def.~\ref{def:graph} in different workspaces. 
  The generated capture graph forms a polygonal enclosure around the evader. 
  Green, purple, and blue circles denote pursuers, and the red triangle denotes the evader. 
  Orange points and segments denote obstacle skeleton vertices and edges, respectively. 
  Cross markers indicate hider goals, while star markers indicate attacker goals.
  }
  \label{fig:capture-graph}
  \vspace{-5mm}
\end{figure}

\begin{definition}(Gate)\label{def:gate}
  A \emph{gate}~$e_{\texttt{g}}$ 
  which is used to form the boundary of an enclosing area for the ambush strategy,
  is the obstacle-avoiding line segment connecting specific vertices in the workspace.
  A line segment is valid only if its
open segment is contained in $\mathcal{W}$ and does not intersect
the interior or boundary of any obstacle, except for admissible
endpoint contacts with workspace-boundary vertices. Specifically,
a valid gate connects either:
  (i) a boundary vertex $\textbf{v}_{\texttt{b}} \in V_{\texttt{b}}$ to an obstacle skeleton vertex $\overline{\textbf{v}}_{\texttt{o}} \in \overline{V}_{\texttt{o}}$, or
  (ii) two skeleton vertices from different obstacles $\overline{\textbf{v}}_{\texttt{o}}^{1}, \overline{\textbf{v}}_{\texttt{o}}^{2} \in \overline{V}_{\texttt{o}}$.
  The set of all gates is denoted as $E_{\texttt{g}} \triangleq \{e_{\texttt{g}}\}$,
  which are used to define the goals for pursuers in the ambush strategy.
  \hfill $\blacksquare$
\end{definition}

\begin{definition}(Goals of Pursuers)\label{def:goals}
  The goals of pursuers $\tilde{\textbf{g}} \triangleq \tilde{\textbf{g}}^{\texttt{a}} \cup \tilde{\textbf{g}}^{\texttt{h}}$
  used to guard the boundary to herd the evader and prevent the evader from escaping from the ambush region,
  and hider goals $\tilde{\textbf{g}}^{\texttt{h}}$ used to capture the evader at concealed locations.
  The attacker goals $\tilde{\textbf{g}}^{\texttt{a}} \subset \mathbb{R}^2$ are defined as the midpoints of gates:
  \begin{equation}\label{eq:attacker-goals}
    \tilde{\textbf{g}}^{\texttt{a}} \triangleq \left\{ \frac{\textbf{v}_i + \textbf{v}_j}{2} \mid (\textbf{v}_i, \textbf{v}_j) \in E_{\texttt{g}} \right\},
  \end{equation}
  where $E_{\texttt{g}}$ is the set of gates defined above,
  with each edge $e_{\texttt{g}}$ defined as an ordered pair~$(\textbf{v}_i, \textbf{v}_j)$ of vertices.
  The hider goals $\tilde{\textbf{g}}^{\texttt{h}} \subset \mathbb{R}^2$ 
  are flanking points near the obstacle vertices:
  \begin{equation}\label{eq:hider-goals}
    \tilde{\textbf{g}}^{\texttt{h}} \triangleq \left\{ \textbf{v}_{\texttt{o}} + \delta \hat{e}_k \mid \textbf{v}_{\texttt{o}} \in V_{\texttt{o}},\ k \in \{1, 2\} \right\},
  \end{equation}
  where $\hat{\mathbf{e}}_1$ and $\hat{\mathbf{e}}_2$ are the normalized edge-tangent directions from 
$\mathbf{v}_{\texttt{o}}$ to its two adjacent vertices in the ordered obstacle polygon, and each generated candidate is projected to the free-space side of the obstacle boundary with a small safety clearance if necessary. 
The flanking distance is chosen as $\delta=\min\{\eta l_{\texttt{o}},\delta_{\max}\}$, where $l_{\texttt{o}}$ is the length of the corresponding adjacent obstacle edge, $0<\eta<1$ is fixed in all experiments,
 and $\delta_{\max}$ prevents the hider goal from being placed too far from the obstacle corner.
  \hfill $\blacksquare$
\end{definition}

\begin{definition}(Assignment)\label{def:assignment}
  The \emph{assignment}, 
  which delineates the allocation of pursuers to distinct goal positions and their corresponding roles to establish an enclosing formation,
  denoted as $A \in \mathcal{A}$,
  which is the set of assignments,
  is a 3-tuple:
 \begin{equation}\label{eq:assignment}
  A \triangleq \left\{ (i, \textbf{g}_i, r_i) \mid \forall i \in \mathcal{N}_{\texttt{p}} \right\},
\end{equation}
where~$i$ is a pursuer index from $\mathcal{N}_{\texttt{p}} \triangleq \{1,\cdots,N_{\texttt{p}}\}$,
  $\textbf{g}_i \in \tilde{g}$ is the goal position,
  $r_i$ is the role determined by goal type, i.e.,
  each $r_i$ is $\texttt{attacker}$ if $\textbf{g}_i \in \tilde{g}^{\texttt{a}}$ or $\texttt{hider}$ if $\textbf{g}_i \in \tilde{g}^{\texttt{h}}$,
  and satisfies constraints that all $\textbf{g}_i$ are distinct. 
\hfill $\blacksquare$
\end{definition}

For each assignment \(A \in \mathcal{A}\),
we introduce a key concept: the \textit{Capture Graph} as illustrated in Fig.~\ref{fig:capture-graph}.
This graph  
essentially representing the enclosing region formed by the assignment,
serves as a fundamental construct in our subsequent theoretical analysis
for determining the success conditions.

\begin{definition}(Capture Graph)\label{def:graph}
For an assignment $A \in \mathcal{A}$,
the \emph{capture graph} $G_A \triangleq (V_A, E_A)$ is a geometric graph embedded in $\mathcal{W}$
that forms a closed polygonal enclosure.
The vertex set comprises gate endpoints~$V_{\texttt{g}}^A$ and obstacle skeleton vertices~$V_{\texttt{h}}^A$,
while the edge set consists of gates~$E_{\texttt{g}}^A$ and obstacle skeleton segments~$E_{\texttt{o}}^A$.
\hfill $\blacksquare$
\end{definition}
The specific definitions of these vertices and edges are as follows.
The gate endpoints $V_{\texttt{g}}^A$ comprise obstacle skeleton vertices
that are endpoints of gates in $E_{\texttt{g}}^A$,
formally $V_{\texttt{g}}^A \triangleq \{  \overline{\textbf{v}}_{\texttt{o}} \in \overline{V}_{\texttt{o}} \mid \overline{v}_{\texttt{o}} \text{ is an endpoint of } e_{\texttt{g}} \in E_{\texttt{g}}^A \}$.
The obstacle skeleton vertices $V_{\texttt{h}}^A$
consist of vertices adjacent to hiders' goal points in $A$,
which are defined as $V_{\texttt{h}}^A \triangleq \left\{ \overline{v}_{\texttt{o}} \in \overline{V}_{\texttt{o}} \mid \exists\, (i, \textbf{g}_i, \texttt{hider}) \in A,\ \operatorname{adj}(\textbf{g}_i,  \overline{\textbf{v}}_{\texttt{o}}) \right\}$
where $\operatorname{adj}(g, v)$ indicates the adjacency between $g$ and~$v$.
The gates $E_{\texttt{g}}^A \triangleq \left\{ e_{\texttt{g}} = (\textbf{v}_1, \textbf{v}_2) \in E_{\texttt{g}} \mid \exists\, (i, \textbf{g}_i, r_i) \in A,\ \textbf{g}_i = \frac{\textbf{v}_1 + \textbf{v}_2}{2} \right\}$
are formed by these gates,
midpoints of which are in pursuers' assignments,
with $E_{\texttt{g}}$ denoting all possible gate e dges.
Finally, the obstacle skeleton segments $E_{\texttt{o}}^A$
connect consecutive vertices in $V_A$ along obstacle skeletons.

\begin{algorithm}[!t]
  \caption{Goal Generation and Role Assignment: $\texttt{Goals\&Assigns}(\cdot)$}
  \label{alg:goals_assigns}
  \KwIn{Workspace $\mathcal{W}$, obstacles~$V_{\texttt{o}}$, pursuers~$\mathcal{N}_{\texttt{p}}$.}
  \KwOut{Assignments $\mathcal{A}$.}

  \SetKwFunction{FCombAgn}{CombAgn}
  \SetKwProg{Fn}{Function}{:}{}
  \Fn{\FCombAgn{$\mathcal{N}, G$}}{
      \If{$\mathcal{N} = \emptyset$}{
          \Return $\emptyset$ \;  \label{alg-line:base-case}
      }

      $i \gets \text{first element in } \mathcal{N}$ \; \label{alg-line:first-pursuer}
      $\mathcal{N}' \gets \mathcal{N} \setminus \{i\}$ \; \label{alg-line:remaining-pursuers}
      $\mathcal{A} \gets \emptyset$ \; \label{alg-line:init-assignment-set}
      \ForEach{$\textbf{g} \in G$}{ \label{alg-line:goal-loop-start}
          $G' \gets G \setminus \{\textbf{g}\}$ \; \label{alg-line:remaining-goals}
          $\mathcal{A}_{\texttt{sub}} \gets \FCombAgn(\mathcal{N}', G')$ \; \label{alg-line:recursive-call}

          \ForEach{assignment $A' \in \mathcal{A}_{\texttt{sub}}$}{ \label{alg-line:sub-assignment-loop}
              $r_i \gets \begin{cases}
                  \texttt{attacker} & \text{if } \textbf{g} \in \tilde{g}^{\texttt{a}} \\
                  \texttt{hider} & \text{otherwise}
              \end{cases}$ \; \label{alg-line:role-determination}
              $\mathcal{A} \gets \mathcal{A} \cup \{ A' \cup \{(i, \textbf{g}, r_i)\} \}$ \; \label{alg-line:combine-assignment}
          }
      }
      \Return $\mathcal{A}$ \; \label{alg-line:return-assignments}
  }

  $\overline{V}_{\texttt{o}} \gets \texttt{SktVet}(V_{\texttt{o}})$ \; \label{alg-line:skeleton}
  $V_{\texttt{b}} \gets \texttt{BodVet}(\mathcal{W}, \overline{V}_{\texttt{o}})$ \; \label{alg-line:boundary}

  $E_{\texttt{g}} \gets \texttt{BudGat}(V_{\texttt{b}}, \overline{V}_{\texttt{o}})$ \; \label{alg-line:gates}
  $E_{\texttt{g}} \gets \texttt{ConEdg} (E_{\texttt{g}})$ \; \label{alg-line:constraints}

  $\tilde{g}^{\texttt{a}} \gets \texttt{AttGol}(E_{\texttt{g}})$ by~\eqref{eq:attacker-goals} \; \label{alg-line:attacker_goals}
  $\tilde{g}^{\texttt{h}} \gets \texttt{HidGol}(V_{\texttt{o}})$ by~\eqref{eq:hider-goals} \;\label{alg-line:hider_goals}
  $\tilde{g} \gets \tilde{g}^{\texttt{a}} \cup \tilde{g}^{\texttt{h}}$ \; \label{alg-line:combine_goals}


  $\mathcal{A} \gets \FCombAgn(\mathcal{N}_{\texttt{p}}, \tilde{g})$ \; \label{alg-line:call-combinatorial}

  \Return{$\mathcal{A}$} \; \label{alg-line:return}
\end{algorithm}

\begin{problem}\label{prb:assignments}
Given a workspace $\mathcal{W}$ with obstacle vertices $V_{\texttt{o}}$ and a set of pursuers $\mathcal{N}_{\texttt{p}}$,
the problem is to generate a set of candidate assignments $\mathcal{A}$,
where each assignment specifies a unique goal and role for every pursuer.
\hfill $\blacksquare$
\end{problem}

We propose an efficient goal generation and role assignment algorithm, $\texttt{Goals\&Assigns}(\cdot)$, to produce candidate assignments for pursuers. As outlined in Alg.~\ref{alg:goals_assigns} and illustrated in Fig.~\ref{fig:ambush-goal}, the algorithm first extracts obstacle skeleton vertices $\overline{V}_{\texttt{o}}$ via medial-axis transformation in Line~\ref{alg-line:skeleton}, as shown in Fig.~\ref{fig:ambush-goal}(a). Boundary vertices $V_{\texttt{b}}$ are then generated by orthogonally projecting $\overline{V}_{\texttt{o}}$ onto workspace edges (Line~\ref{alg-line:boundary}), retaining only projections that do not intersect other obstacles in Fig.~\ref{fig:ambush-goal}(b). Gates $E_{\texttt{g}}$ are constructed as collision‑free connections between $V_{\texttt{b}}$ and $\overline{V}_{\texttt{o}}$ according to Def.~\ref{def:gate} (Line~\ref{alg-line:gates}) in Fig.~\ref{fig:ambush-goal}(c).
Geometric constraints are applied to reduce the combinatorial complexity (Line~\ref{alg-line:constraints}). An optimal subset of gates is then selected by solving the integer program:
\[
\begin{aligned}
\min_{\mathbf{x}} & \sum_{(\textbf{v}_i,\textbf{v}_j) \in E_{\texttt{g}}} \|\textbf{v}_i - \textbf{v}_j\| \cdot x_{ij} + \lambda \sum_{(\textbf{v}_i,\textbf{v}_j) \in E_{\texttt{g}}} (1 - x_{ij}) \\
\text{s.t.} \quad & \sum_{\textbf{v}_j} x_{ij} \leq 1,\; x_{ij} = x_{ji}, \quad \forall (\textbf{v}_i,\textbf{v}_j) \in E_{\texttt{g}},
\end{aligned}
\]
where $\mathbf{x} = \{x_{ij}\}$ is a binary decision vector indicating whether gate $(\textbf{v}_i,\textbf{v}_j)$ is selected, $\|\textbf{v}_i-\textbf{v}_j\|$ is the Euclidean distance, and $\lambda>0$ is a penalty parameter. The problem is solved with GUROBI~\cite{gurobi}, yielding the optimal gate set $E_{\texttt{g}}$.
The gate-selection program is a topology-aware candidate generation step rather than a direct optimizer of capture success.
Each valid gate represents a collision-free cross-section of the free space, whose midpoint provides an attacker goal for blocking a potential escape passage.
The distance term favors short and easily guarded gates, while the penalty on unselected gates preserves sufficient topological connections for forming feasible capture graphs before Sec.~\ref{sec:mcts} further optimizes the final assignment and motion parameters.

As Fig.~\ref{fig:ambush-goal}(d) illustrates, goal positions are generated as follows: attacker goals $\tilde{g}^{\texttt{a}}$ are computed as gate midpoints (Line~\ref{alg-line:attacker_goals}), while hider goals $\tilde{g}^{\texttt{h}}$ are placed near obstacle vertices (Line~\ref{alg-line:hider_goals}). The final role assignments are generated by the recursive function $\texttt{CombinatorialAssign}(\mathcal{N}, G)$ (Line~\ref{alg-line:call-combinatorial}). The function first handles the base case (Line~\ref{alg-line:base-case}), then for the first pursuer $i \in \mathcal{N}$ (Line~\ref{alg-line:first-pursuer}), it iterates over each available goal $\textbf{g} \in G$ (Line~\ref{alg-line:goal-loop-start}), computes the remaining goals $G'$ (Line~\ref{alg-line:remaining-goals}), and recursively solves for the remaining pursuers $\mathcal{N}'$ (Line~\ref{alg-line:recursive-call}). Each returned sub-assignment $A'$ (Line~\ref{alg-line:sub-assignment-loop}) is extended with a role $r_i$ (attacker if $\textbf{g} \in \tilde{g}^{\texttt{a}}$, otherwise hider) (Line~\ref{alg-line:combine-assignment}). The function returns all valid assignments $\mathcal{A}$ (Line~\ref{alg-line:return-assignments}), which satisfy distinct goal assignments in Def.~\ref{def:assignment}, enabling pursuers to execute their designated roles.
Since the evader's future response is unknown, these assignments are worth retaining as admissible candidates, while their actual quality is evaluated later by the following section.

\subsubsection{Motion Strategy}
\label{subsec:move-strategy}

Given feasible assignments $\mathcal{A}$ that define each pursuer's goal and role, a two-stage motion strategy is employed. First, pursuers converge toward assigned goals $\textbf{g}^i \in \tilde{g}$ via coordinated paths to block the early escape. Upon arrival, the second phase starts,
i.e., hiders $\mathcal{N}_{\texttt{h}}$ remain static until the evader nears the hider gate, prompting a sudden attack;
attackers $\mathcal{N}_{\texttt{a}}$ sweep to herd the evader toward the gates.
Due to space limitations, the detailed implementation of the two-stage motion strategy is provided in Appendix~\ref{sec:app-motion}.

\begin{figure}[t!]
  \centering
  \includegraphics[width=1.0\linewidth]{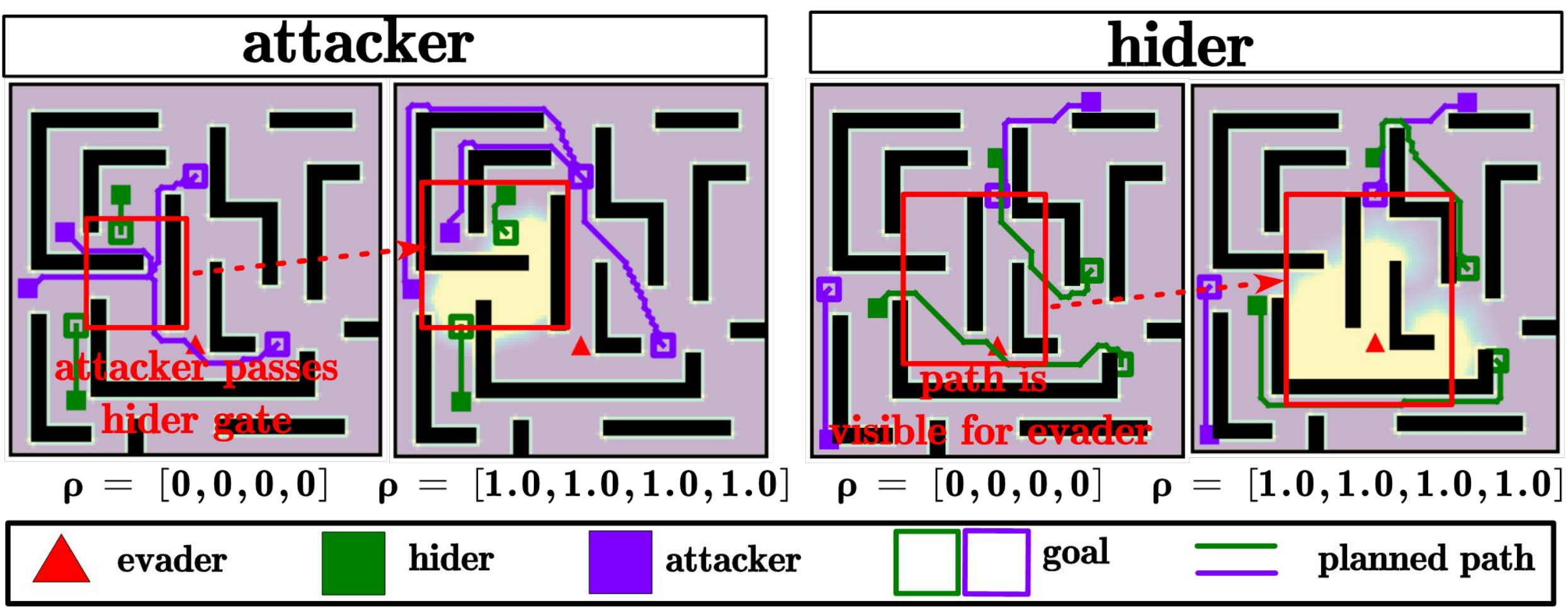}
  \caption{
    The overall ambush strategy under different
    parameters~$\xi$,
    including the role assignment~$A$ and
    motion parameters~$\boldsymbol{\rho}$,
    and their impact on the pursuers' trajectories.
  }
  \label{fig:rho-traj}
  \vspace{-6mm}
\end{figure}

\begin{definition}(Ambush Strategy Parameters)\label{def:ambush-params}
The \emph{ambush strategy} is fully characterized
by the parameter tuple $\xi \triangleq (A, \boldsymbol{\rho})$,
where $A \in \mathcal{A}$
and $\boldsymbol{\rho} = \{\rho^i\}_{i\in\mathcal{N}_{\texttt{p}}}$
is the set of coefficients for the motion strategy as shown in Fig.~\ref{fig:rho-traj}.
The parameter space $\Xi \triangleq \mathcal{A} \times \mathbb{R}_+^{N_{\texttt{p}}}$
encompasses all feasible configurations of the ambush strategy,
hence $\xi \in \Xi$.
\hfill $\blacksquare$
\end{definition}
\subsection{Hybrid Monte-Carlo Tree Search}
\label{sec:mcts}

Each ambush strategy call yields a parameter set~$\Xi$, where each $\xi \in \Xi$ affects the capture success. The goal is to find the optimal $\xi^\star \in \Xi$ that maximizes capture probability under the current state. Unlike prior works limited to discrete search~\cite{xu2022belief, ren2025holistically, king2019decentralized}, this is a hybrid optimization: assignment selection $A \in \widehat{\mathcal{A}}$ is discrete, while the corresponding motion coefficients $\boldsymbol{\rho} \subset \mathbb{R}_+^{\mathcal{N}_{\texttt{p}}}$ are continuous, and the evader’s strategy is dynamic and unknown. To address this, we propose a Hybrid Monte‑Carlo Tree Search (H‑MCTS) algorithm to explore the parameter space and determine $\xi^\star \triangleq (A^\star, \boldsymbol{\rho}^\star)$. 
The specific steps of the H-MCTS algorithm are detailed below.

\subsubsection{Problem of Hybrid Optimization}
\label{subsec:hybrid-opt}
Let $\mathbf{X}_t = (\{\mathbf{p}_i^t\}_{i\in\mathcal{N}_{\texttt{p}}}, \mathbf{p}_{\texttt{e}}^t)$ be the joint system state at time $t$, comprising all pursuer positions and the evader position. The system dynamics follow $\mathbf{X}_{t+1} = f(\mathbf{X}_t, \xi, \pi_{\texttt{e}})$, where $f(\cdot)$ integrates the motion strategies in~\eqref{eq:cost-map} and the evader's response.
Given the current time $t_n$ and future discrete steps $\{t_k\}_{k=n}^{n+\Delta}$ ($\Delta > 0$ is the maximum capture duration), the capture time $T_{\texttt{cap}}$ is defined as the earliest time when the evader is within a distance $d_{\texttt{c}}$ of any pursuer:
$
T_{\texttt{cap}} \triangleq \min \left\{ t \geq t_n \mid \exists i \in \mathcal{N}_{\texttt{p}} : \|\mathbf{p}_i^t - \mathbf{p}_{\texttt{e}}^t\| \leq d_{\texttt{c}} \right\}.
$
The capture indicator $\mathbb{I}_{\texttt{cap}}$ equals 1 if capture occurs within $[t_n, t_{n+\Delta}]$, and 0 otherwise.

\begin{problem}\label{prb:hybrid_opt}
Determine the optimal ambush parameter $\xi^\star \triangleq (A^\star, \boldsymbol{\rho}^\star)$
in a time window~$[t_n, t_{n+\Delta}]$ that maximizes:
\begin{equation}\label{eq:dual_objective}
\begin{aligned}
\underset{{\xi \in \Xi}}{\textbf{max}} \quad & \mathbb{E}_{\pi_{\texttt{e}}} \Bigg[
\mathbb{I}_{\texttt{cap}} \left(1 - \frac{T_{\texttt{cap}} - t_n}{t_{n+\Delta}-t_n }\right) 
+ (1 - \mathbb{I}_{\texttt{cap}}) \left(1 - \frac{S(\xi)}{S_{\texttt{max}}}\right)\Bigg] \\
\textbf{s.t.} \quad
& \mathbf{X}_{t_{k+1}} = f(\mathbf{X}_{t_k}, \xi, \pi_{\texttt{e}}), \quad k = n, \cdots, n + \Delta,
\end{aligned}
\end{equation}
where~$\mathbb{E}_{\pi_{\texttt{e}}}[\cdot]$ denotes expectation over the evader's strategy $\pi_{\texttt{e}}$,
$S(\xi)$ denotes the area of the capture graph associated with the assignment $A$,
and $S_{\texttt{max}}$ is the total area of the workspace.
\hfill $\blacksquare$
\end{problem}

\subsubsection{H-MCTS Method Overview}
\label{subsec:H-MCTS}

We propose a Hybrid Monte-Carlo Tree Search (H-MCTS) algorithm to solve this hybrid decision-making problem. The Alg.~\ref{alg:H-MCTS} builds a search tree $\mathcal{T} \triangleq (\mathcal{V}, \mathcal{E})$
where each node $\nu \in \mathcal{V}$ stores the system state $\mathbf{X}$ and four attributes: the cumulative reward $\nu.Q$, the visit count $\nu.N$, a parent pointer $\nu.\texttt{Parent}$, and a set of child nodes $\nu.\texttt{Children}$. Each edge in $\mathcal{E}$ is labeled with an ambush parameter $\xi = (A, \boldsymbol{\rho})$, and $\Psi(\nu)$ retrieves the node state.
The tree is initialized (lines~\ref{step:init_node}--\ref{step:init_tree}) with a root node $\nu_0$ whose state is $\mathbf{X}_{t_n}$, parent pointer is null, and child set is $\emptyset$.
A schematic illustration of the H-MCTS procedure is provided in Appendix~\ref{sec:app-hmcts} to complement the algorithmic description.

\textbf{Selection Phase.}
Starting from the root node $\nu_0$, the selection phase follows the standard UCT rule to recursively choose the most promising child until a leaf node or a node with unexpanded assignments is reached (Line~\ref{step:selection_condition}).

\textbf{Expansion Phase.}
When encountering an expandable node $\nu$ with unselected assignments, the algorithm first randomly selects an assignment $A$ from $\widehat{\mathcal{A}}$ as Line~\ref{step:sample_gate}. A candidate set of motion coefficients $\widehat{\boldsymbol{\rho}}_+$ is generated by discretizing each component $\rho^i \in [0,1]$ with a step $\Delta\rho>0$ in Line~\ref{step:discrete}.
To manage combinatorial explosion, a downsampling procedure selects high-potential candidates while preserving diversity in Line~\ref{step:perturb}. Each candidate $\boldsymbol{\rho}^k_+$ is evaluated by a cost function:
$$\Phi(\boldsymbol{\rho}^k_+) \triangleq \sum_{i=1}^{\mathcal{N}_{\texttt{p}}} \Big( \sum_{k=1}^{L^i(\boldsymbol{\rho}^k_+)} \sigma_{\texttt{vis}}(\mathbf{m}_k^i(\boldsymbol{\rho}^k_+), \mathbf{p}_{\texttt{e}})
 + \beta  C_{\texttt{pth}}(\boldsymbol{\rho}^k_+) \Big),$$
where $\mathbf{m}_k^i(\boldsymbol{\rho}^k_+)$ denotes the center coordinates of the $k$-th grid cell
along $\texttt{A}^\star$-generated paths over the cost map $\mathcal{M}^i$ in~\eqref{eq:cost-map};
$L^i(\boldsymbol{\rho}^k_+)$ is the path length in grid cells;
$\sigma_{\texttt{vis}}(\mathbf{m}_k^i(\boldsymbol{\rho}^k_+), \mathbf{p}_{\texttt{e}})$ is a binary visibility risk indicator,
i.e., 1 if visible from $\mathbf{p}_{\texttt{e}}$, 0 otherwise;
and $C_{\texttt{pth}}(\boldsymbol{\rho}^k_+)$ represents the cumulative navigation cost.
Candidates are sorted by $\Phi$, partitioned into $K$ segments, and the top $M$ from each are retained, forming $\widetilde{\boldsymbol{\rho}}_+$.
For each $\boldsymbol{\rho}_+^k \in \widetilde{\boldsymbol{\rho}}_+$, the ambush parameter $\xi = (A, \boldsymbol{\rho}_+^k)$ is formed. An evader action $a^{\texttt{e}}$ is sampled from $\hat{\pi}_{\texttt{e}}$ in Line~\ref{step:sample_action}, and the successor state $\mathbf{X}^+ = f(\Psi(\nu), \xi, a^{\texttt{e}})$ is computed in Line~\ref{step:state_transition}. A new node $\nu^+$ with state $\mathbf{X}^+$ is created and added as a child of $\nu$, with the edge labeled by $\xi$ in Line~\ref{step:add_node}. If no unselected assignments remain, the algorithm proceeds to selecting the best child as seen in Line~\ref{step:no_expand}.

\begin{algorithm}[!t]
  \caption{Hybrid Monte Carlo Tree Search algorithm~$\texttt{H-MCTS}(\cdot)$}
  \label{alg:H-MCTS}
  \SetKwInOut{Input}{Input}
  \SetKwInOut{Output}{Output}
  \Input{Current state $\mathbf{X}_{t_n}$, assignment set $\widehat{\mathcal{A}}$,
         assumed evader policy $\hat{\pi}_{\texttt{e}}$,
         computational budget $B$}
  \Output{Optimal ambush parameters $\xi^\star = (A^\star, \boldsymbol{\rho}^\star)$}

  \BlankLine

   \tcp{\textbf{Initialize the search tree}}
  $\nu_0 \gets (\mathbf{X}_{t_n}), \nu_0.N \gets 0, \nu_0.Q \gets 0;$ \label{step:init_node};\\
  $\mathcal{T} \gets \texttt{IntTree}(\nu_0)$ \label{step:init_tree};\\
  $iter \gets 0, \nu \gets \nu_0$ \label{step:init_iter};\\

  \While{$iter < B$ \textbf{and} $\neg \texttt{IsTerminal}(\nu)$ \label{step:loop_condition}}{
    \tcp{\textbf{Selection}}
    \While{$\texttt{HasChildren}(\nu)$ \textbf{and} $\neg \texttt{IsLeaf}(\nu)$ \label{step:selection_condition}}{
      $\nu \gets \underset{\nu' \in \texttt{Children}(\nu)}{\textbf{argmax}} \texttt{UCTS}(\nu', \nu_0)$ \label{step:uct_selection} \;
    }

    \tcp{\textbf{Expansion}}
    \eIf{$\texttt{HasUnSelAgn}(\nu, \widehat{\mathcal{A}})$ \label{step:expand_condition}}{
      $A \gets \texttt{SplUnSelGate}(\nu, \widehat{\mathcal{A}})$ \label{step:sample_gate} \;
      $\widehat{\boldsymbol{\rho}}_+ \gets \texttt{Sample}(\boldsymbol{\rho})$ \label{step:discrete};\\
      $\widetilde{\boldsymbol{\rho}}_+ \gets \texttt{Downsample}(\widehat{\boldsymbol{\rho}}_+)$ \label{step:perturb}; \\
      \ForEach{$\boldsymbol{\rho}_+^k \in \widehat{\boldsymbol{\rho}}_+ \textbf{in parallel}$}{
        $\xi \gets (A, \boldsymbol{\rho}_+^k)$,
        $a^{\texttt{e}} \sim \hat{\pi}_{\texttt{e}}(\Psi(\nu))$ \label{step:sample_action} \;
        $\mathbf{X}^+ \gets f(\Psi(\nu), \xi, a^{\texttt{e}})$ \label{step:state_transition} \;
        $\nu^+ \gets \texttt{CteNode}(\mathbf{X}^+)$ \label{step:create_node} \;
        $\mathcal{T} \gets \texttt{AddNode}(\mathcal{T}, \nu, \nu^+, \xi)$ \label{step:add_node} \;
      \tcp{\textbf{Simulation}}
      $r \gets \texttt{RollSml}(\nu^+, \xi, \hat{\pi}_{\texttt{e}})$ \label{step:rollout} \;
      \tcp{\textbf{Back propagation}}
      $\texttt{Backpropagation}(\nu^+)$ \label{step:back} \;
      }
    }{
      $\nu \gets \textbf{argmax}_{\nu^- \in \texttt{Children}(\nu)}Q(\nu^-)/N(\nu^-)$ \label{step:no_expand} \;
    }

    $iter \gets iter + 1$ \label{step:increment_iter}
  }

  \tcp{\textbf{Termination}}
  $\nu^\star \gets \underset{\nu \in \texttt{Children}(\nu_0)}{\textbf{argmax}} \, Q(\nu)/N(\nu)$ \label{step:select_best} \;
  $\xi^\star \gets \texttt{GetAction}(\mathcal{T}, \nu_0, \nu^\star)$ \label{step:get_action} \;
  \Return{$\xi^\star$} \label{step:return}
\end{algorithm}


\textbf{Simulation Phase.}
Unlike standard rollout policies in~\cite{kocsis2006bandit}, our method adopts a greedy selection strategy, i.e.,
the candidate $\boldsymbol{\rho}^k_+$ with the highest $\Phi(\boldsymbol{\rho}^k_+)$ from $\widetilde{\boldsymbol{\rho}}_+$ is chosen for simulation. Starting from node $\nu^+$, a simulation is executed until either the horizon $T$ is reached or the evader is captured. The rollout return is computed as $r \triangleq \frac{1}{S_{A}} + \lambda \cdot \mathbb{I}_{\texttt{cap}}$, where $\lambda>0$ weights the capture reward, $S_{A}>0$ is the final area of the capture graph, and $\mathbb{I}_{\texttt{cap}}$ indicates the capture success.
It is worth mentioning that the evader strategy considered here is different from the strategies defined in Appendix~\ref{subsec:evader}. For the evaluation of a candidate assignment, we do not assume prior knowledge of the evader's intended destination or future trajectory. Instead, we introduce a surrogate evader policy $\hat{\pi}_{\texttt{e}}$, modeled as a random walk with a speed 1.2 times that of each pursuer. By conducting long-horizon simulations under $\hat{\pi}_{\texttt{e}}$, the proposed framework estimates the expected capture performance of each candidate assignment and the associated motion coefficients. Therefore, the ambush strategy does not rely on explicit prediction of the evader's actual motion, but instead selects assignments that are probabilistically favorable for capture under uncertain evader behavior.

\textbf{Backpropagation Phase.}
After simulation, the reward $r$ is propagated back to the root following the standard MCTS backpropagation procedure, updating the visit count and accumulated reward of each node along the selected path (Line~\ref{step:back}).

\textbf{Termination Phase.}
Once the computational budget $B$ is exhausted or a terminal state is reached, the algorithm selects the optimal child node $\nu^\star$ from the root $\nu_0$'s children (Line~\ref{step:select_best}). The corresponding optimal ambush parameters $\xi^\star = (A^\star, \boldsymbol{\rho}^\star)$ are then retrieved from the tree $\mathcal{T}$ as the label on the edge from $\nu_0$ to $\nu^\star$ (Line~\ref{step:get_action}), and deployed for real-time execution.


\begin{theorem}\label{thm:capture-guarantee}
For an assignment \(A \in \mathcal{A}\) with capture graph \(G_A\),
the evader is guaranteed to be captured under the arbitrary evader strategy \(\pi_{\texttt{e}}\) when:
(I) the number of exit gates must not exceed the number of pursuers,
expressed as \(|E_{\texttt{g}}^A| \leq |\mathcal{N}_{\texttt{p}}|\);
(II) the evader must reside within the interior of the capture polygon,
denoted \(\mathbf{p}_{\texttt{e}} \in \mathrm{int}(G_A)\);
(III) the capture radius of each attacker \(i \in \mathcal{N}_{\texttt{a}}\) must satisfy
\(d_{\texttt{c}} \geq \textup{\textbf{max}}\left\{ \frac{1}{2}
\textup{\textbf{max}}_{e \in E_{\texttt{g}}^A} \|e\|,\, \frac{1}{2} D \sin\varphi_i \right\}\).
\end{theorem}

\begin{proof}
Condition (I) enables pursuers to guard all exit gates.
With \(d_{\texttt{c}} \geq \frac{1}{2} \textbf{max}_{e \in E_{\texttt{g}}^A} \|e\|\),
guards can be achieved by attackers via complete coverage for the gates.
Condition (II) ensures the initial confinement within \(G_A\),
such that the evader cannot escape through guarded gates \(E_{\texttt{g}}^A\)
under any \(\pi_{\texttt{e}}\).
Condition (III) guarantees that each attacker can cover its sector chord,
ensuring no escape routes remain for the evader.
Then the continuous region contraction of the sweeping strategy under Alg.~\ref{alg:H-MCTS}
guarantees capture by an attacker or a hider.
\end{proof}

\begin{remark}
It should be noted that the AMBUSH strategy assumes that the obstacles can induce a valid capture polygon, either directly from polygonal obstacles or through conservative polygonal approximations of non-polygonal obstacles; cases where such a capture polygon cannot be constructed are discussed separately in Sec.~IV-D.
In addition,
the case where the evader avoids all gates and exploits a speed advantage to circulate inside the enclosure is still covered by Theorem~\ref{thm:capture-guarantee} under the stated assumptions. The key reason is that the attackers do not sweep toward fixed positions; instead, they continuously form a sector-sweeping configuration with respect to the current evader position. At each instant, the sectors are constructed around the evader, and the attacker coverage condition in Theorem~\ref{thm:capture-guarantee} ensures that the corresponding sector chords and guarded gates remain fully covered within the capture radius. Hence, no uncovered gap is created for escape, even when the evader moves faster inside the enclosure. Such motion can only shift the sweeping center and may delay capture, but it cannot break the continuous contraction of the admissible region.
\hfill $\blacksquare$
\end{remark}

It is also worth noting that 
the constructed capture polygon is conservative because it provides a sufficient condition for capture rather than an exact characterization of the full capture region. In other words, states outside $\mathcal{P}_{c}$ may still be capturable, but they are not certified by the theorem. This conservativeness may reduce the certified capture region in practice, but it improves the reliability of the theoretical guarantee and provides a computationally tractable safety certificate for the AMBUSH framework.

\begin{theorem}\label{thm:hmcts-optimality}
Given any environment,
the proposed H-MCTS algorithm in Alg.~\ref{alg:H-MCTS}
guarantees the existence of the optimal parameter $\xi^\star$,
which solves Problem~\ref{prb:hybrid_opt}.
\end{theorem}

\begin{proof}
The asymptotic optimality of H-MCTS is established through three principal arguments.
First, the continuous motion coefficient space $\boldsymbol{\rho} \in [0,1]^{\mathcal{N}{\texttt{p}}}$
is systematically sampled as Line~\ref{step:discrete},
yielding a finite hybrid action space $\Xi$.
Second, the UCT selection policy with $C_p$ in Line~\ref{step:uct_selection}
guarantees exhaustive exploration, i.e.,
as $B \to \infty$, every node is visited infinitely often and all $\xi \in \Xi$ are expanded
through the systematic sampling of assignments in Line~\ref{step:sample_gate}
and asymptotic coverage of the parameter space~$\widehat{\boldsymbol{\rho}}_+$.
Third, the simulation phase provides the consistent value estimation to
ensure $\lim_{N(\nu) \to \infty} Q(\nu)/N(\nu) = \bar{V}_B(\nu)$ almost surely
according to the Strong Law of Large Numbers.
Consequently, the termination rule in Line~\ref{step:select_best}
can select~$\xi^\star$, which solves Problem~\ref{prb:hybrid_opt} over the sampled space.
\end{proof}

\begin{remark}\label{rm:thm-discussion}
Theorem~\ref{thm:capture-guarantee} provides conservative sufficient conditions for the guaranteed capture under arbitrary evader strategies.
When these conditions are not met,
the capture becomes dependent on the evader's specific strategies and
cannot be predetermined.
Moreover,
our proposed method is specifically designed for complex environments where walls facilitate enclosure formation. 
Future work will explore multiple strategies switching to handle diverse environmental layouts.
\hfill $\blacksquare$
\end{remark}
\vspace{-5mm}
\subsection{Neural Acceleration}\label{sec:learn}

\begin{figure*}[!t]
    \centering
    \includegraphics[width=0.8\linewidth,height=0.4\linewidth]{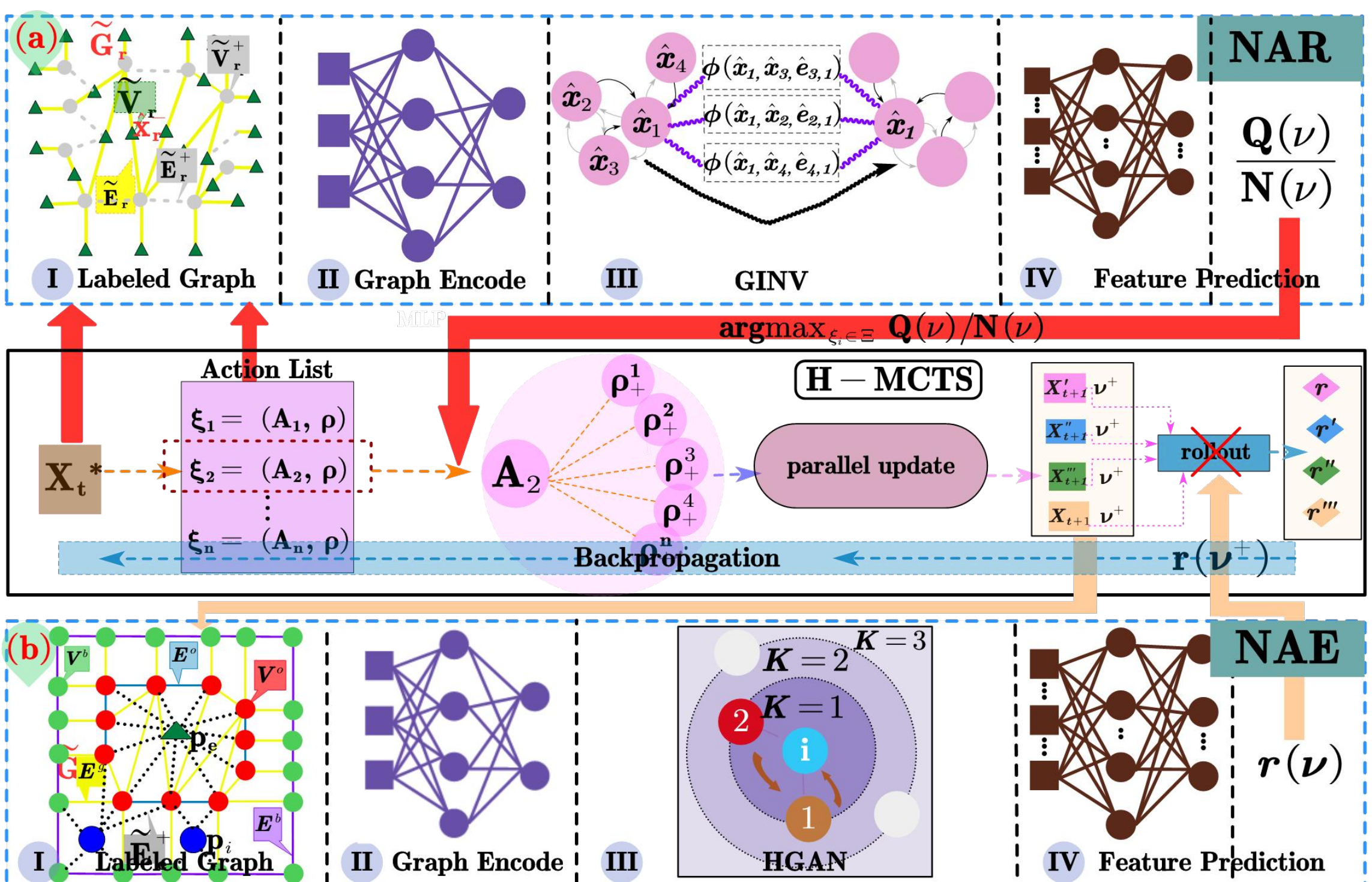}
    \caption{
      Architecture of the proposed neural accelerator NARE:
      \textbf{(a)} The Neural Network Ranker (NAR) takes the visibility graph $\widetilde{G}_{\texttt{r}}$
      to predict the cumulative Q-value for ambush parameters $\xi \in \Xi$;
      \textbf{(b)} The Neural Network Evaluator (NAE) requries the state-encoding graph $\widetilde{G}_{\texttt{e}}$
      to predict the immediate simulation rewards $r(\nu)$ for nodes $\nu$ in the search tree.
      Both components share a common three pre-processing steps including
      the graph encoding, GNN processing, and feature prediction.
    }
    \label{fig:NARE}
    \vspace{-6mm}
  \end{figure*}

Our system evolves dynamically, making the current optimal ambush strategy potentially suboptimal later due to the high cost of evaluating parameters $\xi \in \Xi$ and running simulations in H-MCTS. To tackle this, we propose a neural acceleration called NARE, (Fig.~\ref{fig:NARE}). It has two parts: (I) a Neural Network Ranker (NAR) that ranks parameters by their Q-value ratio to guide the node expansion; and (II) a Neural Network Evaluator (NAE) that predicts the rollout reward $r(\nu)$ for nodes $\nu$, replacing costly simulations.
It should be noted that NARE is used only as an accelerator rather than a theoretically guaranteed policy.
The final decision is still made by H-MCTS, while NARE only prioritizes node expansion and approximates rollout rewards to improve computational efficiency.

\subsubsection{Labeled Graph Input of NARE}
\label{subsec:graph-enc}
Since the NAR predicts the ratio of cumulative rewards to visits ($Q/N$) from complete MCTS procedures, 
while the NAE estimates the immediate simulation reward $r$ of individual nodes, their input graph representations differ fundamentally. 
As shown in Fig.~\ref{fig:NARE}, 
NAR uses a visibility graph representing environmental topology,
whereas NAE employs a heterogeneous state graph that captures the real-time system dynamics.

The NAR input graph $\widetilde{G}_{\texttt{r}} = (\widetilde{V}_{\texttt{r}}, \widetilde{E}_{\texttt{r}})$, shown in Fig.~\ref{fig:NARE}(a), 
is a visibility graph built from boundary vertices $V_{\texttt{b}}$ and obstacle skeleton vertices $\overline{V}_{\texttt{o}}$ in Def.~\ref{def:gate}. Its vertex set $\widetilde{V}_{\texttt{r}}$ corresponds to attacker goal candidates $\tilde{g}^{\texttt{a}}$ from~\eqref{eq:attacker-goals}. Each vertex $v_{\texttt{r}} \in \widetilde{V}_{\texttt{r}}$, located at $\mathbf{p}_{v_r} \in \mathbb{R}^2$, includes three features $\hat{x}_{\texttt{r}}^-$: betweenness centrality $C_b(v_{\texttt{r}})$, its minimum distance to any pursuer, and its distance to the evader. Edges $\widetilde{E}_{\texttt{r}}$ are the gate connections $E_{\texttt{g}}$ from Def.~\ref{def:gate}, linking visible vertices to enable the message passing. This structure allows NAR to capture long-term reward dependencies via the graph connectivity.

The NAE processes a state-encoding graph $\widetilde{G}_{\texttt{e}} = (\widetilde{V}_{\texttt{e}}, \widetilde{E}_{\texttt{e}})$ that models system states via heterogeneous nodes, as in Fig.~\ref{fig:NARE}(b). 
Its vertex set $\widetilde{V}_{\texttt{e}}$ includes seven types: pursuer positions $\mathbf{p}_i$, evader position $\mathbf{p}_{\texttt{e}}$, the combined goal set $\tilde{g}$, the evader's current goal $g_{\texttt{e}}$, boundary vertices $V^{\texttt{b}}$, and obstacle vertices $V^{\texttt{o}}$. The edge set $\widetilde{E}_{\texttt{e}}$ encodes functional relationships: 
an assignment link denotes the geometric association between the current position of pursuer $i$ and its assigned goal, 
which is represented by the line segment $(\mathbf{p}_i \leftrightarrow \tilde{\mathbf{g}}_i)$;
a navigation link denotes the geometric association between the current position of the evader and its intended goal, 
which is represented by the line segment $(\mathbf{p}_{\texttt{e}} \leftrightarrow \mathbf{g}_{\texttt{e}})$;
boundary edges $E_{\texttt{b}}$; obstacle edges $E_{\texttt{o}}$; and gate edges $E_{\texttt{g}}$. 
Each edge has a one-hot encoded type attribute $\hat{e}^-$. 
In addition, distance-based edges are introduced to provide local spatial
connectivity for message passing. Specifically, for any two spatial nodes
$u,v\in\widetilde{V}_{\texttt{e}}$, 
an undirected distance edge is added if
$v\in\mathcal{N}_{k_d}(u)$ and $\|\mathbf{x}_u-\mathbf{x}_v\|_2\le r_d$,
where $\mathcal{N}_{k_d}(u)$ denotes the set of the $k_d$ nearest neighbors
of node $u$ and $r_d$ is a fixed distance threshold used in all experiments.
This design allows NAE to assess immediate rewards from spatial
and assignment relations.

\subsubsection{Network Architecture of NARE}
\label{subsec:nare}

Following the tailored graph structures depicted in Fig.~\ref{fig:NARE}, 
both the NAR and NAE components share a unified three-stage processing framework: graph encoding, GNN-based message passing, and feature prediction.
(I) Graph Encoding. Based on their respective input graphs, initial feature embeddings are generated. For NAR, each vertex's 3-dimensional feature vector $\hat{x}_{\texttt{r}}^-$ is encoded via a two-layer MLP. Since its edges are unweighted visibility connections, they require no parametric encoding. For NAE, its heterogeneous vertices and edges~$\hat{e}^-$ are embedded into a shared feature space using a three-layer MLP.
(II) GNN-based Module. The encoded features are refined through $L$ stacked GNN layers. To suit its homogeneous topology, NAR employs Graph Isomorphism Networks (GINV)~\cite{xu2018powerful}, which ensure the permutation-invariant aggregation. For NAE's heterogeneous graph, we adopt Heterogeneous Graph Attention Networks (HGAN)~\cite{velivckovic2017graph} with modality-specific attention mechanisms to model the distinct interaction types, e.g., assignment and navigation.
(III) Feature Prediction. The final graph-level output is obtained via a virtual super-node bidirectionally connected to all vertices, followed by an MLP. This design enables fully distributed execution, as all robots share identical network parameters. 
Each robot only needs to broadcast its own position~$(\mathbf{p}_i$, $\mathbf{p}_{\texttt{e}})$ and assignment~$\tilde{g}_i$ to construct the graph locally and perform the independent inference.

\subsubsection{Training and Network Parameter Setting}
\label{subsec:train-exc}

Training datasets for NAR and NAE are generated by executing H-MCTS across diverse scenarios with varying robot numbers, initial positions, and velocities. The NAR dataset $\mathcal{D}_{\texttt{r}}$ collects every first-layer child node $\nu$ of the MCTS root, annotated with its obtained Q-value ratio $\hat{Q}(\nu)/\hat{N}(\nu)$. The NAE dataset $\mathcal{D}_{\texttt{e}}$ comprises all expanded nodes $\nu$ from the tree, each labeled with the immediate simulation reward $\hat{r}(\nu)$ from its rollout. All nodes are encoded into their corresponding graph representations in Sec.~\ref{subsec:graph-enc}. 
Both components are trained supervisedly, i.e.,
NAR minimizes the MSE loss $\mathcal{L}_{\texttt{r}} = \frac{1}{|\mathcal{D}_{\texttt{r}}|}\sum (Q(\nu)/N(\nu) - \hat{Q}(\nu)/\hat{N}(\nu))^2$ between predicted and obtained Q-ratios, while NAE minimizes $\mathcal{L}_{\texttt{e}} = \frac{1}{|\mathcal{D}_{\texttt{e}}|}\sum (r(\nu) - \hat{r}(\nu))^2$ between predicted and obtained rewards. The training and validation process is illustrated in Fig.~\ref{fig:learn-result}.
Notably, to ensure the generalization capability of the NARE, 
the validation set is composed of entirely different scenarios 
featuring distinct topologies and velocity ratios compared to those in the training set.

\subsubsection{Execution of H-MCTS via NARE}
\label{subsec:exec}

To maintain the solution quality while accelerating online execution, we integrate the trained NARE modules into H-MCTS via a dual-stage framework. 
The process, illustrated in Fig.~\ref{fig:NARE}, is as follows:
(I) Expansion Acceleration: During node expansion (Alg.~\ref{alg:H-MCTS}, Line~\ref{step:sample_gate}), NAR ranks candidate assignments $\widehat{\mathcal{A}}$ by their predicted $Q/N$ values in descending order, as follows:
\begin{equation}
\frac{Q(\nu_1)}{N(\nu_1)} \geq \frac{Q(\nu_2)}{N(\nu_2)} \geq \cdots \geq \frac{Q(\nu_{|\widehat{\mathcal{A}}|})}{N(\nu_{|\widehat{\mathcal{A}}|})},
\label{eq:ranking}
\end{equation}
The assignment with the highest unexplored score is selected first, after which standard expansion proceeds. Example outputs are shown in Fig.~\ref{fig:NAR}.
(II) Simulation Acceleration:  For leaf nodes, NAE predicts $r(\nu)$ directly from graph $\widetilde{G}_{\texttt{e}}$, replacing the rollout in Line~\ref{step:rollout} of Alg.~\ref{alg:H-MCTS}.


\begin{figure}[!t]
  \centering
  \includegraphics[width=1.0\linewidth]{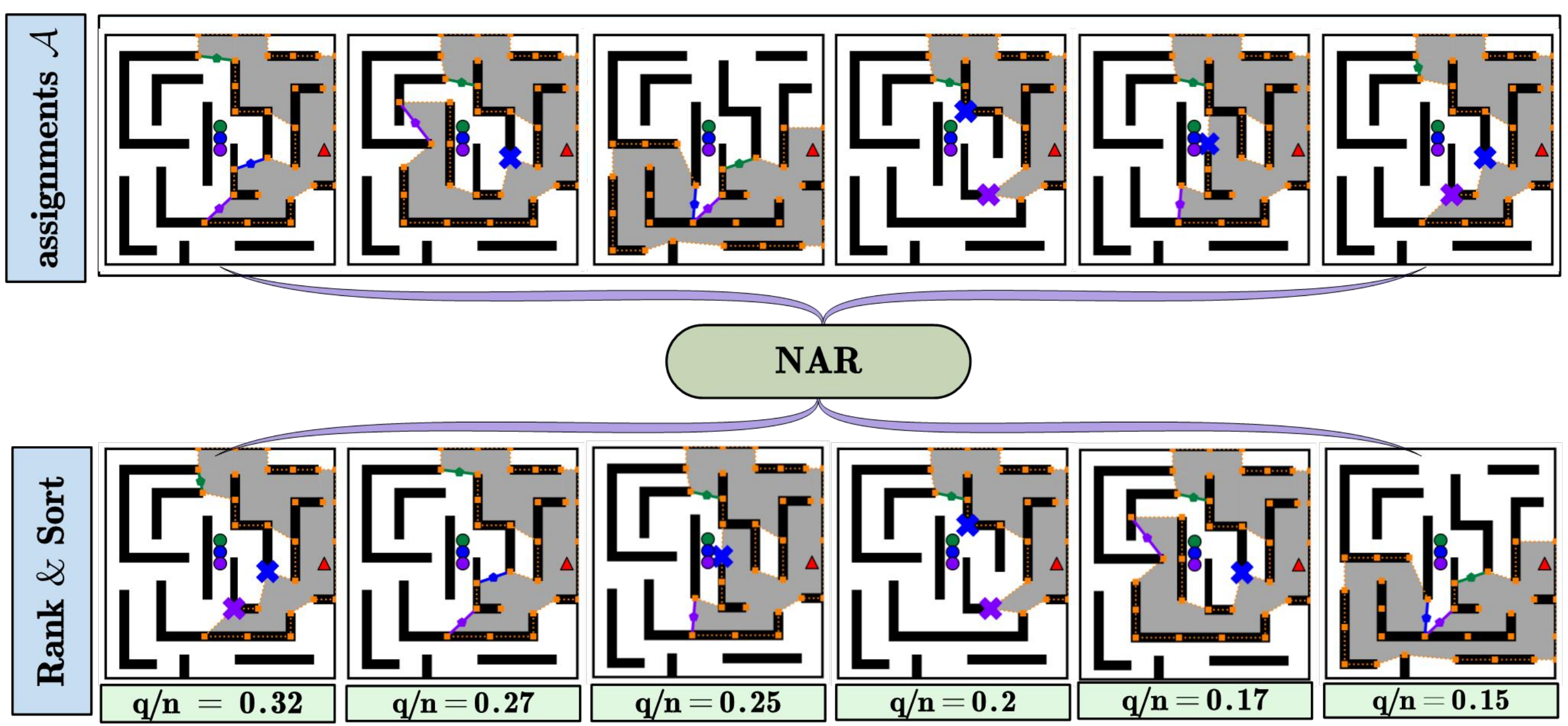}
  \caption{
    Illustration of the ambush assignments (\textbf{top}),
    which are ranked and sorted via the learned NAR (\textbf{bottom}).
    Each gray enclosure denotes the capture graph induced by a specific assignment, 
    including the goals of hiders and attackers.
  }
  \label{fig:NAR}
  \vspace{-6mm}
\end{figure}

\subsection{Online Execution and Adaptation}\label{subsec:online}

In dynamic pursuit-evasion, fixed-interval replanning has two key limitations,
i.e., sensitivity to parameter tuning towards different scenarios and unnecessary interruptions of effective strategies. 
To address these, we propose a dynamic triggering mechanism that initiates H-MCTS replanning along with the role switching for hiders
under three conditions.
First,
replanning triggers when the evader breaches the ambush perimeter, meaning no pursuer can reach its interception position before the evader escapes its assigned sector, namely:
\[
\min_{i \in \mathcal{N}_{\texttt{p}}} \|\mathbf{p}^t_i - \mathbf{p}^t_{\texttt{e}}\| > \big(2d_{\texttt{c}} + \max_{i \in \mathcal{N}_{\texttt{p}}}(\overline{v}_{\texttt{p}} \cdot \Delta t_{\texttt{rct}})\big),
\]
where $\Delta t_{\texttt{rct}}$ is the pursuer reaction time. This detects unexpected evader maneuvers that invalidate the current ambush, resetting all  positions $g_i$ based on the evader's current location.
Second, replanning activates if all pursuers reach their ambush positions significantly earlier than the estimated capture time:
\[
\|\mathbf{p}^t_i - g_i\| < \varepsilon_{\texttt{pos}},\; \forall i \in \mathcal{N}_{\texttt{p}},\; \text{and} \; t < t_n + T_{\texttt{est}},
\]
where $\varepsilon_{\texttt{pos}}$ is a position tolerance, $t_n$ is the last replanning time, and $T_{\texttt{est}}$ is the expected capture time. This adapts to evader behavior changes by regenerating $T_{\texttt{est}}$ via updated simulations.
Lastly, a maximum interval $\Delta t_{\texttt{max}}$ ensures the timeliness. Replanning is forced if this interval elapses without other triggers, indicating the current strategy is ineffective. This hybrid approach minimizes unnecessary replanning while ensuring the responsiveness. 
During execution, transition data $\{(\mathbf{X}_t, \xi, \mathbf{X}_{t+1})\}$ are logged to refine the estimated evader policy $\hat{\pi}_{\texttt{e}}$ and reward calibration.
\subsection{Generalization}\label{subsec:online}

\subsubsection{Heterogeneous pursuers}
\label{subsubsec:heterogeneous}

The proposed ambush strategy supports heterogeneous pursuer teams with distinct capabilities, such as varying capture radius and velocities. This is enabled by two aspects: (i) the agent model in Sec.~\ref{sec:problem} configures per-pursuer parameters like $d_{\texttt{c}}$ and $\overline{v}_{\texttt{p}}$;
(ii) the core strategy modules including goal generation and role assignment,
operate independently of kinematics, while the motion strategy $f(\cdot)$ naturally accommodates speed differences. Furthermore, the H-MCTS algorithm evaluates strategies via simulation rather than explicit motion constraints, allowing it to adapt automatically to heterogeneous robots without modification. This decoupled design resolves capability differences at the configuration level, maintaining computational efficiency comparable to homogeneous cases.

\subsubsection{Limited view for pursuers}
\label{subsubsec:limited_view}

The ambush strategy may degrade if pursuers lack the real-time evader positioning and must rely on memorized locations. To address this, we introduce a phase-switching mechanism. After executing an ambush, each pursuer locally checks evader visibility via its sensors. If any pursuer detects the evader, the ambush continues; otherwise, the team switches to a cooperative search-and-exploration strategy as seen in the work~\cite{rockenbauer2024traversing, gul2021novel}, systematically exploring regions around the evader’s last known position through coordinated waypoints. This decoupled design separates perception constraints from strategic planning, and the switching mechanism provides extensibility for similar partial-observation challenges.

\subsubsection{Multiple evaders}
\label{subsubsec:multiple_evaders}

The modular design can be extended to multi-evader scenarios using task allocation methods~\cite{chen2023accelerated}. Neighboring pursuers form dynamic, evader-specific coalitions, each executing an independent ambush strategy. Capture times provide utility metrics, enabling convergence to a Nash equilibrium through iterative switching.
The process consists of three phases: (i) coalition initialization via partial partitioning; (ii) utility-driven optimization to reach an equilibrium assignment; (iii) dynamic reallocation upon evader escapes or captures, which redistributes pursuers while maintaining active strategies. Formally, the allocation maximizes the mean coalition utility:
$
\textbf{max}_{\mathcal{A}} \sum_{m=1}^{M} u(\mathcal{R}_m, \omega_m),
$
where \(\mathcal{A}\) is the set of robot-task assignments, \(\mathcal{R}_m\) is the coalition for the \(m\)-th evader's task \(\omega_m\), and \(u(\cdot)\) is the utility function. The solution satisfies constraints of valid robot subsets, mutual exclusivity, and complete assignment. This approach supports the adaptation through local communication, maintaining the real-time performance in such dynamic scenarios.


\subsubsection{Sparse Environments}
\label{subsubsec:sparse_envs}
While the proposed ambush strategy is primarily designed for obstacle-rich environments, where discretization of the space and gate formation enable effective encirclements, it can also be adapted to sparse or open scenarios. Such scenarios include a few large obstacles, small scattered obstacles, or largely free space. 
In these cases, the Alg.~1 can be bypassed, and the rollout procedures in Alg.~2 can be executed with motion strategies replaced by encirclement algorithms capable of handling minimal obstacles, as proposed in~\cite{fang2020cooperative}.
This modular substitution preserves coordination and planning structures while extending applicability to low-density and free-space environments.

\begin{figure*}[!t]
    \centering
    \includegraphics[width=0.95\linewidth]{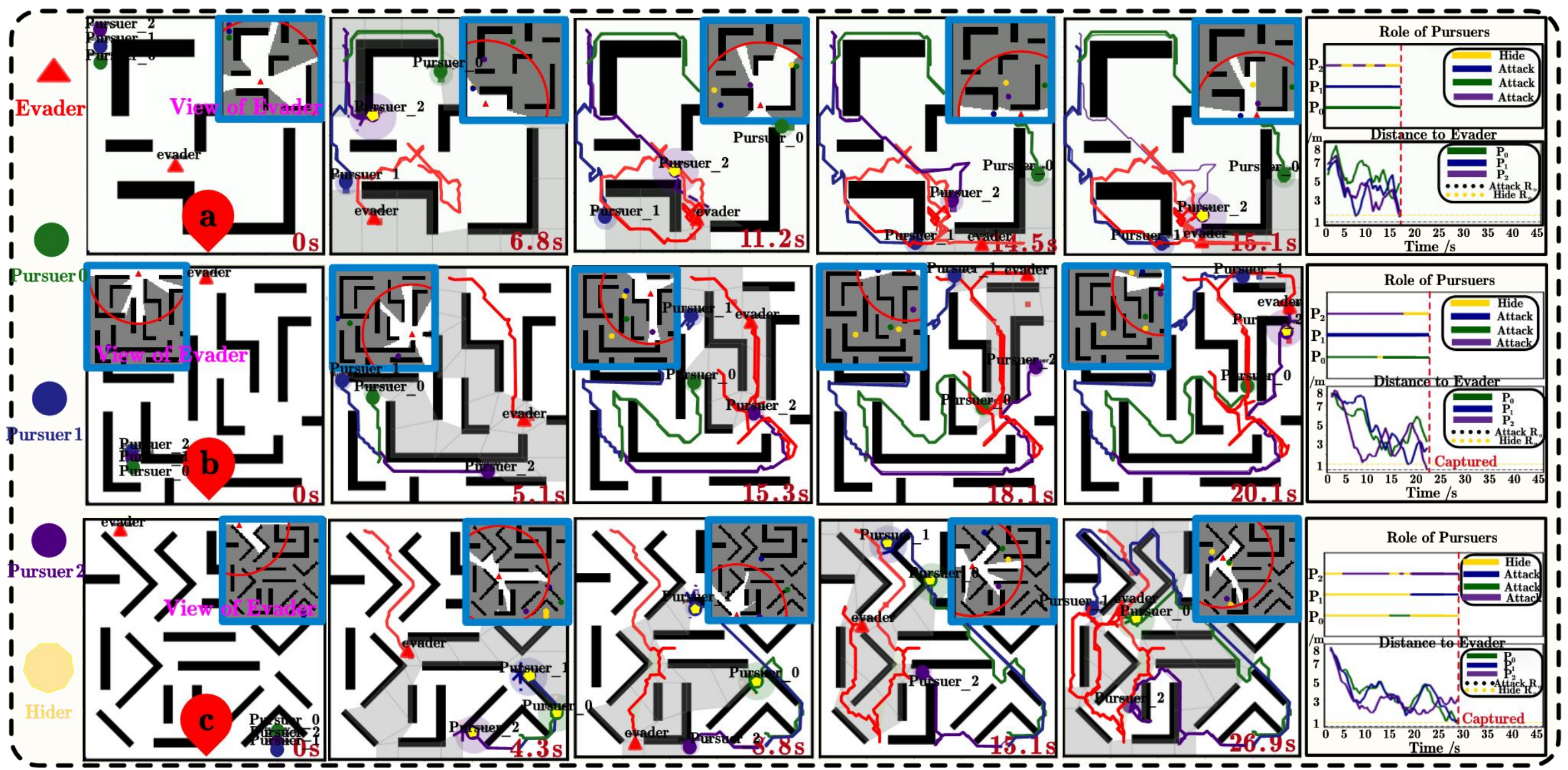}
    \caption{
      Snapshots of the AMBUSH strategy for $3$ pursuers and $1$ evader,
      under different scenarios in the evaluation:
      Scenario-I (\textbf{top}),
      Scenario-II (\textbf{middle});
      and Scenario-III (\textbf{bottom}).
      Note that the local view of the evader is shown in the corners,
      while the switching of pursuer roles and their relative distance
      to the evader are shown in the right column.
    }
    \label{fig:all-case}
    \vspace{-4mm}
  \end{figure*}

\subsubsection{Different Robot Model}
\label{subsubsec:robot_model}

The proposed framework adopts a hierarchical structure, where the high-level planner 
generates discrete waypoints, while the low-level controller 
tracks the resulting reference trajectories. After solving the assignment problem 
via Alg.~2, collision-free paths are generated using A* search. Since the planner 
only outputs geometric paths and task-level decisions, it remains independent of 
specific robot dynamics.
Then we consider two representative models. 
For non-holonomic ground robots, a unicycle model is used, where the controller 
regulates forward and angular velocities to align the robot with the reference path. 
For systems with inertial dynamics, such as UAVs or acceleration-limited vehicles, 
a double-integrator model is adopted, where acceleration is regulated to track the 
desired velocity. In both cases, saturation is imposed to ensure physical feasibility, 
so the same planning pipeline can be transferred across heterogeneous platforms by 
only modifying the low-level control law.

\vspace{2mm}

\subsubsection{Non-polygonal Obstacle Environments}
\label{subsubsec:non_poly}
The proposed vertex-based representation can be extended to non-polygonal obstacles, 
such as circular or curved obstacles, by using conservative polygonal approximations. 
Specifically, each obstacle is enclosed by a bounding rectangle, or decomposed into 
several local rectangles for finer approximation. The resulting vertices and edges 
enable the same skeleton extraction and gate construction procedures without modifying 
the high-level planning framework. 
For hider-goal generation,
candidate goals are sampled along the angular span of the original obstacle boundary 
and placed with a safety offset.
These sampled goals are then checked against the approximated obstacle boundary to avoid placing them inside the conservative enclosure. In this way, the extension only affects the geometric preprocessing step, while the subsequent visibility reasoning and gate-based planning remain unchanged.

\section{Numerical Experiments} \label{sec:experiments}


\subsection{Experimental Setup}
\label{sec:exp-setup}

As shown in Fig.~\ref{fig:all-case}, experiments are conducted with multiple pursuers ($N_{\texttt{p}}\ge2$) and one evader in three $10\times10$ m\textsuperscript{2} environments: a basic one with regular obstacles in Fig.~\ref{fig:all-case}(a), a complex one with polygonal obstacles in Fig.~\ref{fig:all-case}(b), and a challenging one with irregular obstacles (c) in Fig.~\ref{fig:all-case}. All agents are randomly initialized in free space. The pursuers’ maximum speed is $\overline{v}_{\texttt{p}} = 1.0$ m/s, compared to the evader’s $\overline{v}_{\texttt{e}} = 1.6$ m/s. The line-of-sight detection range is $d_{\texttt{o}} = 10$ m, and the capture radius is $d_{\texttt{c}} = 0.4$ m. Our event-triggered control strategy updates only when all pursuers reach their goals or the evader escapes. Trials are limited to $30$ s; failure to capture within this time results in the termination.
Each experiment runs five trials with different initial positions.

To further validate the effectiveness of the proposed method,
five state-of-the-art approaches are compared as follows:
(I) Analytic-based Method~\cite{fang2020cooperative}:
        A distributed algorithm enables slower pursuers to capture faster evaders
        through the adaptive switching between two strategies, i.e., the collaborative encirclement formation
        and the direct hunting approach;
(II) Geometric-based Method~\cite{wang2023distributed}:
        Centralized algorithm using buffered Voronoi cells to coordinate pursuers in cluttered environments;
(III) RL-based Method~\cite{de2021decentralized}:
        This approach implements a decentralized deep reinforcement learning framework
        to learn cooperative pursuit strategies through a structured curriculum mechanism,
        and replicate their hyperparameter configurations;
(IV) Shooting:
        Probabilistic gate selection method that replaces H-MCTS with a scoring mechanism.
        For each candidate gate, the expected capture probability is computed by:
        (i) generating evader's shortest paths to goals,
        (ii) evaluating overlap between pursuer zones and escape paths,
        (iii) aggregating probability-weighted coverage.
        Then the highest-scoring gate is selected;
(V) MCTS:
        Direct implementation of H-MCTS algorithm without incorporating the NARE module.

\subsection{Results and Analysis}
\label{subsec: overall-perf}

\subsubsection{Overall Performance}

\begin{figure}[!t]
  \centering
  \includegraphics[width=0.99\linewidth]{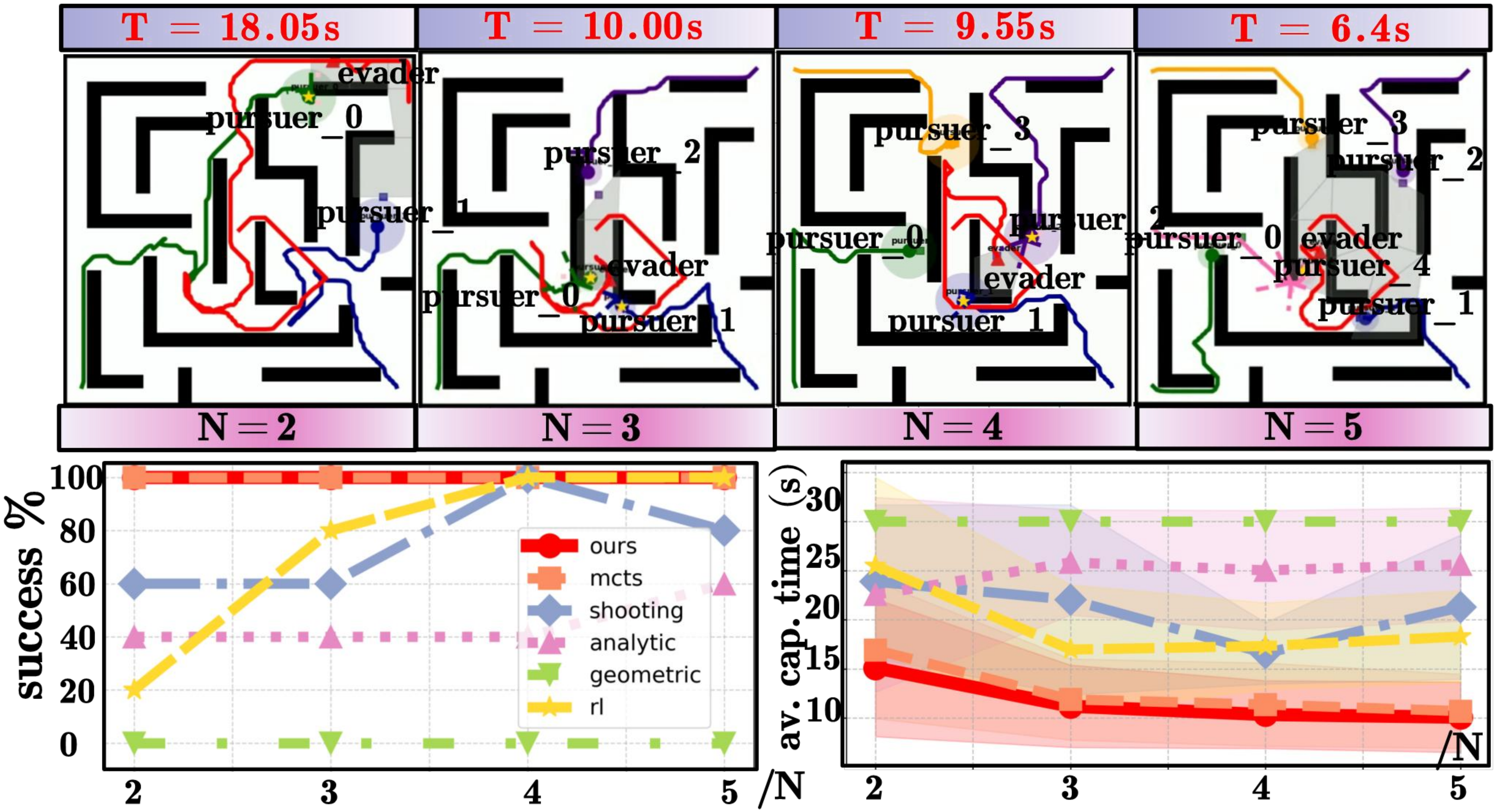}
  \vspace{-5mm}
  \caption{The capture rate and average capture time (\textbf{bottom})
    under different number of pursuers ranging from $2$ to $5$
    over five trials,
    along with the snapshots of final capture (\textbf{top}).
  }
  \label{fig:number}
  \vspace{-3mm}
\end{figure}

\begin{figure}[!t]
  \centering
  \includegraphics[width=0.99\linewidth]{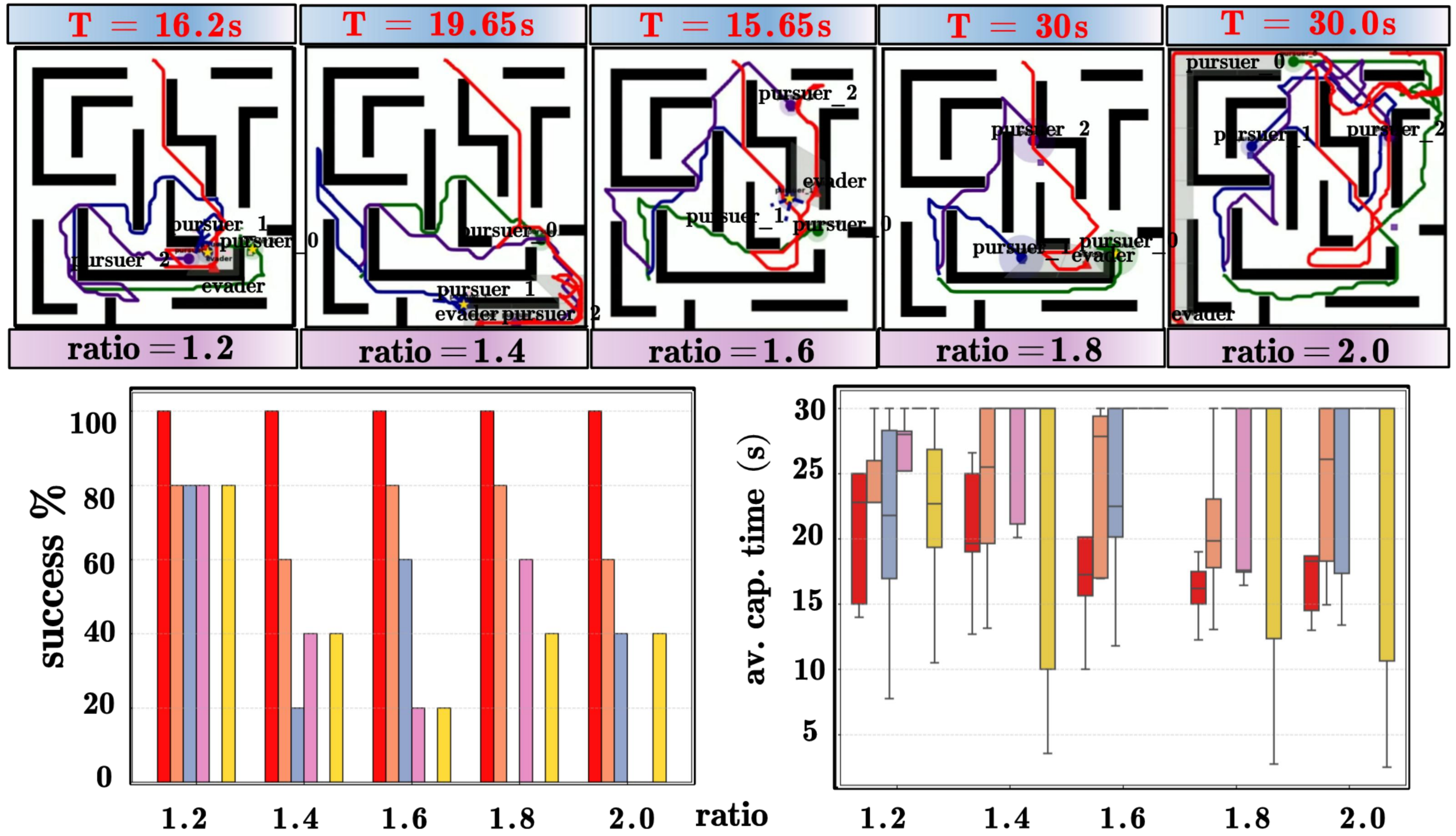}
  \vspace{-4mm}
  \caption{
    The average capture rate and capture time
    under different velocity ratios between the evader and the pursuers,
    ranging from $1.2$ to $2.0$ through performing five trials with randomized positions.
  }
  \label{fig:velocity}
  \vspace{-6mm}
\end{figure}

The proposed method is evaluated in the scenarios presented in Fig.~\ref{fig:all-case}. It achieves real-time policy updates averaging 5.5~s. The pipeline comprises an offline phase for topological decomposition and goal computation ($\sim$100~s and 2~s, respectively), and an online phase that filters 20\% of goals and prunes 40\% of the search space, each in 0.1~s. The H-MCTS algorithm executes gate selection in 5~ms/cycle and parameter discretization in 0.5~s.
In Scenario-I, pursuers transit from the initial clustering to specialized roles after evader escape at $t = 6.8\text{s}$.
By $t = 11.2\text{s}$, P0 colored by green guards distal gates,
P1 colored by blue herds the evader colored by red towards the ambush region,
and P2 colored by purple conceals near the evader.
Despite P2's failed surprise attack at $t = 14.5\text{s}$,
the synchronized gate closure achieves capture at $t = 15.1\text{s}$.
Scenario-II establishes an ambush region at $t = 5.2\text{s}$ systematically shrunk via P0-P1 herding,
culminating in P2's emergence-driven capture at $t = 20.1\text{s}$.
Scenario-III demonstrates the environmental adaptation through dynamic gate recomputation
at $t = \{4.3\text{s}, 15.1\text{s}\}$, with final interception via curved paths at $t = 26.9\text{s}$.
Key findings indicate that the strategy-invocation frequency increases with environmental complexity (averaging 2.4, 8.8, and 5 times), while the capture time also grows (6.7~s, 9.7~s, 17.0~s). Role adaptation is observed, transitioning from attacker dominance ($67\pm4\%$) to greater hider involvement ($33\pm3\%$), effectively balancing the adaptability and efficiency.

\subsubsection{Analysis of Key Parameters}
\label{subsubsec:ana-para}

We systematically evaluate the performance of our proposed method across critical parameters
over five trials with randomized initial positions in each parameter setting.
Key factors influencing capture efficiency, i.e., success rate and capture time,
are examined through the following variations, including:
pursuer numbers, speed ratios, evader intelligence levels, and capture ranges.

\begin{figure}[!t]
  \centering
  \includegraphics[width=0.9\linewidth]{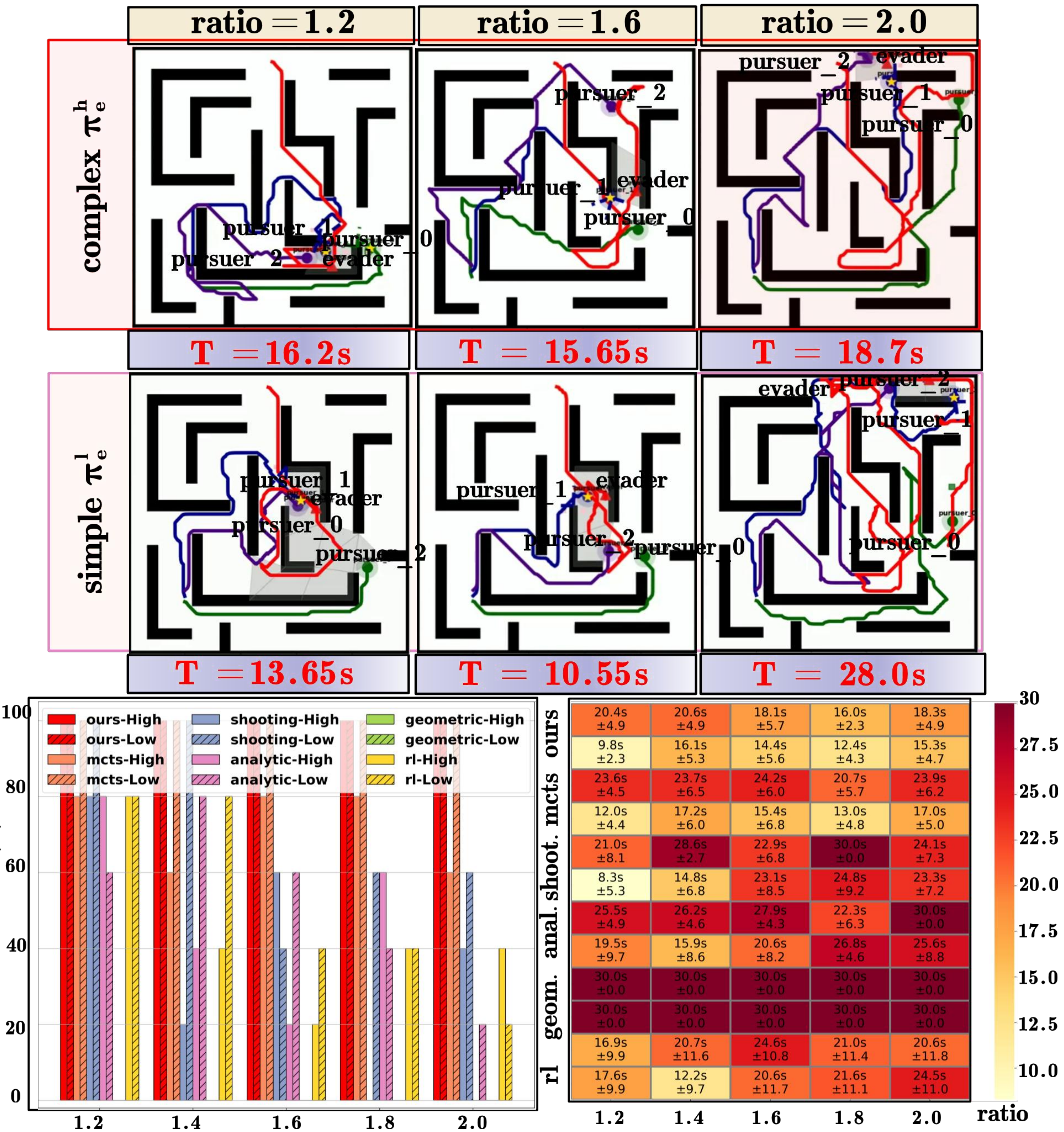}
  \caption{
    Snapshots of the final capture under complex (\textbf{top}) and simple (\textbf{middle})
    strategies of the evader and different velocity ratios;
    Comparisons of the average capture rate and time of different baselines,
    under different evader strategies and velocity ratios (\textbf{bottom})
    over five trials with different initial positions.
  }
  \label{fig:intelligence}
  \vspace{-6mm}
\end{figure}

\textbf{(I) Number of Pursuers:}
Fig.~\ref{fig:number} shows both our method and its MCTS variant achieve $100\%$ success across all pursuer numbers from $2$ to~$5$,
with capture time decreasing from $15.7 \pm 3.3\text{s}$ to $10.08 \pm 4.0\text{s}$.
Time saturates beyond 4 pursuers ($10.35 \pm 3.9\text{s}$ vs. $10.08 \pm 4.0\text{s}$) due to the agent redundancy,
as evidenced by the inactive pink trajectory (N=5).
The shooting method shows the high variability ($60-100\%$ success; 
$24.7 \pm 2.3\text{s}$ to $15.7 \pm 3.3\text{s}$ in capture time),
while the analytic approach exhibits the poor scaling in the success rate from $40\%$ to $60\%$
and from $25 \pm 1.8\text{s}$ to $22 \pm 2.3\text{s}$ in capture time.
Voronoi method fails completely (0\% success) in narrow environments.
RL baseline displays the strong size-dependence ($20\%$ to $100\%$ success) but $52.7\%$ longer capture times
than ours ($25.2 \pm 4.2\text{s}$ vs. $10.08 \pm 4.0\text{s}$ at N=5),
highlighting their reliance on the numerical advantage over coordination.

\textbf{(II) Evader Velocity:}
Fig.~\ref{fig:velocity} shows our method maintains $100\%$ success across speed ratios of~$1.2-2.0$,
with capture times exhibiting a convex trend, i.e.,
decreasing from $20.4 \pm 5.4\text{s}$ at the speed of~$1.2$ to $15.9 \pm 2.6\text{s}$ at~$1.8$,
then rising to $18.3 \pm 5.4\text{s}$ at the speed of~$2.0$.
This pattern reflects the trade-off between the entering or escaping from the ambush region, i.e.,
moderate ratios of~$1.2-1.8$ enhance the entrapment while higher ratios ($>1.8$) enable escapes.
The MCTS variant follows this trend but with $25-40\%$ longer times.
Shooting method shows the random performance ($20-80\%$ success,
$20.1-27.5\text{s}$ in capture time),
analytic approach completely fails,
and the RL baseline, trained at a speed ratio of~$1.6$, exhibits limited generalization,
i.e.,
it achieves $70\%$ success for speed ratios below $1.6$, 
but drops to $40\%$ for ratios above $1.8$,
with capture times $45$--$60\%$ longer than our method.

\begin{figure}[!t]
  \centering
  \includegraphics[width=0.99\linewidth]{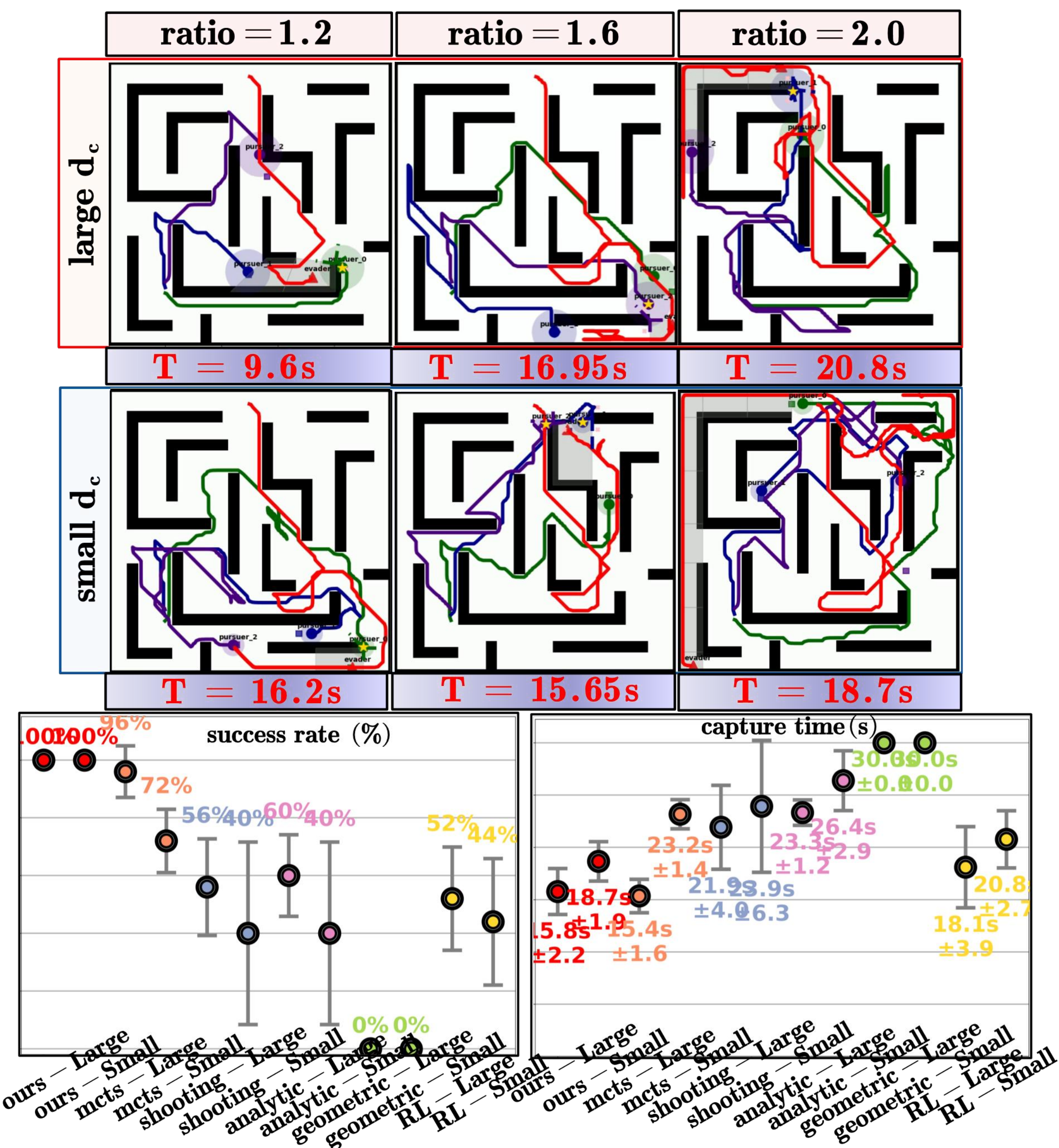}
  \caption{
    Snapshots of the final capture under different capture ranges of the pursuers
    and velocity ratios (\textbf{top});
    Comparisons of the average capture rate and time of different baselines,
    under different  capture ranges (\textbf{bottom})
    via five trials with different initial positions.}
  \label{fig:capture-range}
  \vspace{-6mm}
\end{figure}

\textbf{(III) Evader Intelligence:}
Our method maintains $100\%$ success across different intelligence levels of the evader as described in Sec.~\ref{subsec:evader}
at the speed ratios of~$1.2-2.0$ in Fig.~\ref{fig:intelligence},
though requiring $108.2\%$ longer capture times against complex evaders 
($20.4\pm4.9\text{s}$ vs. $9.8\pm1.7\text{s}$ at the speed ratio of~$1.2$).
MCTS shows the high sensitivity ($37.8-96.7\%$ prolonged times),
shooting method exhibits the stochasticity ($8.3-24.8\text{s}$),
analytic approach degrades progressively ($17.2-64.8\%$ longer times),
geometric method fails completely ($0\%$ success),
and RL baseline generalizes poorly ($21.6\text{s}$ vs. ours $12.4\text{s}$).
Against human operators,
our method achieves $60.0\%$ success with local perception
but achieves $40\%$ success at a speed ratio of $1.2$, 
$20\%$ at $1.6$, and $0\%$ at $2.0$ under global vision,
demonstrating the effective adaptation to realistic partial-information scenarios
while revealing inherent limitations against the perfect situational awareness.

\textbf{(IV) Capture Range:}
Capture range~$d_{\texttt{c}}$ critically influences the effectiveness as shown in Fig.~\ref{fig:capture-range}.
Our method maintains $100\%$ success across all speeds of~$1.2-2.0\text{m/s}$ and ranges $d_c$ from~$0.4\text{m}$ to~$0.8\text{m}$.
Baselines show the variable success ($20-100\%$),
with MCTS, shooting, and analytic methods achieving $25\%$, $60\%$,
and $60\%$ higher success respectively at $d_c=0.8\text{m}$ vs. $0.4\text{m}$.
Capture range dominates efficiency more than other factors,
with all baselines except geometric method showing the strong sensitivity.
Notably, RL maintains near-optimal performance when tested at $d_c=0.4\text{m}$ after training at $0.8\text{m}$,
demonstrating exceptional generalization.

\subsubsection{Discussion}
\label{subsubsec:discuss}
Our analysis identifies capture range as the dominant factor for improving capture effectiveness at high evader speed.
As Fig.~\ref{fig:discuss} (a) shows, by using $3$ pursuers with speed ratio of~$2.0$, capture radius $0.4\text{m}$ under Scenario-II,
increasing range to $0.8\text{m}$ enables the successful capture with proportionally decreasing time in Fig.~\ref{fig:discuss} (b),
demonstrating the geometrically enhanced effectiveness.
Adding $4$-$5$ pursuers achieves the successful capture but with diminishing returns due to overlapping operational zones.
Similarly, increasing pursuer speed from $1.0$ to $2.0\text{m/s}$ enables success but exhibits non-linear saturation,
highlighting the coordination's importance over kinematic superiority.
The parametric study reveals a clear hierarchy of influence, i.e.,
capture range dominates capture effectiveness,
followed by secondary factors of pursuer number and speed showing comparable impacts.
These findings provide critical design guidelines, i.e.,
prioritize capture range optimization over the brute-force scaling of team size or mobility under resource constraints.

\begin{figure}[!t]
  \centering
  \includegraphics[width=0.99\linewidth]{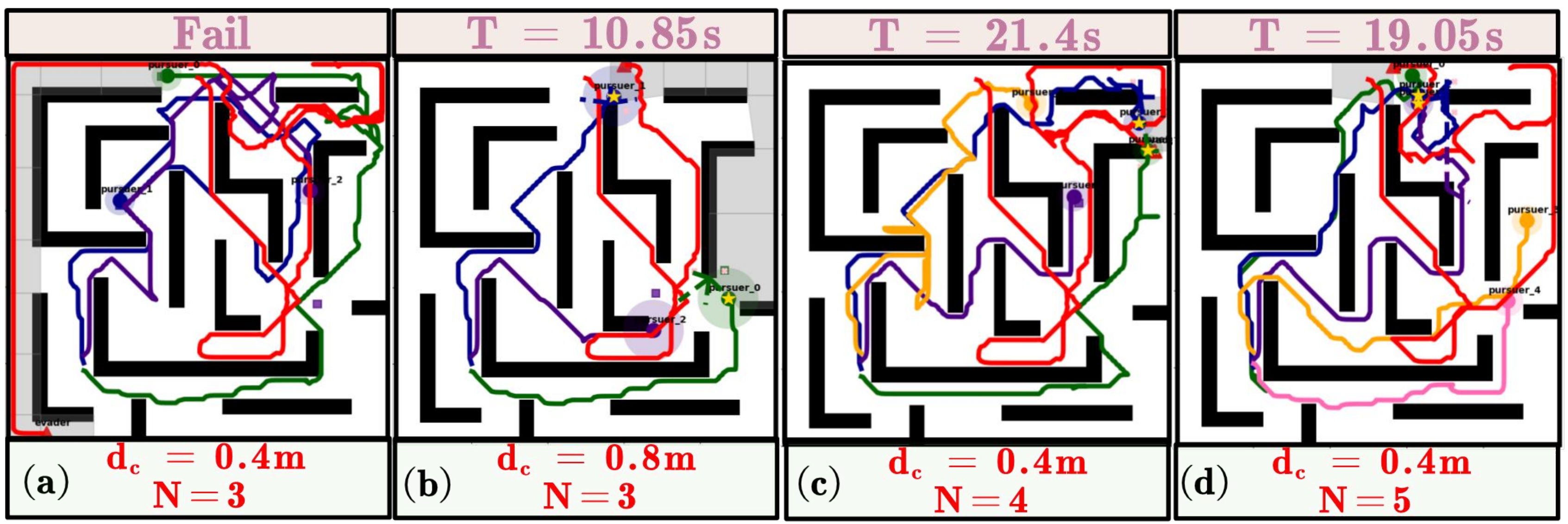}
  \caption{
    Snapshots of the final capture when the number of pursuers,
    and their capture range change.
  } 
  \label{fig:discuss}
  \vspace{-3mm}
\end{figure}

\begin{figure}[!t]
  \centering
  \includegraphics[width=0.99\linewidth]{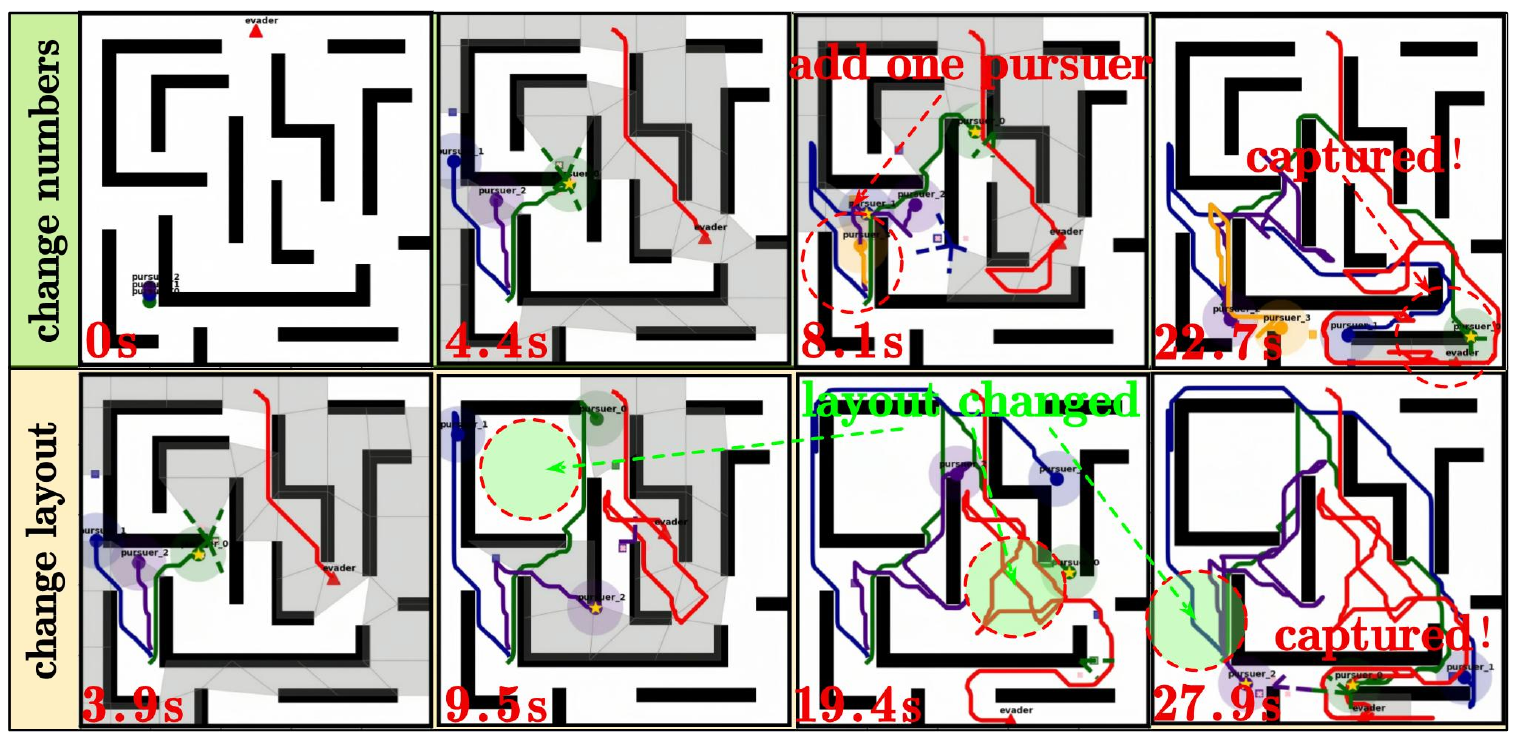}
  \caption{
Evolution of the capture process under online changes. 
\textbf{Top}: the number of pursuers increases from 3 to 4 during execution. 
\textbf{Bottom}: the environment undergoes three layout changes during online execution.
  } 
  \label{fig:change}
  \vspace{-7mm}
\end{figure}

\subsubsection{Scalability and Adaptability}
\label{subsubsec:scale-adapt}
Fig.~\ref{fig:change} shows the capture evolution when the number of pursuers or the environmental layout changes during online execution. In the top case, the number of pursuers increases from three to four at $t=8.1\,\mathrm{s}$, and the evader is eventually captured. In the bottom case, the environment initially contains a complex obstacle layout at $t=3.9\,\mathrm{s}$. During execution, three obstacles disappear successively at $9.5\,\mathrm{s}$, $19.4\,\mathrm{s}$, and $27.9\,\mathrm{s}$, as marked by the green circles, while the proposed method still completes the capture task.

\subsection{Comparisons against baselines}
\label{subsec:comp}

\begin{figure*}[!t]
    \centering
    \includegraphics[width=0.95\linewidth]{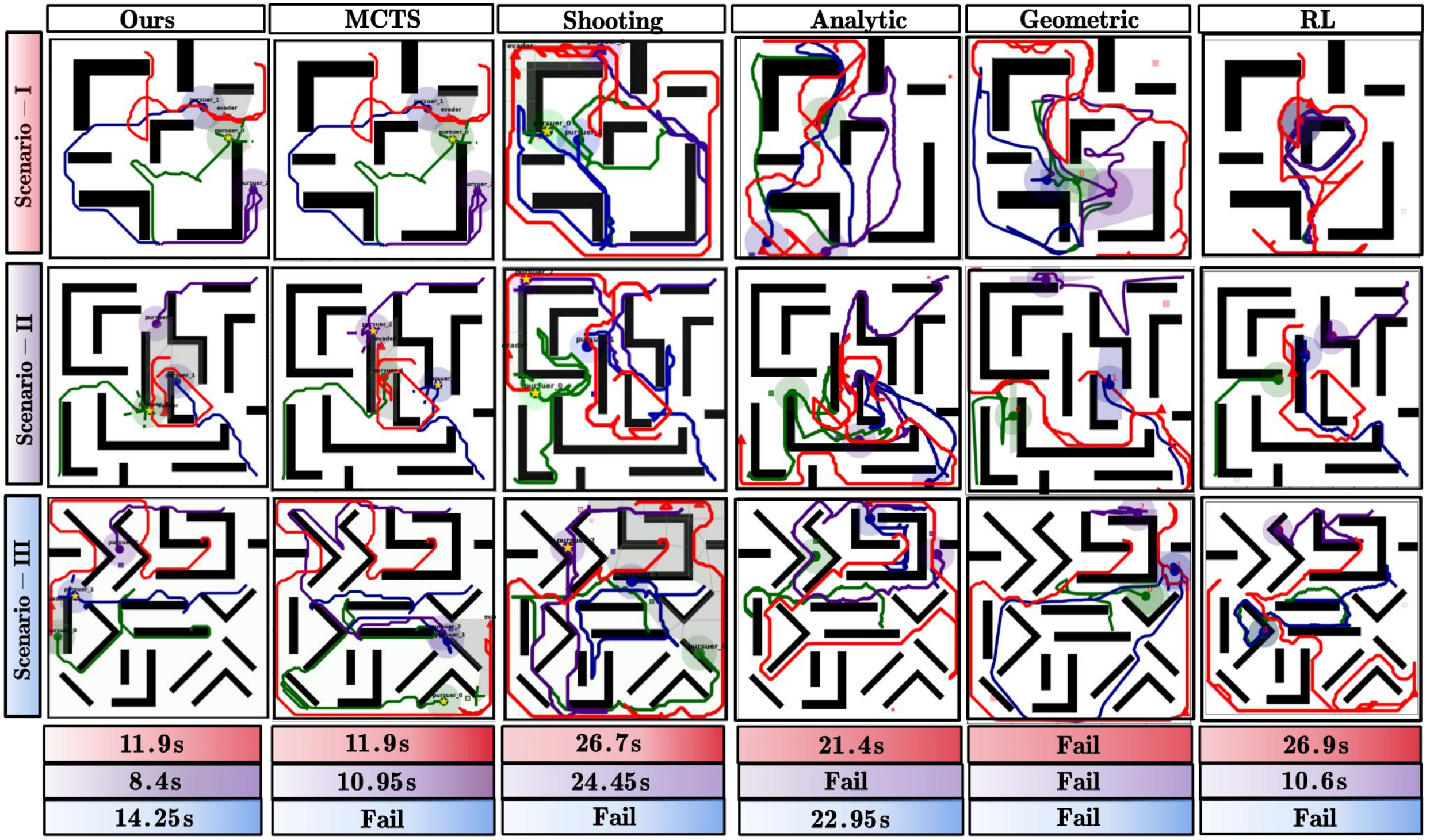}
    \caption{
      Comparison of the proposed method with five baselines across three scenarios
      with varying difficulties,
      where the final capture time is summarized in the bottom.
    }
    \label{fig:all-comp}
    \vspace{-3mm}
  \end{figure*}

\begin{table*}[!t]
  \centering
  \caption{\\\uppercase{Comparison of baselines across three scenarios.}}
  \begin{tabularx}{\textwidth}{c c c *{4}{Y} *{4}{Y} *{4}{Y}}
  \toprule\midrule
    \multirow{2}{*}{\textbf{Scene}} & \multirow{2}{*}{\textbf{Method}} & \multirow{2}{*}{\textbf{Success Rate (\%)}}
                                    & \multicolumn{4}{c}{\textbf{Capture Time (s)}}
                                    & \multicolumn{4}{c}{\textbf{Plan Time ($\times 10^{-3}$ s)}}
                                    & \multicolumn{4}{c}{\textbf{NO. of Replans (\#)}} \\
                                    &                                  &
                                    & \textbf{Avg} & \textbf{Std} & \textbf{Max} & \textbf{Min}
                                    & \textbf{Avg} & \textbf{Std} & \textbf{Max} & \textbf{Min}
                                    & \textbf{Avg} & \textbf{Std} & \textbf{Max} & \textbf{Min} \\
    \midrule
    \multirow{6}{*}{\begin{tabular}[c]{@{}c@{}}Scenario\\I\end{tabular}}
      & analytic      & 100.0 &10.3 &2.1 &15.8 &8.0 &6.7 &3.1 &10.8 &3.5 &128.3 &38.5 &160.0 &72.0 \\
      & geometric     & 0.0   &30.0 &0.0 &30.0 &30.0 &12.0 &30.4 &158 &87.8 &600.0 &0.0 &600.0 &600.0 \\
      & RL            & 100  &8.6 &7.2 &19.9 &1.6 &2.2 &1.3 &4.0 &1.0 &171.6 &143.9 &398.0 &32.0 \\
      & shooting      & 40.0 &8.2 &4.2 &13.4 &5.1 &60.02 &67.6 &181 &29.5 &3 &1.4 &5 &2 \\
      & MCTS          & 100.0&7.5 &5.8 &15.9 &3.4 &55.24 &59.7 &162 &28.3 &2.8 &2.7 &7 &1 \\
      & \textbf{Ours} & \textbf{100.0} &6.7 &4.5 &11.8 &3.4 &28.6 &0.4 &29.3 &28.4 &2.4 &1.9 &5 &1 \\
    \midrule
    \multirow{6}{*}{\begin{tabular}[c]{@{}c@{}}Scenario\\II\end{tabular}}
      & analytic      & 40.0 & 25.9 & 6.0 & 30.0 & 17.0 &20.1 &9.6 &34.0 &12.0 &516.8 &107.3 &600.0 &339.0 \\
      & geometric     & 0.0    & 30.0 & 0.0 & 30.0 & 30.0 &16.2 &6.4 &169.0 &152.0 &600.0 &0.0 &600.0 &600.0 \\
      & RL            & 100  &12.8  &2.7 &14.9 &8.1 &2.2 &1.3 &4 &1 &255.2 &53.9 &298.0 &162.0 \\
      & shooting      & 60.0 & 22.1 & 10.8 & 30 & 8.8 &31.2 &0.4 &32.0 &31.0 &10.6 &1.9 &13 &8 \\
      & MCTS          & 100.0& 11.9 & 4.6 & 18.3 & 6.6 &32.8 &0.8 &34.0 &32.0 &9 &1.6 &11 &7\\
      & \textbf{Ours} & \textbf{100.0} & 9.7 & 1.8 & 10.9 & 8.4 &21.2 &1.1 &23.0 &20.0 &8.8 &1.9 &12 &7 \\
    \midrule
    \multirow{6}{*}{\begin{tabular}[c]{@{}c@{}}Scenario\\III\end{tabular}}
      & analytic      & 20.0 &25.8 &6.0 &30 &16.9 &14.8 &8.0 &27.3 &7.3 &420.4 &168.6 &600.0 &189.0 \\
      & geometric     & 0.0    &30.0 &0.0 &30.0&30.0 &14.9 &9.6 &157 &131 &600.0 &0.0 &600.0 &600.0 \\
      & RL            & 20.0   &23.5 &5.7 &30.0 &15.1 &2.4 &0.5 &3 &2 &418.8 &114.0 &600.0 &302.0 \\
      & shooting      & 60.0 & 25.8 &7.0 &30.0 &14.0 &31.0 &1.8 &32.9 &29.0 &6.8 &3.1 &10.0 &3.0 \\
      & MCTS          & 80   &21.8 &11.9 &30 &3.9 &30.2 &1.6 &31.7 &28.5 &7 &3.9 &10.0 &1.0 \\
      & \textbf{Ours} & \textbf{100.0} &17.0 &8.7 &27.9 &3.9 &18.7 &0.2 &19.0 &18.5 &5 &2.5 &8.0 &1.0 \\
    \midrule\toprule
  \end{tabularx}
  \label{tab:result}
  \vspace{-4mm}
\end{table*}

\begin{figure}[!t]
  \centering
  \includegraphics[width=1.0\linewidth]{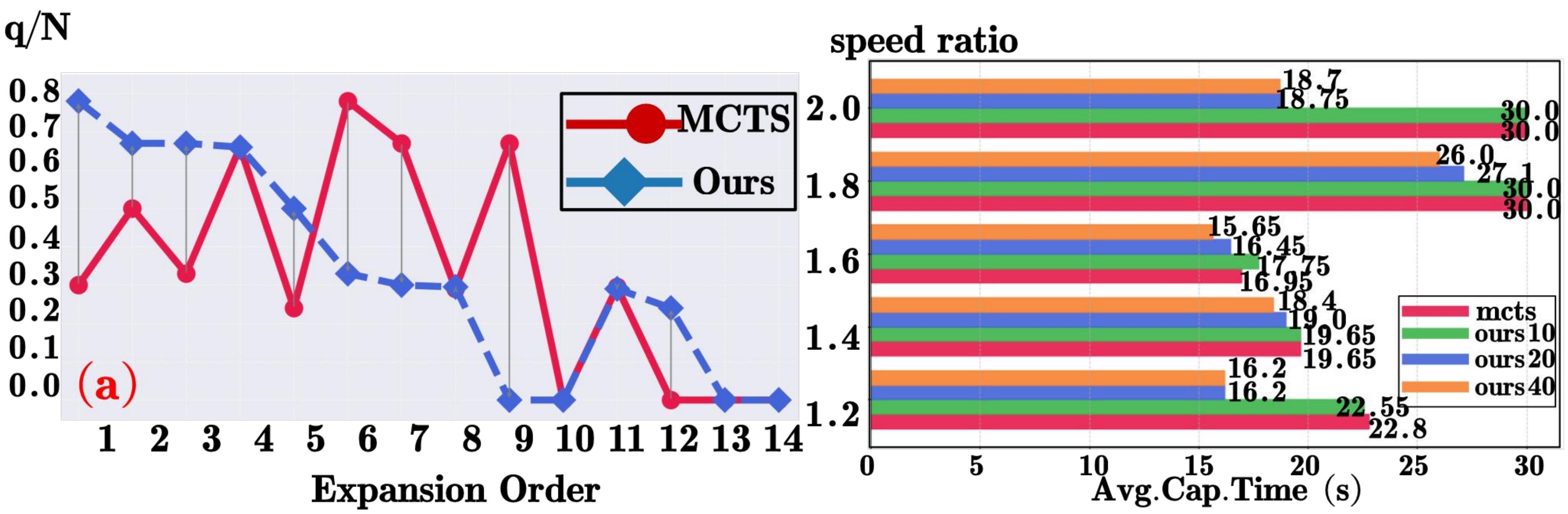}
  \caption{
    \textbf{(a)} The average node quality along with the expansion order,
    by the proposed method is much higher than the vanilla MCTS without acceleration;
    \textbf{(b)} The average capture time by comparing the proposed method exploring different
    number of nodes with vanilla MCTS at different speed ratios.
  }
  \label{fig:learn}
  \vspace{-4mm}
\end{figure}

To validate the proposed method, 
comparative experiments against five baselines are conducted across three scenarios,
each executing $5$ randomized trials with a $30\text{s}$ time limit.
Performance is evaluated using the success rate, average capture time, 
planning time, and number of replans.
Our method achieves minimal capture times ($11.9\text{s}$, $8.4\text{s}$, $14.25\text{s}$) across all scenarios,
with particularly significant advantages in the complex Scenario III
($14.25\text{s}$ vs $22.95\text{s}$ for Analytic and $30.0\text{s}$ for RL).
As evidenced in Table.~\ref{tab:result} and Fig.\ref{fig:all-comp}, 
it maintains $100\%$ success rate across all trials and initial positions,
while baselines exhibit substantial degradation, i.e.,
Analytic approach declines from $100\%$ to $20\%$ success,
RL baseline drops to $20\%$ in Scenario III, Geometric method completely fails in cluttered environments,
and Shooting approach shows the unstable $40-60\%$ performance.
Although MCTS achieves $100\%$ success in simpler scenarios,
it declines to $80\%$ in Scenario III with longer capture times ($21.8\text{s}$ vs our $14.25\text{s}$).
It should be noted that Scenario III is an unseen environment for both our method and the RL baseline. 
This explains the notable performance degradation of 
the RL baseline due to the poor generality. 
In contrast, NARE in our framework is only used to guide and
accelerate node evaluation during H-MCTS,
rather than directly determining the final action. 
Therefore, even in unseen and cluttered environments,
the proposed method can still preserve the robustness of 
search-based planning.
From a computational perspective,
our method maintains the real-time feasibility of~$18.7-28.6\text{ms}$
despite the overhead from constructions of the topological map and goal generation,
comparable to Analytic, geometric and RL methods.
Most significantly, it demonstrates the optimal replanning with only $2.4-8.8$ replans in each trial,
i.e., a $35\times$ reduction versus baselines ($128-600$ replans).
It is worth mentioning that the non-linear relationship between replanning frequency and capture effectiveness
underscores the importance of timely decision making in highly dynamic pursuit scenarios.

\begin{figure}[!t]
    \centering
    \includegraphics[width=1.0\linewidth]{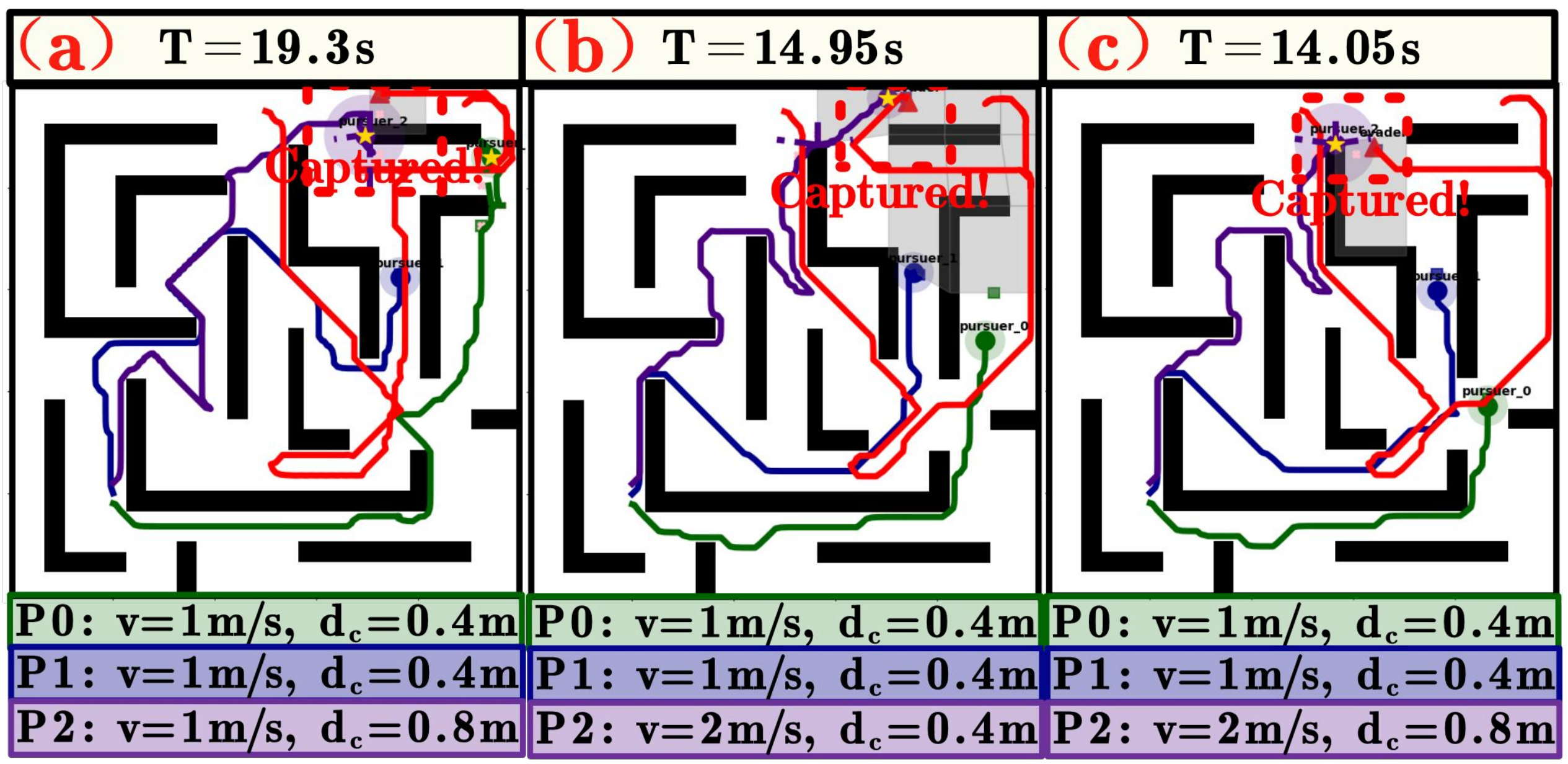}
    \caption{
    Heterogeneous team of pursuers with different capabilities, including
    \textbf{(a)} different capture ranges,
    \textbf{(b)} velocities, and \textbf{(c)} both.
    }
    \label{fig:gen-hetero}
    \vspace{-5mm}
  \end{figure}
\subsection{Results of Neural Acceleration}
\label{subsec:learn-result}
We evaluate the NARE through two complementary perspectives.
First, we examine NARE's guidance mechanism for expansion order.
As shown in Fig.~\ref{fig:learn}(a),
vanilla MCTS without NARE frequently prioritizes low-value nodes early,
delaying discovery of high-scoring nodes (blue symbols).
In contrast, NARE-guided MCTS achieves the strategic node ordering,
enabling early identification of high-potential nodes and superior solution quality under equivalent time budgets
by setting different node counts at various speed ratios in Fig.~\ref{fig:learn}(b).
It can be seen by comparing with vanilla MCTS that
(I) at the ratio of~$1.2$, $10$-node NARE matches the $20$-node MCTS performance;
(II) between the ratios of $1.4$-$1.6$, $20$-node NARE matches MCTS effectiveness
while $40$-node achieves~$12.4\%$ improvement;
(III) at higher ratios of $1.8$-$2.0$, $10$-node NARE matches $20$-node MCTS results.
These results validate that the proposed NARE can achieve similar performance with much
fewer nodes explored during expansion,
via prioritizing nodes with higher predicted quality.

We further evaluate the robustness of NARE beyond the unseen parameter settings analyzed in Section~\ref{subsubsec:ana-para},
such as five pursuers, an evader speed of~$2$ m/s, and a capture range of~$2$ m.
Specifically, we construct $50$ maps with different topologies,
which differ in obstacle shapes, obstacle numbers, spatial layouts, and the induced free-space connectivity.
Among them, $45$ maps are used for training, while the remaining $5$ more complex maps are reserved for testing.
Each testing case is evaluated under five different initial configurations,
and NARE achieves a $90\%$ capture success rate.
These results show that NARE has a good generality in unseen environments.

\subsection{Generalization}
\label{subsec:general-exp}

\subsubsection{Heterogeneous Pursuer}
\label{subsubsec:heterogeneous}

To validate our method's compatibility with heterogeneous configurations,
i.e., pursuers with different capture ranges and velocities,
comparative experiments are conducted in Scenario-II wth~$3$ pursuers
employing complex strategies at~$2.0\text{m/s}$.
Three configurations are tested versus the homogeneous baseline
with uniformly~$1.0\text{m/s}$ speed and~$0.4\text{m}$ capture range for all pursuers.
The capture range of one pursuer is increased to~$0.8\text{m}$ in Fig.~\ref{fig:gen-hetero}(a);
The maximum velocity of one pursuer is increased to~$2.0\text{m/s}$ in Fig.~\ref{fig:gen-hetero}(b);
and the combination of both in Fig.~\ref{fig:gen-hetero}(c).
Notably, the team performance is improved along with the increased capability of individual robots, i.e.,
any single enhancement improves the effectiveness, with the strongest robot consistently executing captures.
This confirms the compatibility of proposed method with heterogeneous robotic fleets.

\subsubsection{Limited View for Pursuers}
\label{subsubsec:gen-limited-view}

\begin{figure}[!t]
    \centering
    \includegraphics[width=1.0\linewidth]{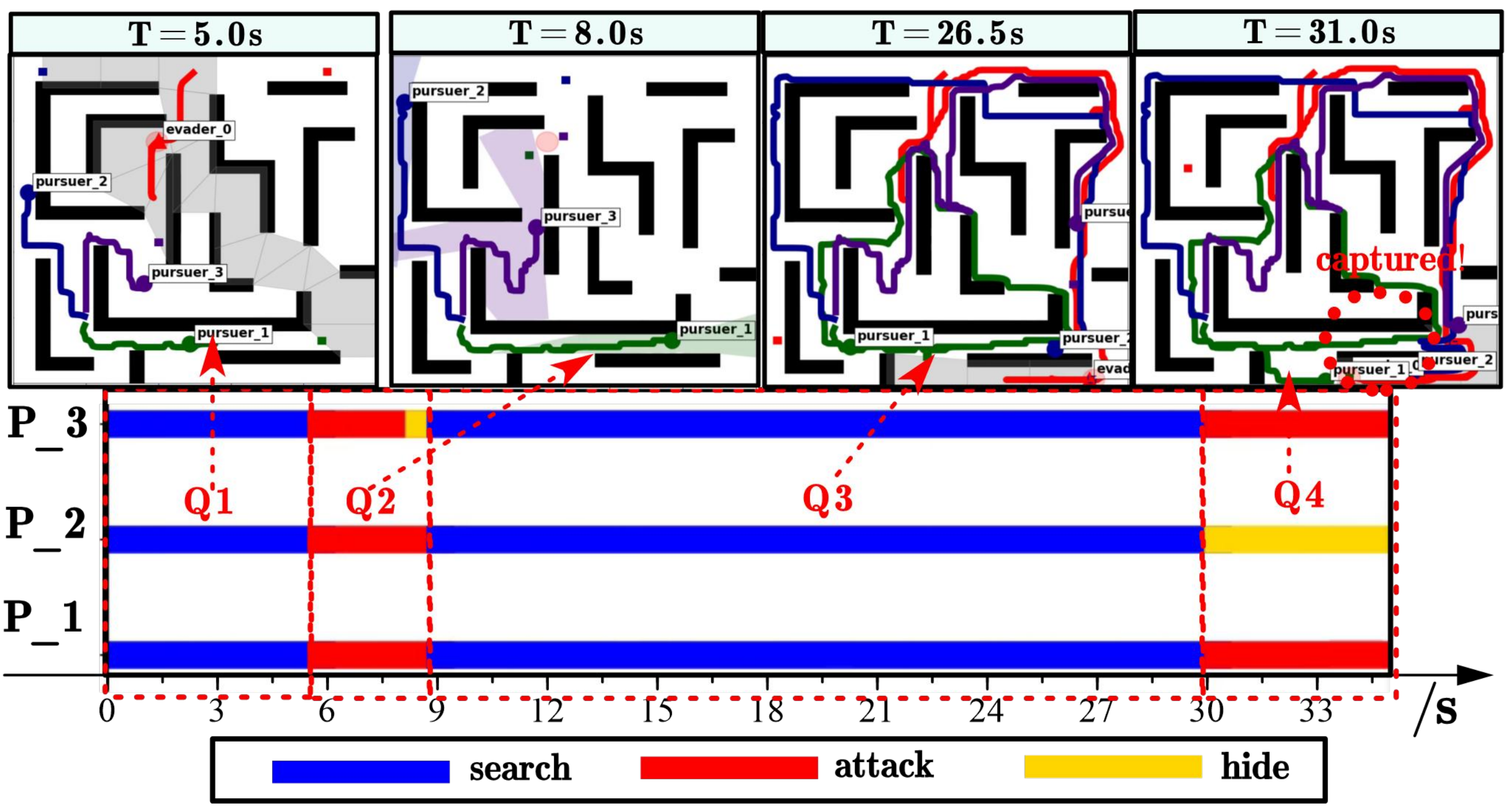}
    \caption{
    Snapshots of a successful capture at~$31\text{s}$,
    where all pursuers have a restricted view of the environment (\textbf{top}).
    Note that an additional ``search'' mode is introduced when the
    evader is not visible, along with the ``ambush'' strategy (\textbf{bottom}).
    }
    \label{fig:gen-fov}
    \vspace{-5mm}
  \end{figure}

Our method demonstrates the robustness under the constrained visibility,
e.g., $5\text{m}$ LOS range in Scenario-II with $3$ pursuers
with the speed of~$1.0 \text{m/s}$ vs the evader with~$1.2 \text{m/s}$.
As Fig.~\ref{fig:gen-fov} shows,
initial search mode in Phase~Q1 transitions to attack when Pursuer~$3$ detects the evader at~$t=5.0\text{s}$ in Phase~Q2.
After the evader disappears in Phase~Q3, all robots revert to search until Pursuer~$2$ reacquires the evader at~$26.5\text{s}$ in Phase~Q4,
enabling coordinated attacks that achieve the successful capture at $31.0\text{s}$.
These results demonstrate our ambush strategy's robustness when integrated with search algorithms
and dynamic mode-switching mechanisms under severely restricted visibility conditions.

\subsubsection{Multiple Evaders}
\label{subsubsec:gen-multi-evader}

\begin{figure}[!t]
  \centering
  \includegraphics[width=0.95\linewidth]{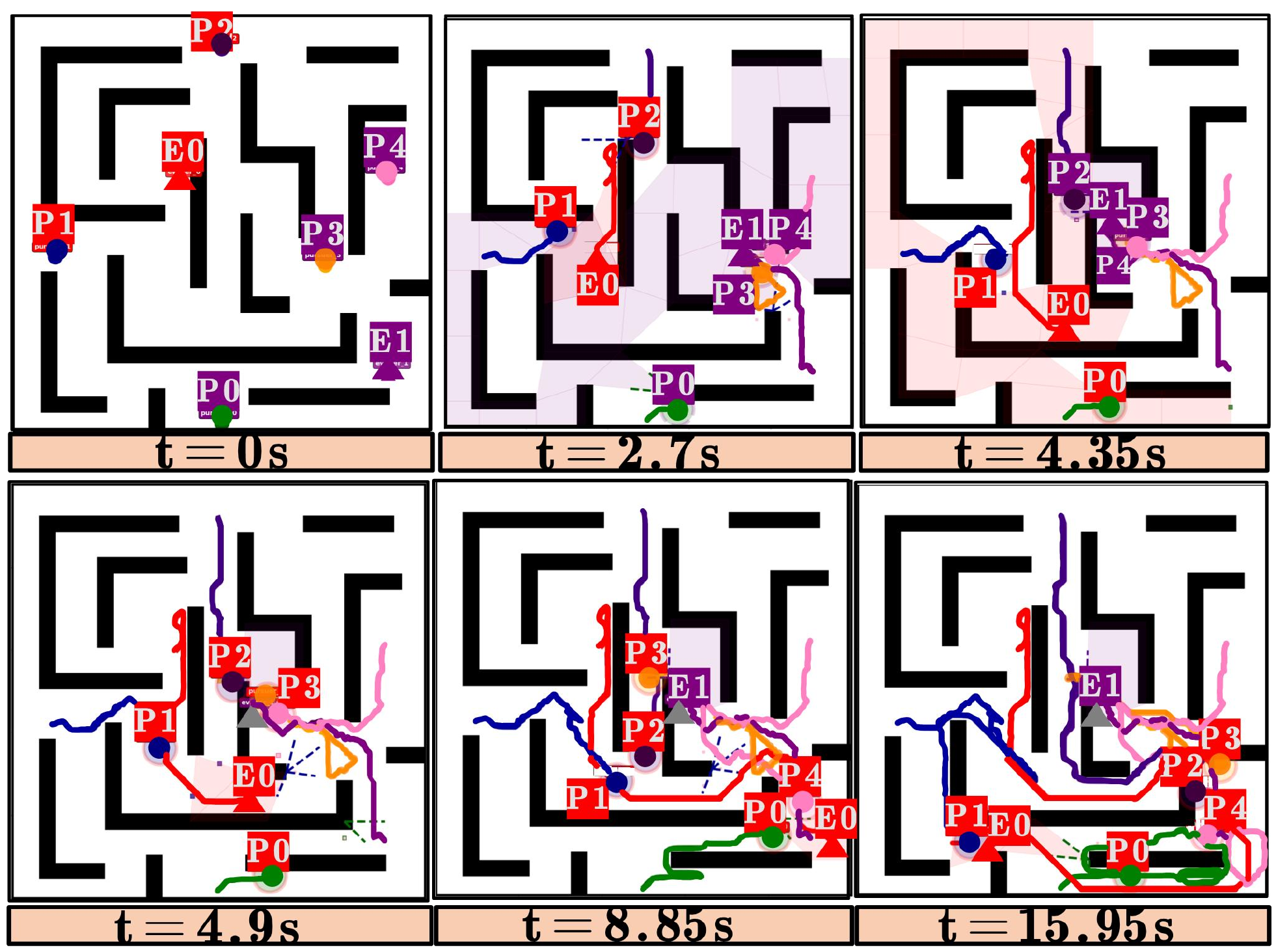}
  \caption{
  Illustration of a multi-evaders extension involving~$5$ pursuers and~$2$ evaders in Scenario-II.
  }
  \label{fig:multi-evaders}
  \vspace{-6mm}
\end{figure}

To validate the efficacy of our ambush strategy in multi-evaders scenarios,
we conduct the experiment in Scenario-II involving~$5$ pursuers and~$2$ evaders.
The system parameters are configured with a speed ratio of~$1.6$ and a uniform capture range of~$0.4\text{m}$.
As Fig.~\ref{fig:multi-evaders} shows,
initial geometric allocation partitions pursuers between evaders at $t=0\text{s}$;
Dynamic coalition reformations respond to the escapes,
i.e., Pursuer~$2$ reassists the capture of Evader~$1$ at~$t=4.9\text{s}$ after the escape at $t=2.7$s;
Focused containment captures Evader~$0$ at~$t=15.95\text{s}$ following the escape at~$t=8.85\text{s}$.
The entire mission requires only~$3$ coalition formations ($0.2\text{s}$/reallocation) and $6$ strategy adjustments,
confirming the real-time operation.
All captures result from coordinated ambush strategy within geometrically bounded groups,
validating the performance scalability without degradation from single-evader cases.

\begin{figure}[!t]
    \centering
    \includegraphics[width=1.0\linewidth]{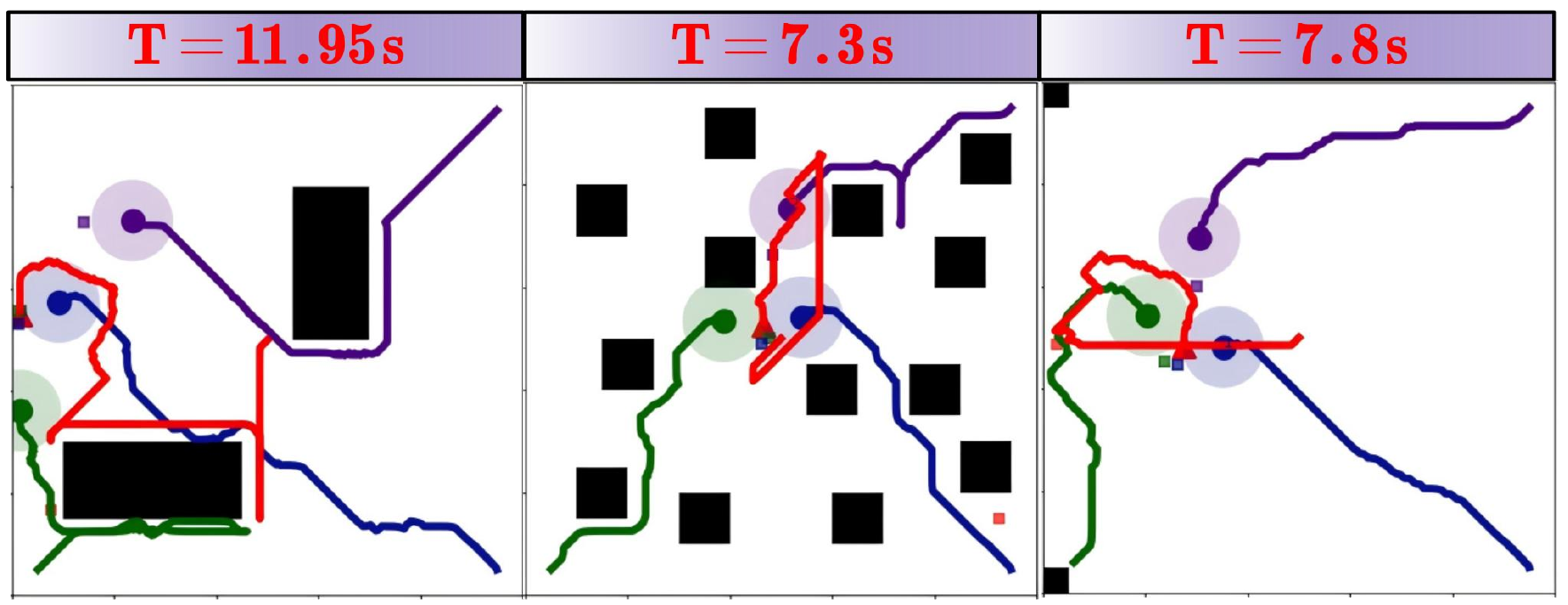}
    \caption{
  Snapshots of different sparse environments. 
  \textbf{(left)} Sparse environment with large obstacles; 
  \textbf{(middle)} Sparse environment with small scattered obstacles; 
  \textbf{(right)} Largely free space.
  }
    \label{fig:sparse}
    \vspace{-6mm}
  \end{figure}

\subsubsection{Sparse Environments}
\label{subsubsec:result-sparse}
To validate the adaptability of our proposed method to sparse environments, as introduced in Section~\ref{subsubsec:sparse_envs}, we have conducted experiments in three representative sparse scenarios, illustrated in Fig.~\ref{fig:sparse}. In the left scenario, two large obstacles were placed on a map of identical size to previous experiments, representing a sparse large-obstacle environment. The middle scenario featured multiple small obstacles to emulate a sparse small-obstacle environment. 
Finally, the right scenario corresponds to a largely free space. Experimental results indicate that the proposed approach achieves effective captures across all three sparse environment types. Notably, as obstacle density decreases, capture efficiency improves, primarily because pursuers incur less time detouring around obstacles.

\begin{figure}[!t]
    \centering
    \includegraphics[width=0.8\linewidth]{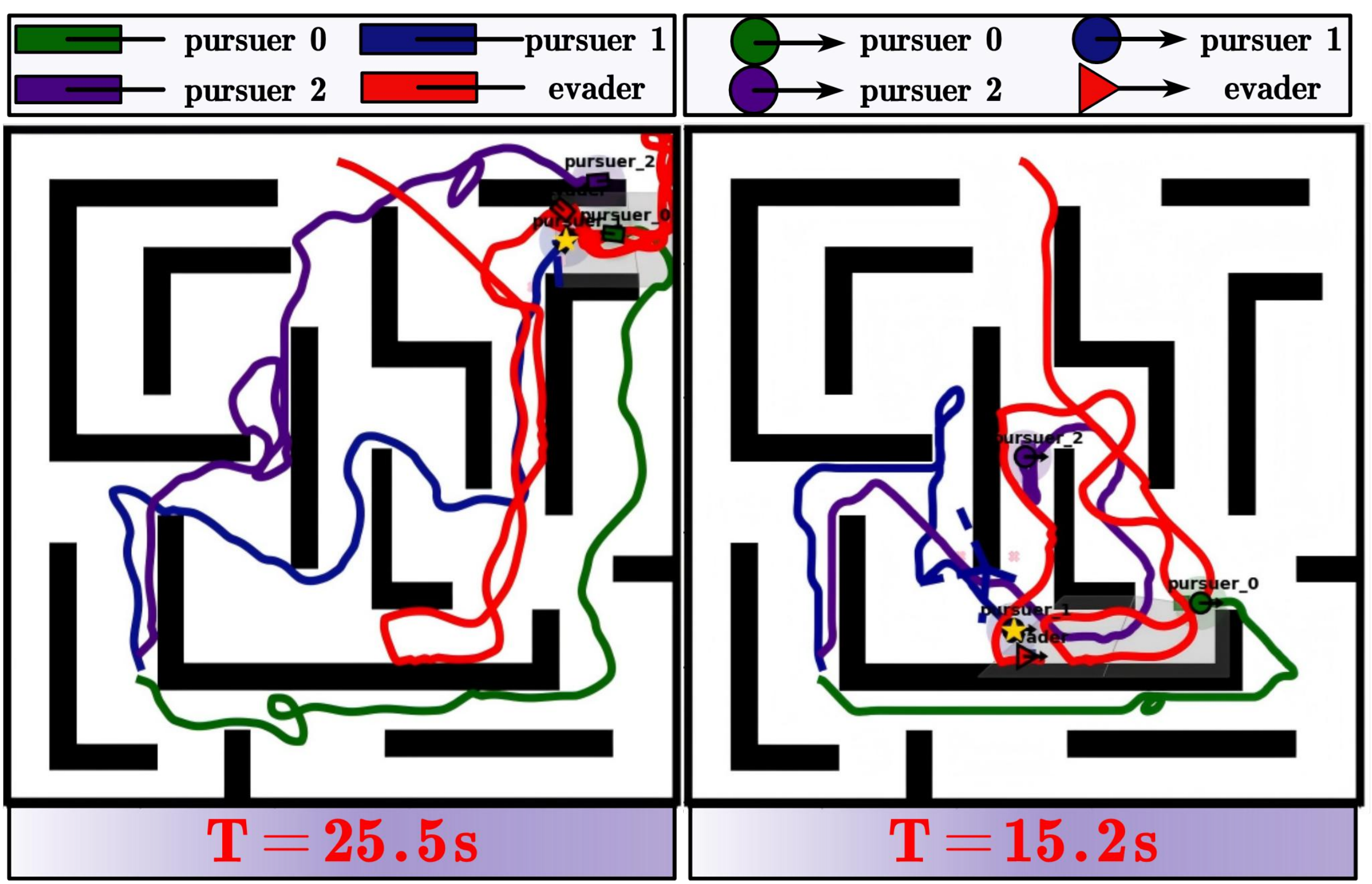}
    \caption{
Capture results under different robot dynamic models. 
The \textbf{left} subfigure shows the capture process and trajectories under the unicycle model, 
while the \textbf{right} panel illustrates the results under the double-integrator model.
}
    \label{fig:model}
    \vspace{-5mm}
  \end{figure}

\begin{figure}[!t]
    \centering
    \includegraphics[width=0.9\linewidth]{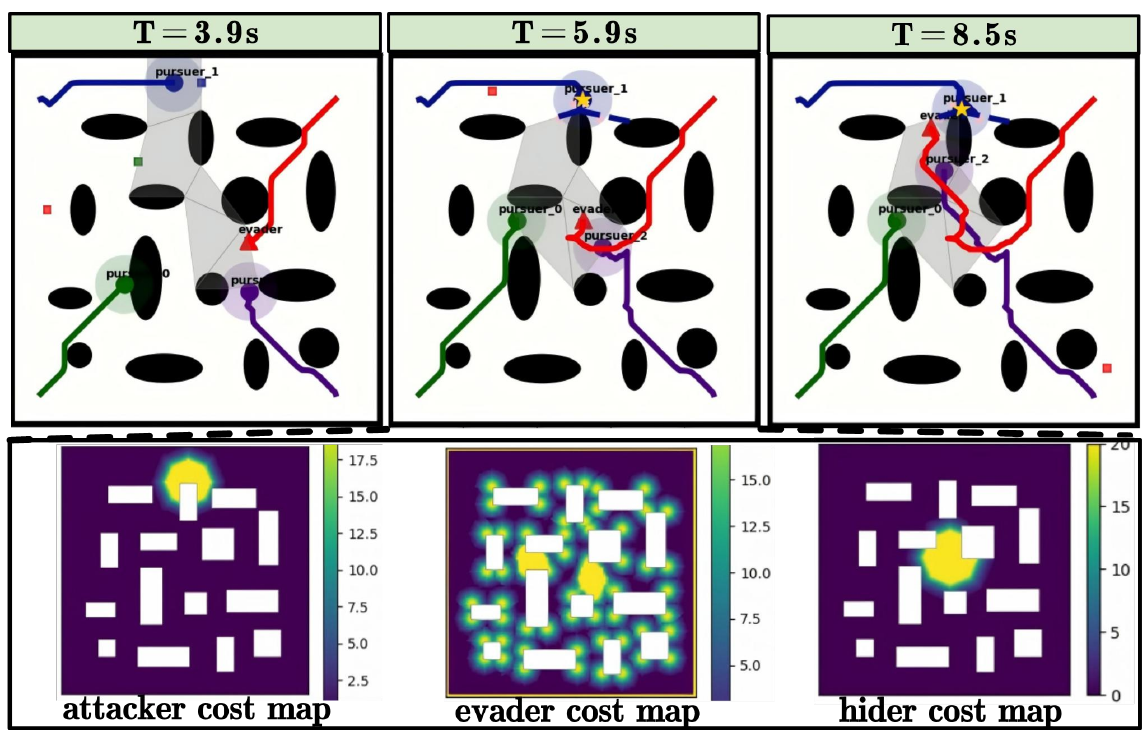}
    \caption{
Evolution of the capture process in a non-polygonal obstacle environment.
}
    \label{fig:non-poly}
    \vspace{-6mm}
  \end{figure}

\subsubsection{Different Robot Model}
\label{subsubsec:diff-model}
To evaluate the performance of the proposed method under different robot dynamics,
the same high-level pursuit strategy is used in both cases, while only the low-level robot dynamics are changed.
For the unicycle model, the angular velocity limits are set to $\pi\,\mathrm{rad/s}$ for the pursuer and $2\pi\,\mathrm{rad/s}$ for the evader, yielding a comparable minimum turning radius of approximately $0.32\,\mathrm{m}$. The heading control gain is set to $k_\omega=3.0$. For the double-integrator model, the proportional gain is set to $k_p=10.0$, and the maximum acceleration is limited to $5.0\,\mathrm{m/s^2}$. 
As shown in Fig.~\ref{fig:model}, successful capture can be achieved under both dynamics models. This result indicates that the effectiveness of the proposed method does not rely on a particular low-level robot model. Instead, the capture behavior is mainly determined by the high-level pursuit strategy. Moreover, compared with the point-mass model, whose trajectories tend to exhibit abrupt zig-zag motions, the trajectories generated by the unicycle and double-integrator models are more dynamically constrained and physically realistic. This difference is caused by the velocity, turning-rate, and acceleration limits imposed by the corresponding robot dynamics.

\subsubsection{Non-polygonal Obstacle Environments}
\label{subsubsec:exp-diff-model}
To validate the improved method in non-polygonal obstacle environments, we construct complex scenarios containing circular and curved obstacles. Fig.~\ref{fig:non-poly} (top) shows the evolution of the capture process. The proposed method can still build an effective capture graph and guide the pursuers to complete the task. At $t=5.9\,\mathrm{s}$, the hider moves around the curved obstacle boundary and reaches a favorable hiding position. At $t=8.5\,\mathrm{s}$, the evader is successfully captured by the hider.
The cost maps of the three roles at $t=5.9\,\mathrm{s}$ are also shown in Fig.~\ref{fig:non-poly} (bottom). The results indicate that the bounding-rectangle approximation is used to construct the cost maps for non-polygonal obstacles, which is consistent with the extension described in Section~\ref{subsubsec:non_poly}.

\subsection{Hardware Experiment}\label{sec:hardware}


\subsubsection{Experimental Setup}
\label{subsec:exp-setup}

\begin{figure*}[!t]
  \centering
  \includegraphics[width=0.99\linewidth]{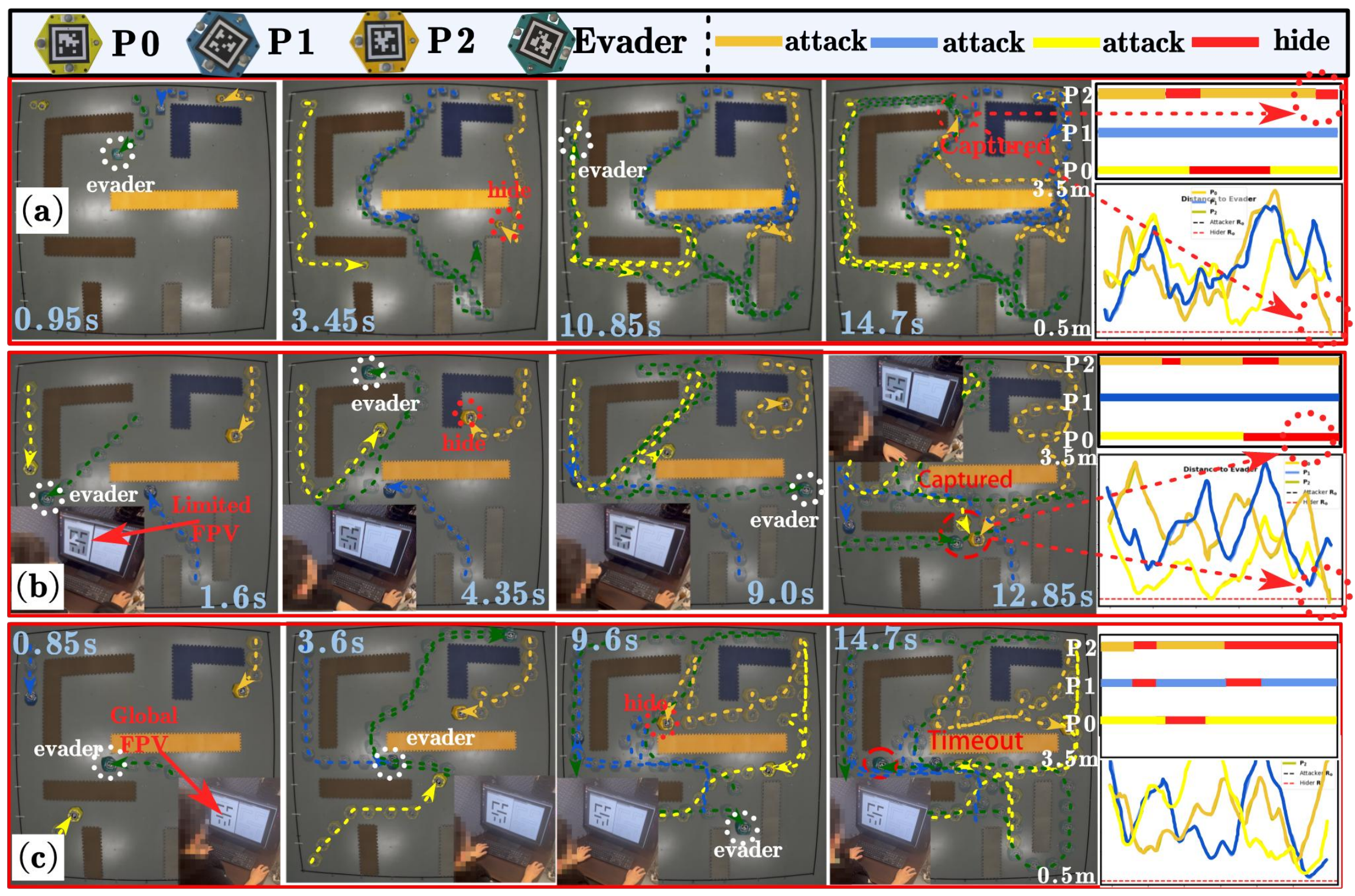}
  \vspace{-2mm}
  \caption{
  Hardware validation of the proposed ambush strategy in a complex environment with~$3$ pursuers and~$1$ evader.
  \textbf{(a)} The experimental setup features a 5m $\times$ 5m office-like layout and an autonomous evader;
  \textbf{(b)} The evader is tele-operated by an operator with limited first
  person view, where the capture takes places at $t=12.85$s;
  \textbf{(c)} The tele-operator is given a global view, where the capture failed within $60$s.
  }
  \label{fig:exp}
  \vspace{-5mm}
\end{figure*}

As illustrated in Fig.~\ref{fig:exp},
our experimental validation employs a $5\text{m} \times 5\text{m}$ arena featuring strategically arranged rectangular
and L-shaped obstacles to create the complex obstacle layout.
The testbed incorporates four differential-driven wheeled UGVs tracked by an OptiTrack motion capture system with 20 infrared cameras,
achieving the millimeter-level positioning accuracy at~$50\text{Hz}$ update rate.
Each robot utilizes a hierarchical planning architecture:
an~$\texttt{A}^\star$ global planner for the obstacle-free path generation and a LOS local planner for real-time collision avoidance.
The pursuit-evasion dynamics are configured with~$3$ pursuers operating at~$0.1\text{m/s}$ velocity and~$0.4\text{m}$ capture range,
while~$1$ evader moves at~$0.16\text{m/s}$ with~$2.5\text{m}$ perception range.
All robotic systems communicate through a centralized ROS architecture,
where a dedicated central node enforces capture constraints including relative distance thresholds and obstacle occlusion conditions.
The planning algorithms execute on a central computer station with Intel i9-13900K
that communicates with robots via a wireless local network exhibiting~$0.05\text{ms}$ latency.

\subsubsection{Results}
\label{subsec:autonomous-evader}
Hardware validation with $3$ pursuers vs. the autonomous evader from Sec.~\ref{subsec:autonomous-evader} 
confirms the real-world effectiveness. 
Despite trajectory deviations from collision avoidance and tracking errors, 
the robust capture is achieved through the continuous replanning. 
The experimental timeline reveals key strategic phases, i.e.,
the evader escapes the initial formation by $t=0.95\text{s}$; 
pursuers establish the ambush region with Pursuer~$0$ hiding at $t=3.45\text{s}$; 
the evader breaches perimeter at $t=10.85\text{s}$ by exploiting the speed advantage; 
the successful capture occurs at $t=14.7\text{s}$ when Pursuer~$2$ optimizes hidden position after~$7$ strategic adaptations. 
The complete mission duration of~$14.7\text{s}$ demonstrates our method's practical applicability in physical environments,
where the hierarchical architecture ensures effective operation despite low-level control imperfections.
Additional experiment videos are available in the supplementary material.

\subsubsection{Human-controlled Evader}
\label{subsec:human-evader}

To thoroughly evaluate our ambush strategy's effectiveness against more intelligent evaders,
we design a human-in-the-loop experiment where human operators control the evader
while pursuers execute our autonomous ambush strategy.
This human-robot competition framework tests our method against advanced escape strategies under two visual conditions:
(I) limited first-of view (FoV) simulating the realistic evader perception,
and (II) privileged global view representing the ideal situational awareness.
As shown in Fig.~\ref{fig:exp}(b)-(c),
the results demonstrate significant performance differences.
Under FoV conditions, our strategy successfully capture human-controlled evaders
with average capture time of~$12.85\text{s}$,
which is even shorter than the autonomous evader capture time of~$14.7\text{s}$.
However, with the global view access, the capture fails despite continuous strategy adaptation,
revealing fundamental limitations against perfect information adversaries.
This confirms that our ambush strategy becomes less effective against more intelligent evaders,
yet its successful captures against high-skill human operators under local view conditions
demonstrate the fundamental effectiveness of the ambush strategy.

\begin{table}[!t]
\centering
\caption{Average computation time (s) of Alg.~1 under representative 
numbers of pursuers and obstacles.}
\label{tab:alg1_runtime}
\vspace{-1mm}
\normalsize
\setlength{\tabcolsep}{3.5pt}
\renewcommand{\arraystretch}{0.92}
\begin{tabular}{c|cccccc}
\hline
\multirow{2}{*}{$N_o$}
& \multicolumn{6}{c}{Number of pursuers $N_p$} \\
\cline{2-7}
& 2 & 3 & 4 & 5 & 6 & 7 \\
\hline
2  & 1.63  & 2.83  & 5.65   & 12.64  & 30.97   & 81.94 \\
6  & 3.31  & 5.73  & 11.46  & 25.62  & 62.76   & 166.06 \\
10 & 9.20  & 15.94 & 31.87  & 71.27  & 174.58  & 461.90 \\
12 & 16.53 & 28.63 & 57.27  & 128.05 & 313.67  & 829.89 \\
14 & 30.87 & 53.47 & 106.93 & 239.11 & 585.71  & 1549.63 \\
17 & 83.72 & 145.00 & 290.00 & 648.46 & 1588.40 & 4202.50 \\
\hline
\end{tabular}
\vspace{-4mm}
\end{table}


\section{Conclusion and Future Work} \label{sec:conclusion}

This work addresses the fundamental problem of capturing faster, intelligent evaders with multiple slower pursuers in obstacle-dense environments. We propose a parameterized ambush framework that incorporates topological analysis, visibility constraints, and continuous motion dynamics, along with a hybrid MCTS planner for long-horizon strategy optimization. By training offline heuristics to guide search order and replace rollout estimates, the approach significantly accelerates online planning while guaranteeing capture. Extensive simulations and hardware experiments confirm its effectiveness against evaders of varying sophistication, including human operators.
While the method demonstrates strong performance, it is subject to certain limitations,
such as its assumption of fully known environments and reliance on centralized sensing that point toward future research in partial observability, distributed estimation, and multi-evader pursuit scenarios.

An important direction for future work is distributed deployment of the proposed framework. This may be achieved by allowing each pursuer to maintain a local system belief, generate local assignment proposals using lightweight H-MCTS/NARE inference, and coordinate with neighboring pursuers through limited communication to satisfy capture constraints without relying on a central node.


\appendices

  \section{The Time Complexity of \texttt{Goals\&Assigns}$(\cdot)$}
\label{app:complexity}
The time complexity of Alg.~\ref{alg:goals_assigns} is dominated by
the combinatorial assignment procedure.
Geometric initialization as in Lines~14--21
computes the vertices of skeleton and boundary, gates,
and goals, which has time complexity~$O(|V_{\texttt{o}}|^2)$.
The core complexity stems from $\texttt{CombinatorialAssign}(\cdot)$ in Lines~1--13,
generating all distinct $n = |\mathcal{N}_{\texttt{p}}|$ pursuers to $m = O(|V_{\texttt{o}}|)$ goal assignments.
This requires $O\left(n \cdot \frac{m!}{(m-n)!}\right)$ operations
when $n \leq m$ explores above permutations,
and can be reduced to $O(m \cdot m!)$ when $n \geq m$, i.e., 
factorial in $|V_{\texttt{o}}|$.
In addition, we conducted extensive experiments 
with different combinations of pursuers and obstacles to 
evaluate the practical computational limit of Alg.~\ref{alg:goals_assigns}. 
Table~\ref{tab:alg1_runtime} shows that the proposed implementation can support 
up to 4 pursuers and 12 obstacles 
while still approximately satisfying the real-time requirement with 57s.

\begin{figure}[!t]
  \centering
  \includegraphics[width=0.95\linewidth]{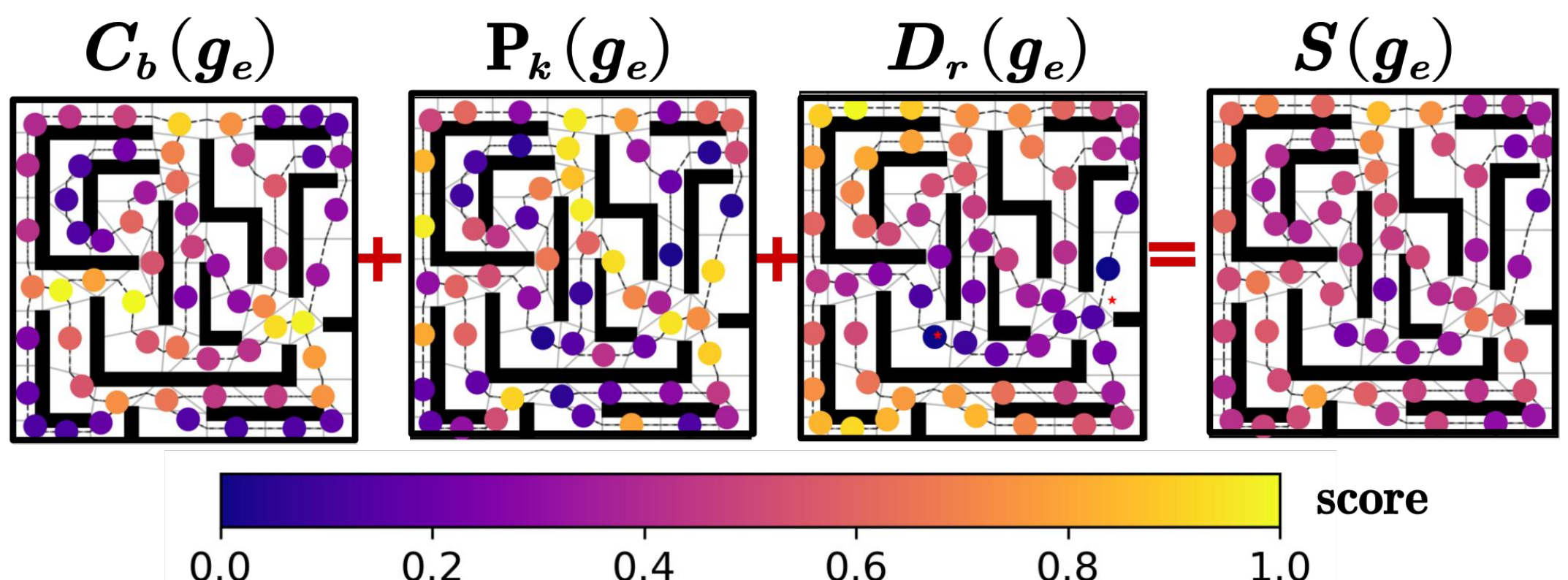}
  \caption{
    Illustration of various aspects considered
    to select temporary goals for the evader,
    including betweenness centrality (\textbf{left}),
    the stochastic perturbation (\textbf{middle}),
    and the risk-aware distance (\textbf{right}).
  }
  \label{fig:evader-goals}
  \vspace{-4mm}
\end{figure}

\vspace{-7mm}

\section{Evader Strategy}
\label{subsec:evader}
Since the evader's strategy $\pi_{\texttt{e}}$ is unknown, two strategies are designed to validate the ambush strategy,
namely, a simple strategy $\pi_{\texttt{e}}^{\texttt{l}}$ reacting only to memorized pursuers, and a complex strategy $\pi_{\texttt{e}}^{\texttt{h}}$ accounting for the pursuers' predicted high-probability future regions.

\textbf{Simple Strategy} $\pi_{\texttt{e}}^{\texttt{l}}$:
The evader reacts solely to immediate threats from pursuers within its memory
which stores the last known positions of pursuers within its line-of-sight perception range~$d_{\texttt{o}}$,
until the pursuers are re-detected and re-updated in the memory,
denoted by $\mathcal{M}_{\texttt{e}}$.
Specifically, inspired by potential field methods~\cite{fang2020cooperative},
the evader's velocity is opposite to the sum of the vectors pointing from the evader to each memorized pursuer,
namely:
$
\mathbf{v}_{\texttt{e}} = -\overline{v}_{\texttt{e}} \cdot \texttt{Norm}\left( \sum_{i \in \mathcal{M}_{\texttt{e}}} \frac{\mathbf{p}_i - \mathbf{p}_{\texttt{e}}}{\|\mathbf{p}_{\texttt{e}} - \mathbf{p}_i\|^2} \right),
$
where $\overline{v}_{\texttt{e}}$ is the evader's maximum speed,
$\texttt{Norm}(\cdot)$ denotes the vector normalization,
$\mathbf{p}_{\texttt{e}}$ and $\mathbf{p}_i$ are positions, respectively.

\textbf{Complex Strategy} $\pi_{\texttt{e}}^{\texttt{h}}$:
The evader selects an optimal goal~$g_{\texttt{e}}^\star$ from candidate goals~$\tilde{g}_{\texttt{e}}$ uniformly sampled in the free workspace $\mathcal{W}$.
As  illustrated in Fig.~\ref{fig:evader-goals},
each candidate goal is evaluated within the memoried pursuers~$\mathcal{M}_{\texttt{e}}$ by a composite score combining:
(I) betweenness centrality $C_b$ to measure connectivity in the visibility graph $\mathcal{G}_{\texttt{e}}$~\cite{d2021visibility},
(II) risk-aware distance $D_r = \textbf{exp}\left(\textbf{min}_{i \in \mathcal{M}_{\texttt{e}}} \| g_{\texttt{e}} - \mathbf{p}_i \|^2 / d_{\texttt{c}}^2 \right)$
to penalize the proximity to memoried pursuers,
and (III) stochastic perturbation $P_k \sim \mathcal{N}(0, \sigma^2)$ for the behavioral diversification.
The optimal goal is selected within~$\tilde{g}_{\texttt{e}}$
by maximizing the weighted sum of~$C_b$, $P_k$, and $D_r$.
For navigation, the evader adopts a cost map
$
\mathcal{M}^{\texttt{e}}(\mathbf{c}_{jk}) = \sum_{i \in \mathcal{M}_{\texttt{e}}} \frac{1}{2\pi d_{\texttt{c}}} \textbf{exp}\left(-\frac{\|\mathbf{c}_{jk} - \mathbf{p}_i\|^2}{2d_{\texttt{c}}}\right) + \mathcal{M}^{\texttt{e}}_{\texttt{g}} + \mathcal{M}^{\texttt{e}}_{\texttt{o}}
$
to combine the pursuer repulsion, goal attraction, and obstacle avoidance.
Optimal paths are computed via the~$\texttt{A}^\star$ search,
and updated at every step.

\vspace{-3mm}

\section{The time complexity of $\texttt{H-MCTS}(\cdot)$}
The computational complexity of~$\texttt{H-MCTS}(\cdot)$ is dominated by the following components,
i.e.,
discretization and downsampling requires $O(K^{\mathcal{N}_{\texttt{p}}})$
operations for the exhaustive generation of motion coefficients with~$K$ being the discretization numbers,
where the evaluation of each candidate computes $\Phi(\boldsymbol{\rho}^k_+)$ via path planning and visibility checks;
Then tree traversal exhibits $O(B H |\widehat{\mathcal{A}}|)$ complexity for selection phases,
with $B$ iterations, $H$ horizon, and $|\widehat{\mathcal{A}}|$ the size of assignment set;
Lastly, evolution of system dynamics incurs $O(B D \mathcal{N}_{\texttt{p}}^3)$ costs
for collision checks during transitions as in Line~\ref{step:state_transition}
and $O(B D H \mathcal{N}_{\texttt{p}}^2)$
for rollout simulations in Line~\ref{step:rollout} with $D$ downsampled candidates.
The proposed downsampling reduces $D$ to $O(1)$,
yielding the final complexity of~$O\left(B \left(K^{\mathcal{N}_{\texttt{p}}} + H |\widehat{\mathcal{A}}| + \mathcal{N}_{\texttt{p}}^3\right)\right)$.
This remains tractable for small $\mathcal{N}_{\texttt{p}}$ through the parallel candidate evaluation.

\begin{figure}[!t]
  \centering
  \includegraphics[width=1.0\linewidth]{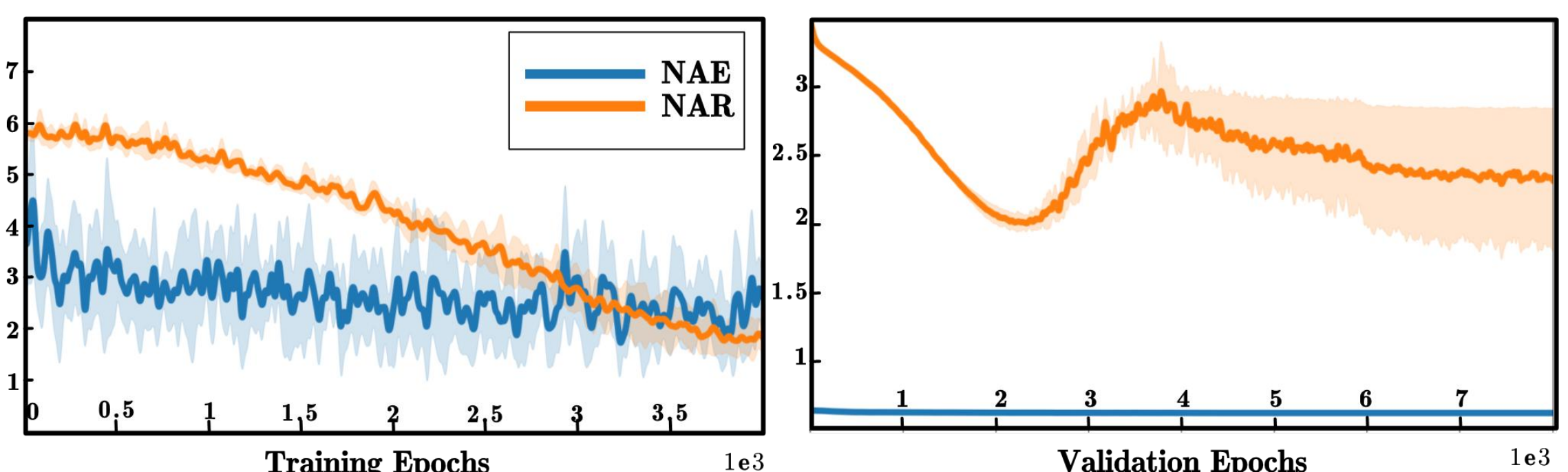}
  \caption{
    The training loss (\textbf{left})
    and validation loss (\textbf{right})
    for the proposed NAE and NAR networks,
    in the simulations of Sec.~\ref{sec:experiments}.
  }
  \label{fig:learn-result}
  \vspace{-6mm}
\end{figure}

\vspace{-5mm}

\section{Implementation Details of NARE}

To ensure the generalization capability of NARE, 
the validation set is constructed using scenarios that are not included in the training set. 
In the training data, the number of pursuers varies from 2 to 4, the evader speed ranges from 1.2 to 1.8, 
and the capture radius ranges from $1.2\,\mathrm{m}$ to $1.8\,\mathrm{m}$. 
The training environments are selected from the first 45 maps. 
In contrast, the validation set adopts more challenging and unseen settings, 
where the number of pursuers is increased to 5, the evader speed is set to 2.0, 
and the capture radius is set to $2.0\,\mathrm{m}$. 
The validation environments are selected from the remaining 5 maps with different layouts.

The NARE framework comprises two specialized GNNs.
NAR processes visibility graphs $\widetilde{G}_{\texttt{r}}$ with 3D node features
encoded via a two-layer MLP with dimensions of~$64~\text{and}~32$.
NAE handles heterogeneous state graphs $\widetilde{G}_{\texttt{e}}$ with seven node types and five edge types,
where edge features are processed through
a three-layer MLP sharing the same~$32$ dimension with LayerNorm for each layer.
Both networks employ three GNN layers, i.e.,
NAR uses GINV with edge-gated message passing,
while NAE uses HGAN with~$8$ attention heads.
A virtual super-node~$h_\nu^{(0)} \in \mathbb{R}^{32}$ connects to all vertices and is updated through GNN layers.
The final output is decoded through a series of transformations from~$\mathbb{R}^{32}$ to~$\mathbb{R}^{128}$ via a ReLU activation,
followed by layer normalization to reduce the dimension to~$\mathbb{R}^{64}$, and finally projected to~$\mathbb{R}^{1}$.
Training used datasets $\mathcal{D}_{\texttt{r}}$ and $\mathcal{D}_{\texttt{e}}$ split 80\%, 10\%, and 10\% for training, validation, and testing, respectively.
Optimization employs Adam of the weight decay $10^{-3}$ and the batch size of~$32$ with LR scheduled from $5\times10^{-5}$ to $10^{-6}$
via ReduceLROnPlateau over~$4,000$ episodes.
The training and validation curves for both NAR and NAE networks are shown in Fig.~\ref{fig:learn-result}.
Both networks demonstrate the stable convergence during training
with low validation errors.
As illustrated in Fig.~\ref{fig:NAR},
a system state along with corresponding assignments $\mathcal{A}$ is presented,
for which the NAR ranks those assignments with relevance scores.
The results indicate that dispersed surrounding assignments
generally yield better performance than those concentrating around a single target point.
This observation is further corroborated by the trend in predicted scores,
demonstrating both the accuracy and practical relevance of the NAR predictions.

\begin{figure}[!t]
  \centering
  \includegraphics[width=0.95\linewidth]{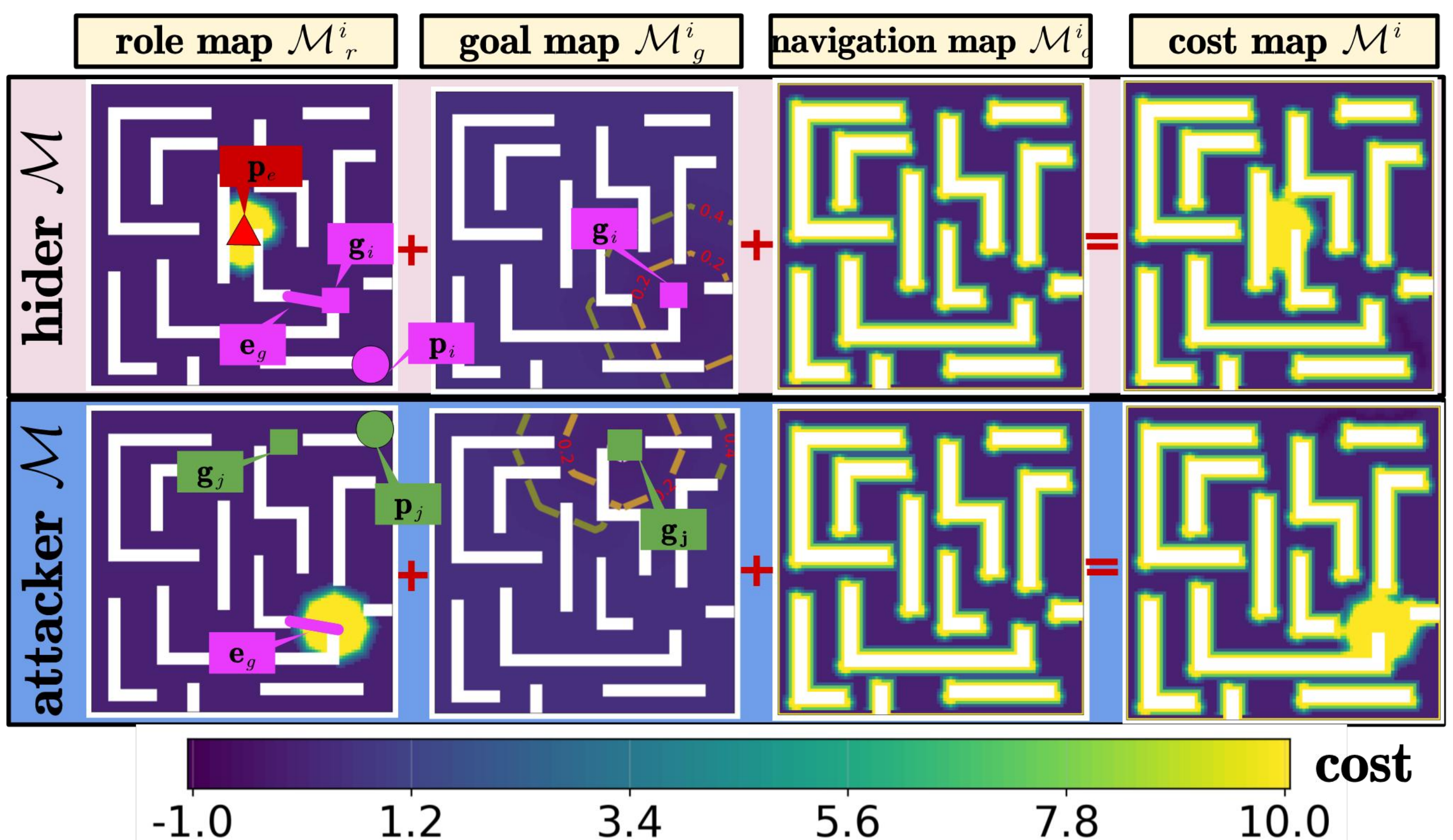}
  \caption{
    The cost map adopted for the motion strategy of pursuers,
    depending on their role as hiders (\textbf{top})
    and attackers (\textbf{bottom}).
  }
  \label{fig:cost-map}
  \vspace{-4mm}
\end{figure}


\section{Motion Strategy of Pursuers}
\label{sec:app-motion}

\textbf{Stage I:}
The core mechanism of arrival at pursuers' goals employs discrete cost maps~$\mathcal{M}$ defined over a grid representation of the workspace $\mathcal{W}$ as shown in Fig.~\ref{fig:cost-map},
where pursuers compute optimal paths using~$\texttt{A}^\star$ search over an 8-connected grid.
The composite cost at grid cell $\mathbf{c}_{jk}$,
represented by its center coordinates $\mathbf{c}_{jk} \triangleq (\textbf{x}_j, \textbf{y}_k)$, is defined as:
\begin{equation}
\mathcal{M}^i(\mathbf{c}_{jk}) = \rho^i \mathcal{M}^i_{\texttt{r}}(\mathbf{c}_{jk}) + (1-\rho^i) \mathcal{M}^i_{\texttt{g}}(\mathbf{c}_{jk}) + \mathcal{M}^i_{\texttt{o}}(\mathbf{c}_{jk}),
\label{eq:cost-map}
\end{equation}
with coefficient $0 \le \rho^i \le 1$ balancing tactical objectives against navigational costs.
Role-specific costs $\mathcal{M}^i_{\texttt{r}}$ encode specialized behaviors tailored to each role. For attackers $i \in \mathcal{N}_{\texttt{a}}$, the cost field generates adaptive repulsion toward the evader, modulated by the proximity to designated hider gates. For hiders $i \in \mathcal{N}_{\texttt{h}}$, the cost field encourages avoidance of the evader's detection radius $d_{\texttt{o}}$. Meanwhile, the goal attraction cost $\mathcal{M}^i_{\texttt{g}}$ applies uniformly across all roles via a normalized distance penalty relative to the goal, and the obstacle avoidance cost $\mathcal{M}^i_{\texttt{o}}$ assigns a prohibitive cost $c_{\texttt{o}} > 1$ to obstacle-occupied cells. The cost map updates dynamically during pursuit, enabling simultaneous optimization of role-specific goals, goal convergence, and obstacle avoidance against moving threats.
The cost map is not an independent heuristic policy, but a parameterized trajectory-generation module for executing a selected ambush assignment.
Its coefficient $\rho$ is optimized together with the discrete role-goal assignment by H-MCTS, 
while the capture quality is evaluated at the planning level.

\textbf{Stage II:}
As shown in Fig.~\ref{fig:stage},
when all pursuers reach their goals and the evader remains within capture graph $G_A$,
attackers $\mathcal{N}_{\texttt{a}}$ execute a sweeping strategy in the work~\cite{fang2020cooperative} to herd the evader
toward hider gates~$\{e_{\texttt{h}}^j\} \subset E_{\texttt{g}}^A$.
Centered at $\mathbf{p}_{\texttt{e}}$,
a circumcircle $\mathcal{C}$ with diameter $D = \mathrm{diam}(\mathcal{C})$ encloses $V_A$,
as shown in Fig.~\ref{fig:stage}(b).
After excluding the total central angle $\alpha_{\mathrm{h}} = \sum_{e \in \{e_{\texttt{h}}^j\}} \theta_e$ occupied by hider gates,
the remaining angular space $2\pi - \alpha_{\mathrm{h}}$ is equally partitioned among attackers,
i.e., $\varphi^i = (2\pi - \alpha_{\mathrm{h}})/|\mathcal{N}_{\texttt{a}}|, \forall i \in \mathcal{N}_{\texttt{a}}$ as shown in Fig.~\ref{fig:stage}(c).
Following the assignment of the nearest angular sector to each attacker,
their goal positions are computed as follows and shown in Fig.~\ref{fig:stage}(d).
Each attacker $i \in \mathcal{N}_{\texttt{a}}$ calculates its goal position
using $\mathbf{q}^i = \mathbf{p}_{\texttt{e}} + d_{\texttt{c}} \cdot \left[ \cos(\varphi^i/2), \sin(\varphi^i/2) \right]^\top$.
Attackers navigate directly to these dynamically updated $\mathbf{q}^i$ positions,
while hiders $\mathcal{N}_{\texttt{h}}$ remain stationary
until the evader's proximity to any $e_{\texttt{h}}^j \in E_{\texttt{g}}^A$ triggers a sudden attack.
To explicitly encode the herding objective, each attacker tracks its dynamically updated target $\mathbf{q}^i$ by solving:
\begin{equation}
\min_{\mathbf{u}^i}
\left\|
\mathbf{p}^i+\Delta t\,\mathbf{u}^i-\mathbf{q}^i
\right\|_2^2,
\qquad \forall i \in \mathcal{N}_{\texttt{a}},
\label{eq:herding_obj}
\end{equation}
subject to the velocity bound and workspace constraints. Equivalently, the preferred velocity of attacker $i$ is given by:
\begin{equation}
\mathbf{u}_{\mathrm{ref}}^i
=
\overline{v}_{{\texttt{p}}}
\frac{\mathbf{q}^i-\mathbf{p}^i}
{\|\mathbf{q}^i-\mathbf{p}^i\|_2+\varepsilon},
\label{eq:herding_ref_vel}
\end{equation}
where $\varepsilon>0$ is a small constant introduced for numerical robustness. Since $\mathbf{q}^i$ is defined relative to the evader position and only over the angular sectors excluding the hider gates, tracking $\mathbf{q}^i$ drives the attackers to occupy the complementary sectors around the evader and progressively compress its free angular space, thereby herding it toward $\{e_{\texttt{h}}^j\}$.
In addition to static obstacle avoidance encoded in $\mathcal{M}_{\texttt{o}}^i$, dynamic collision avoidance among pursuers is explicitly addressed during execution. To this end, we incorporate the Optimal Reciprocal Collision Avoidance (ORCA) method as a local motion refinement module~\cite{van2010optimal}. Specifically, the global paths generated by the aforementioned $\texttt{A}^\star$ planner are used as reference trajectories, while ORCA computes collision-free local velocities in real time to avoid inter-pursuer collisions under dynamic interactions. In this way, the proposed framework combines global planning for static environments with local collision avoidance for multi-robot motion, thereby producing trajectories that are both obstacle-free and safe.

\begin{figure}[!t]
  \centering
  \includegraphics[width=1.0\linewidth]{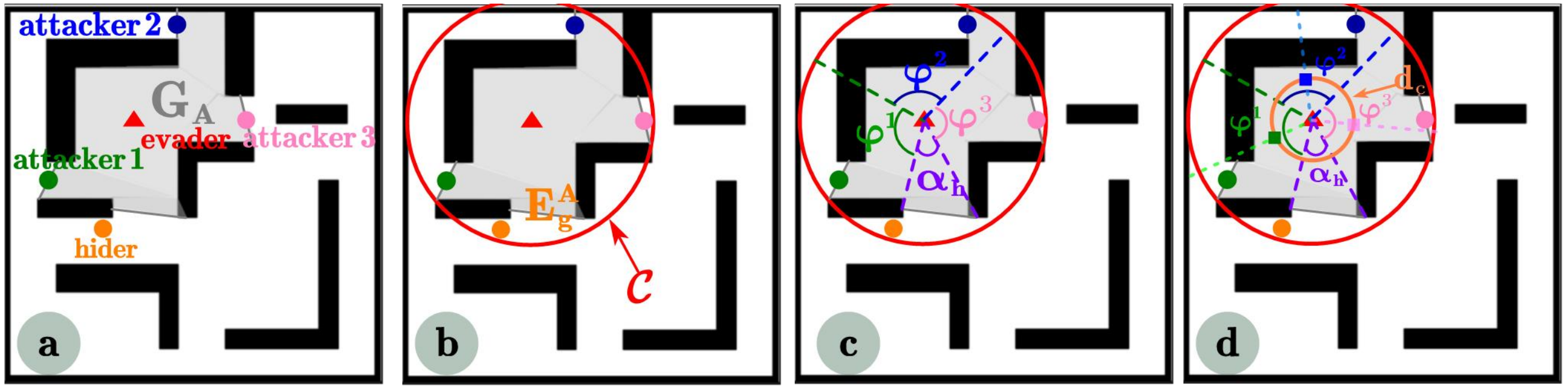}
  \caption{
    The sweeping strategy of the attackers to
    herd the evader towards the hider gates,
    where the hiders wait for a surprise capture.
  }
  \label{fig:stage}
  \vspace{-6mm}
\end{figure}

\section{Illustration of the H-MCTS Procedure}
\label{sec:app-hmcts}

Fig.~\ref{fig:H-MCTS} visually summarizes the proposed H-MCTS procedure. 
The purpose is to clarify the relationship among assignment sampling, motion-coefficient generation, rollout simulation, and reward backpropagation, rather than to repeat the detailed steps in Alg.~\ref{alg:H-MCTS}. 
In the search tree, each edge represents an ambush parameter $\xi=(A,\boldsymbol{\rho})$, and each node stores the corresponding system state after applying this parameter under the assumed evader policy. 
The best child of the root node is selected to determine the real-time assignment of ambush strategy
in the current system state.

\vspace{2mm}

\begin{figure}[!t]
  \centering
  \includegraphics[width=0.99\linewidth]{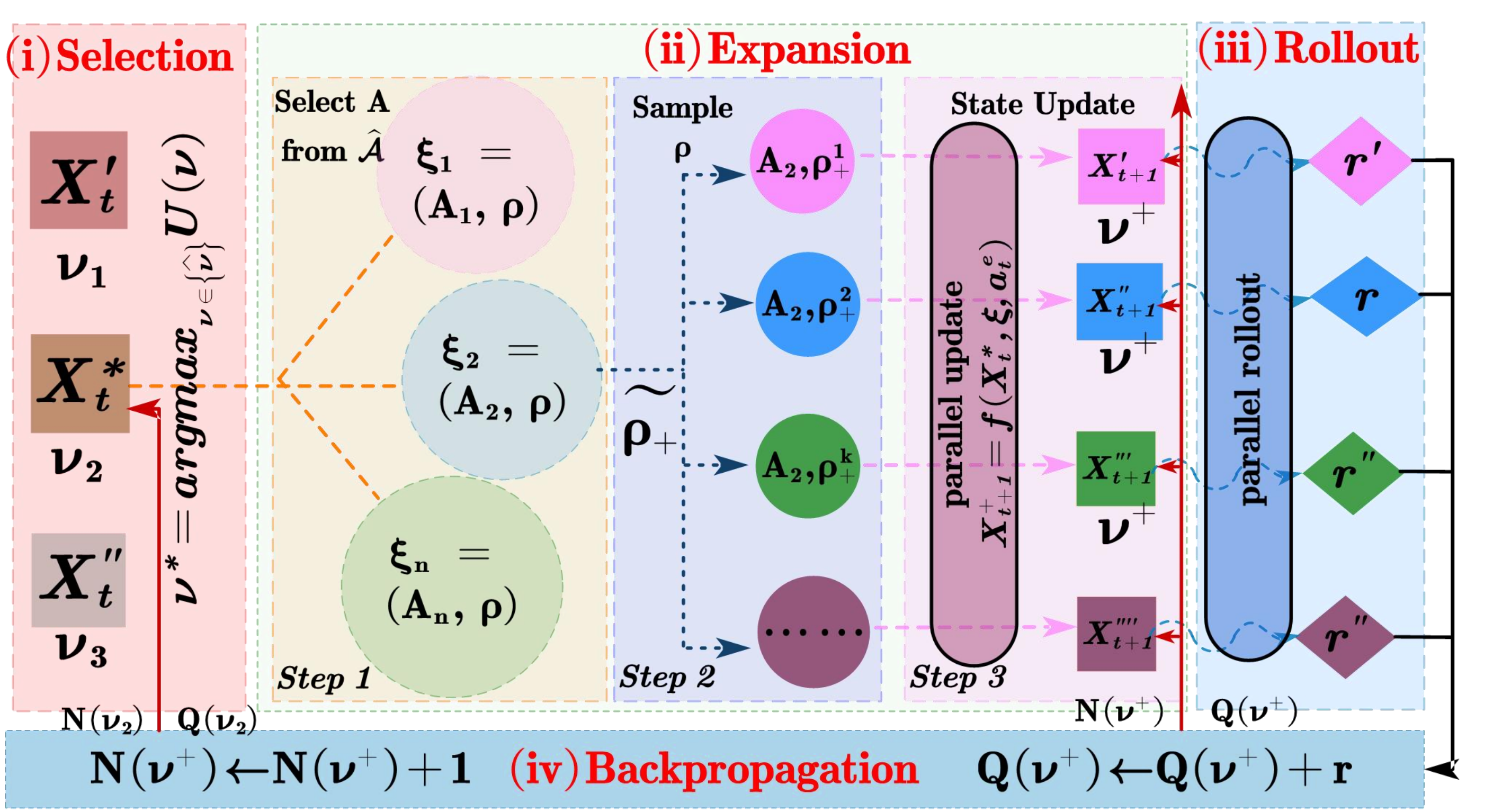}
  \vspace{-3mm}
  \caption{
    Diagram of the proposed H-MCTS scheme, consisting of four stages
    as selection, expansion, simulation, and backpropagation.
    Note that the expansion stage samples over both the discrete assignments~$A$
    and the continuous motion coefficients~$\boldsymbol{\rho}$
    for all pursuers.
  }
  \label{fig:H-MCTS}
  \vspace{-6mm}
\end{figure}

\bibliographystyle{IEEEtran}
\bibliography{contents/references}

\end{document}